  \RequirePackage{fix-cm}
  \documentclass[twocolumn]{svjour3}          
  \smartqed  
  \usepackage[table]{xcolor}
  \usepackage{bm}
  \usepackage{bbm}
  \usepackage{graphicx}
  \usepackage{multicol}
  \usepackage[lined,boxed,commentsnumbered,ruled]{algorithm2e}
  \usepackage{tabu,multirow}
  \usepackage[leftcaption]{sidecap}
  \usepackage{amsmath}
  \usepackage{amssymb}
  \usepackage{color}
  \usepackage{xcolor}
  \usepackage{enumitem}
  \usepackage{verbatim}
  \usepackage{pythonhighlight}
  \usepackage[colorlinks,citecolor=blue]{hyperref}
  
  \usepackage{graphicx}
  \usepackage{subfigure}
  \usepackage{float}  
  \usepackage{textcomp}
  
  \usepackage{multicol}
  \usepackage{booktabs}

  \usepackage{eso-pic}
  \usepackage{xspace}
  \makeatletter
  \DeclareRobustCommand\onedot{\futurelet\@let@token\@onedot}
  \def\@onedot{\ifx\@let@token.\else.\null\fi\xspace}

  \makeatother
  
  \usepackage{mathtools}

\begin{document}
  \sloppy
  
  \title{On the Generalization and Causal Explanation in Self-Supervised Learning
  }
  
  
  
  \author{ Wenwen~Qiang \and
           Zeen~Song \and
           Ziyin~Gu \and
           Jiangmeng~Li \and
           Changwen~Zheng \and
           Fuchun~Sun \and
           Hui~Xiong
  }
  
 \institute{Wenwen Qiang, Zeen Song, Ziyin Gu, Jiangmeng Li, Changwen Zheng are with the
              National Key Laboratory of Space Integrated Information System, Institute of Software Chinese Academy of Sciences, University of Chinese Academy of Sciences, Beijing, China. \{qiangwenwen, songzeen, ziyin, jiangmeng2019, changwen\}@iscas.ac.cn.\\
           \and
           Fuchun Sun is with the
              National Key Laboratory of Space Integrated Information System, Department of Computer Science and Technology, Tsinghua University, Beijing, China. E-mail: fcsun@tsinghua.edu.cn.\\
          \and
          Hui Xiong is with the Hong Kong University of Science and Technology, China. E-mail: xionghui@ust.hk.\\
          \and
          Wenwen Qiang and Zeen Song have contributed equally to this paper. Corresponding author: Jiangmeng Li.
}

\date{Received: date / Accepted: date}

\def\ourconv{RIConv++\xspace}
\def\smallgap{\vspace{0.05in}}
  
  \maketitle
  
  \begin{abstract}
Self-supervised learning (SSL) methods learn from unlabeled data and achieve high generalization performance on downstream tasks. However, they may also suffer from overfitting to their training data and lose the ability to adapt to new tasks. To investigate this phenomenon, we conduct experiments on various SSL methods and datasets and make two observations: (1) Overfitting occurs abruptly in later layers and epochs, while generalizing features are learned in early layers for all epochs; (2) Coding rate reduction can be used as an indicator to measure the degree of overfitting in SSL models. Based on these observations, we propose Undoing Memorization Mechanism (UMM), a plug-and-play method that mitigates overfitting of the pre-trained feature extractor by aligning the feature distributions of the early and the last layers to maximize the coding rate reduction of the last layer output. The learning process of UMM is a bi-level optimization process. We provide a causal analysis of UMM to explain how UMM can help the pre-trained feature extractor overcome overfitting and recover generalization. We also demonstrate that UMM significantly improves the generalization performance of SSL methods on various downstream tasks. The source code is to be released at~\href{https://github.com/ZeenSong/UMM}{https://github.com/ZeenSong/UMM}.

  
  
  \end{abstract}

\section{Introduction}

Representation learning without instance-level annotations is an important research topic in machine learning. One of the recent advances in this area is self-supervised learning (SSL), an unsupervised learning approach that has shown superior performance in various computer vision tasks, such as classification, object detection, segmentation, and transfer learning \cite{jaiswal2020survey,SI2022727,radford2021learning}. A key feature of SSL is its instance-based learning paradigm, which treats each training sample as a separate class. This simple but elegant paradigm enables SSL to learn semantic information from the data itself, which can be transferred to different downstream tasks with satisfactory results. As SSL continues to evolve, SSL methods, e.g., SimCLR \cite{chen2020simple} and BYOL \cite{grill2020bootstrap}, are also closing the performance gap with supervised methods. Unless otherwise stated, SSL refers to contrastive and non-contrastive learning methods that are based on data augmentation invariance in this paper.

The objective function of SSL methods \cite{chen2020simple,grill2020bootstrap,zbontar2021barlow} can be generally decomposed into two parts: alignment and constraint. The alignment part aims to make the features of different augmented samples from the same ancestor as similar as possible. The constraint part imposes additional prior knowledge on the data distribution, parameter update mode, features, etc. during training. An essential concern in SSL is the learning of an effective model that can generalize well to downstream tasks, typically achieved through deep neural networks. Usually, deep neural networks have more parameters than the available training data, which can lead to overfitting in a real-world scenario \cite{novak2018sensitivity}. Overfitting means that the network memorizes the specific characteristics of the training data instead of learning a generalizable solution \cite{zhang2021understanding,stephenson2021geometry}. 
In general, overfitting can be observed by measuring the variability in test accuracy of intermediate models obtained at each training update. For instance, if at some point during training, the model's test accuracy starts to consistently decrease, it can be considered as an indication of overfitting. 
However, conducting simultaneous training and testing is time-consuming and resource-intensive, making it highly meaningful to find a means of quantifying overfitting issues solely based on the training data during the training process. Therefore, this paper primarily focuses on the following issues: first, whether SSL encounters overfitting problems; second, how to quantify overfitting between training phases; third, how to address overfitting when it occurs in SSL.

To explore whether SSL encounters overfitting problems, we then conducted a detailed analysis of the SSL training process on different datasets and methods. Specifically, we selected five prominent self-supervised learning techniques: SimCLR, BYOL, Barlow Twins \cite{zbontar2021barlow}, SwAV \cite{swav}, and VICRegL \cite{Vicregl}. Our analysis also included three small-scale classification datasets, e.g., CIFAR-10 \cite{cifar10}, CIFAR-100 \cite{cifar10}, and Tiny-ImageNet \cite{leTinyImagenetVisual2015}, as well as two large-scale datasets, e.g., ImageNet \cite{krizhevsky2012imagenet} and ImageNet100 \cite{tianContrastiveMultiviewCoding2020}. We found that different SSL methods exhibit similar patterns in the training phase. Specifically, the test accuracy of the features from the early-layer of the feature extractor increases and stabilizes as the training progresses, while the test accuracy of the features from the last-layer increases, then decreases, and finally stabilizes. Also, we can observe that the highest test accuracy of the last-layer features throughout the training process is significantly better than the test accuracy of the last-layer features obtained at the end of training. Meanwhile, the highest test accuracy of the early-layer features throughout the training process and the test accuracy of the early-layer features obtained at the end of training are nearly consistent. Thus, we conclude that SSL indeed suffers from overfitting, and that generalization is achieved in early layers for all epochs, while memorization occurs in later layers and epochs. 

To quantify overfitting between training phases, taking into account the close relationship between feature information content and generalization, we proposed a novel perspective by reevaluating the changes in feature information content during the SSL training phase from an information theory standpoint. We used the Coding Rate Reduction \cite{chan2022redunet} to quantify the amount of information embedded in the feature representation and made intriguing observations regarding certain features. Specifically, we found that the coding rate reduction of the last-layer output first increases to a maximum and then decreases and stabilizes, while the coding rate reduction of the early-layer output increases and stabilizes. In other words, the trends in the performance of the coding rate reduction during the training phase align with the trends in model testing accuracy. Therefore, we conclude that coding rate reduction can be used to quantify the overfitting problem of SSL, with lower values indicating higher overfitting.

To address the overfitting problem of SSL, we propose a novel learning mechanism called Undoing Memorization Mechanism (UMM). UMM can be viewed as fine-tuning the parameters of the feature extractor pre-trained by SSL methods, from the early layer to the last layer, while keeping the parameters before the early layer fixed. The motivation of UMM is based on our experimental finding that early layer features are not affected by overfitting to training data. UMM is designed as a dual optimization process. The first optimization process (constraint condition) aims to align the features from the last layer with the features from the early layer. The second optimization process (objective function) aims to maximize the coding rate reduction of the last layer features by using the informative part of the early layer features. This way, the last layer features can recover generalization. We provide the rationale of UMM design and explain why UMM can restore generalization from a causal analysis perspective. We also validate the effectiveness of UMM through extensive experiments. The main contributions are as follows:
\begin{itemize}
    \item We experimentally validate that in the SSL training process, the feature extractor learns generalizable features in early layers for all epochs, and abruptly memorizes specific features in later layers and epochs. Meanwhile, we find that Coding Rate Reduction can serve as an indicator to evaluate the overfitting problem.
    \item We propose a bi-level optimization-based Undoing Memorization Mechanism (UMM) that can be viewed as tuning the parameters of the feature extractor pre-trained by SSL methods. UMM undoes the memorization of the last layer by rewinding based on the early layer and can be easily integrated with the exciting SSL method.
    \item We propose to use causal analysis to explain how UMM can help the pre-trained feature extractor overcome overfitting and recover generalization. We also validated the effectiveness of UMM on various downstream tasks.

\end{itemize}

\section{Related Work}
\textbf{Learning Method.} Contrastive learning (CL) is a prominent approach to learn visual representations without labels \cite{oord2018representation,tian2020contrastive,hjelm2018learning}. SimCLR \cite{chen2020simple} is the first widely used CL method that achieves comparable performance to supervised learning. However, SimCLR requires large training batches and high computational resources. MoCo \cite{moco,chen2020improved,chen2021empirical} addresses this issue by using dynamic memory allocation. MetAug \cite{li2022metaug} generates hard positive samples to reduce the negative sample redundancy. Some methods avoid negative samples altogether, such as BYOL \cite{grill2020bootstrap}, W-MSE \cite{ermolov2021whitening}, Simsiam \cite{chen2021exploring}, and Barlow Twins \cite{zbontar2021barlow}. However, these methods ignore the intrinsic structure of the data distribution. SwAV \cite{swav} and PCL \cite{li2020prototypical} exploit the clustering structures embedded in the data distribution. LMCL \cite{chen2021large} mines the large margin between positive and negative samples. CL can be seen as an instance-based learning paradigm, which limits its ability to capture the relationship between different instances. ReSSL \cite{zheng2021ressl,tomasev2022pushing} measures the similarity of the data distribution based on two augmented samples. I-JEPA \cite{assran2023self} predicts the representations of various target blocks in the same image without relying on hand-crafted data augmentations. SiameseIM\cite{tao2023siamese} shows that it is possible to obtain both semantic alignment and spatial sensitivity with single dense loss.  Unlike these methods, our proposed UMM aims to address the overfitting problem of SSL.

\textbf{Theoretical Analysis.} \cite{arora2019theoretical} provide the first generalization bound for contrastive learning based on the Rademacher complexity. Wang et al. \cite{wang2020understanding} analyze the contrastive learning from the perspective of alignment and uniformity on the hypersphere. Roland et al. \cite{zimmermann2021contrastive} reveal a fundamental connection between contrastive learning, generative modeling, and nonlinear independent component analysis. Wang et al. \cite{wang2020understanding} show that the temperature parameter controls the alignment and uniformity of the learned features in the feature space. Von et al. and Qiang et al. \cite{von2021self,qiang2022interventional} understand the contrastive learning from causal analysis, such as intervention and counterfactual. Ash et al. \cite{ash2021investigating} report that the performance of a representation learned via contrastive learning can degrade with the number of negative samples. However, Nozawa and Pranjal \cite{awasthi2022more,nozawa2021understanding} argue that a larger number of negative samples do not necessarily harm contrastive learning. Hu et al. \cite{hu2022your} demonstrate through analysis that the self-supervised contrastive learning can be viewed as a special case of stochastic neighbor embedding and preserve the pairwise distance specified by data augmentation. Garrido et al. \cite{garrido2022duality} show that the sample-contrastive learning method such as SimCLR and dimension-contrastive learning approaches such as VICReg are closely related. Simon et al. \cite{simon2023stepwise} study the dynamics of the learning process of Barlow Twins and find that the model learns the top eigenmodes of a certain contrastive kernel in a stepwise fashion, the same stepwise learning phenomenon also exists in various contrastive learning process. G{\'a}lvez et al.\cite{galvez2023role} analyze the clustering-based (SwAV) and distillation-based (BYOL) contrastive learning with a different lower bound on the mutual information.  In this paper, we also provide a theoretical analysis to support the validity of our proposed UMM.

\begin{figure*}[t]
    \centering
    \subfigure[CIFAR-10]{
    \includegraphics[width=0.3\textwidth]{ 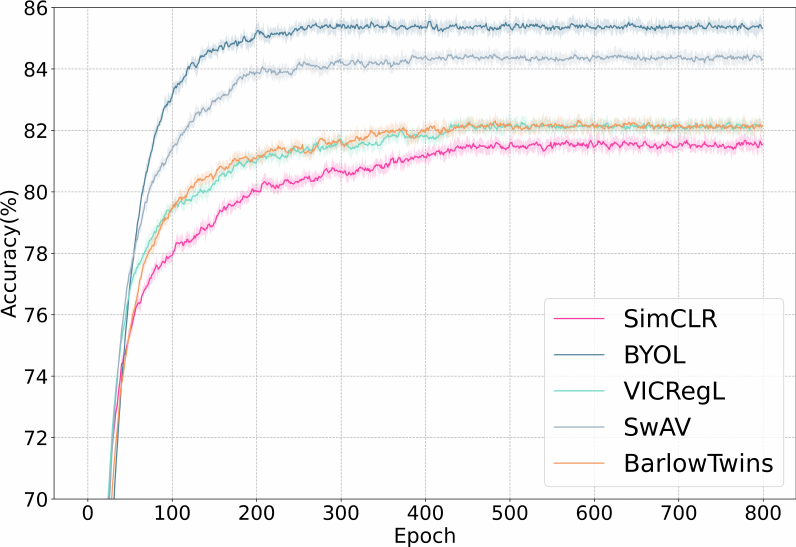}
    }
    \subfigure[CIFAR-100]{
    \includegraphics[width=0.3\textwidth]{ 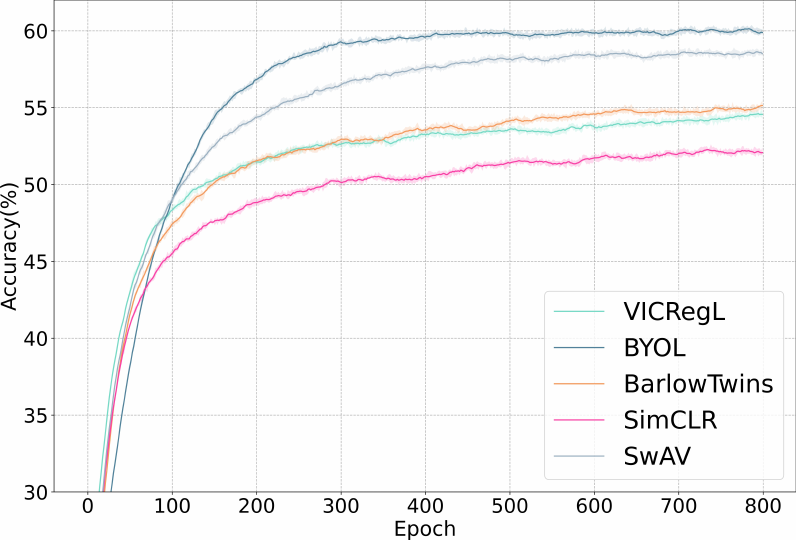}
    }
    \subfigure[Tiny-ImageNet]{
    \includegraphics[width=0.3\textwidth]{ 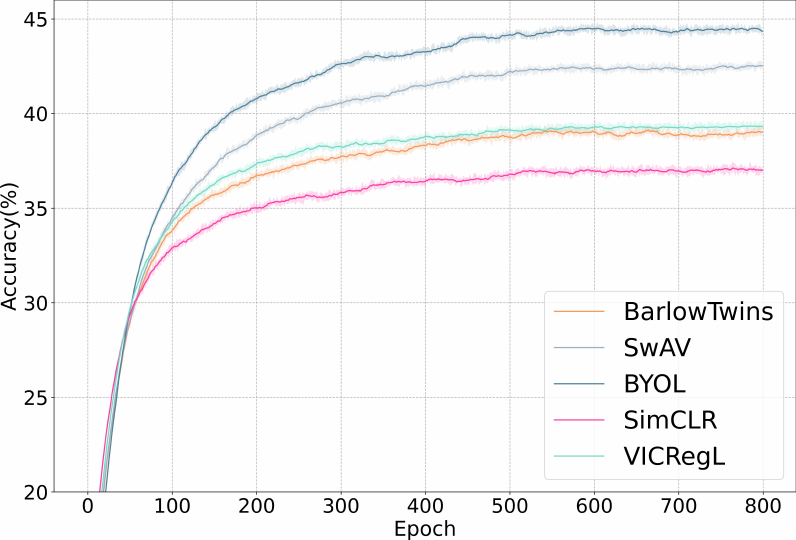}
    }
    \subfigure[ImageNet-100]{
    \includegraphics[width=0.3\textwidth]{ 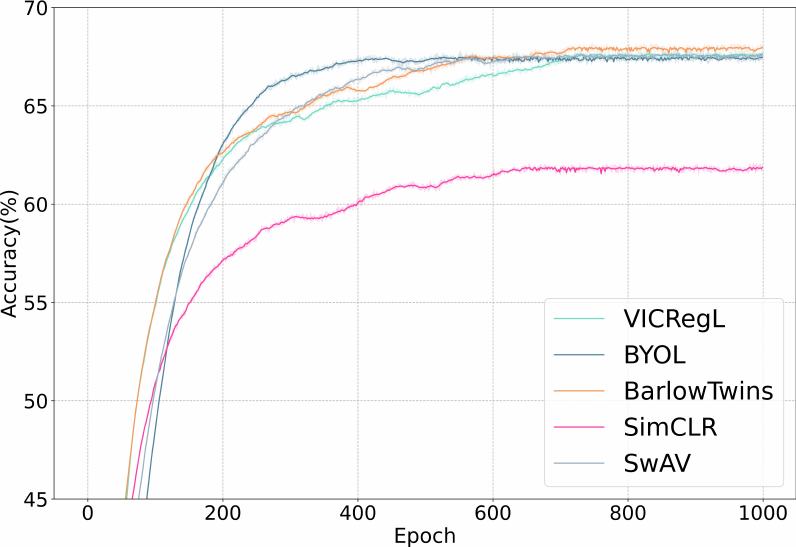}
    }
    \subfigure[ImageNet]{
    \includegraphics[width=0.3\textwidth]{ 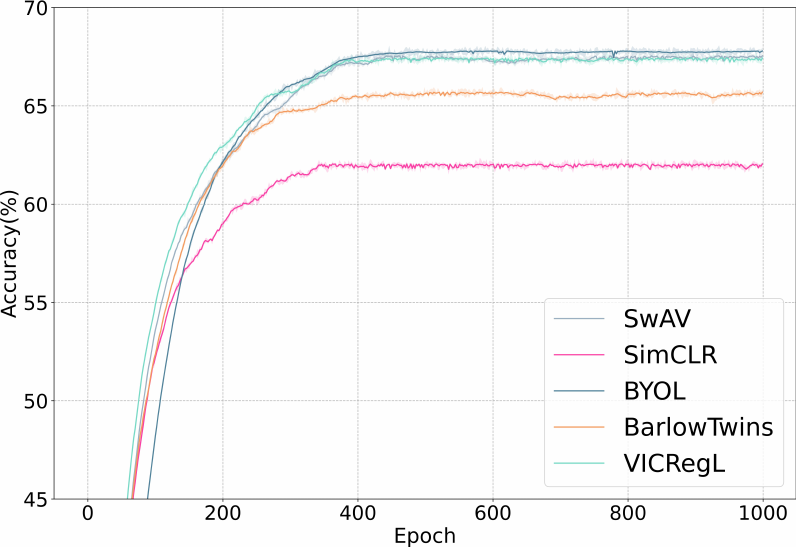}
    }
    \\
    \subfigure[CIFAR-10]{
    \includegraphics[width=0.3\textwidth]{ 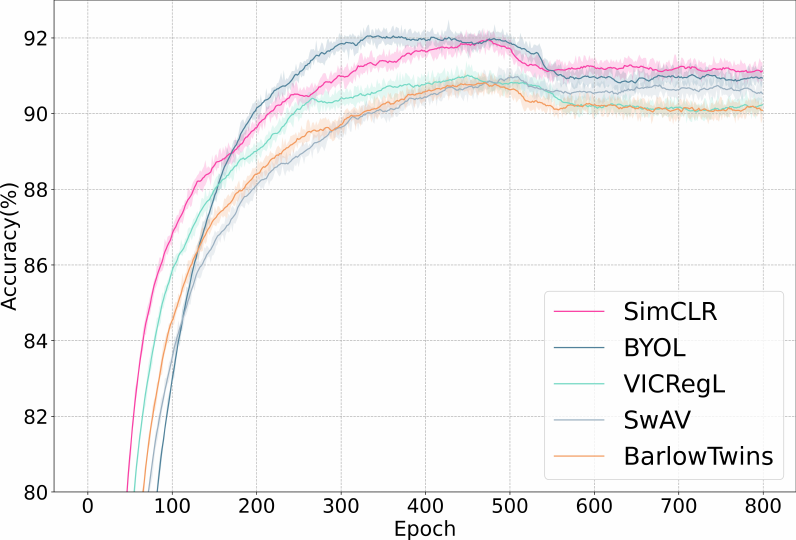}
    }
    \subfigure[CIFAR-100]{
    \includegraphics[width=0.3\textwidth]{ 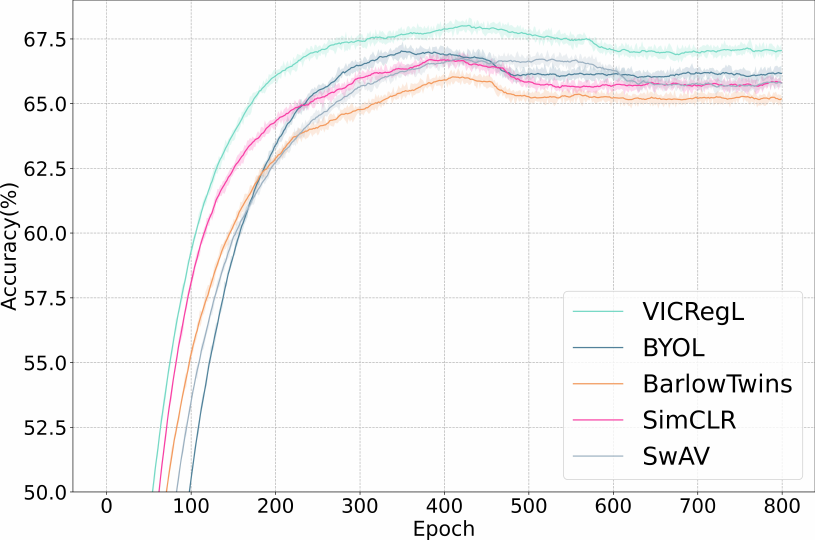}
    }
    \subfigure[Tiny-ImageNet]{
    \includegraphics[width=0.3\textwidth]{ 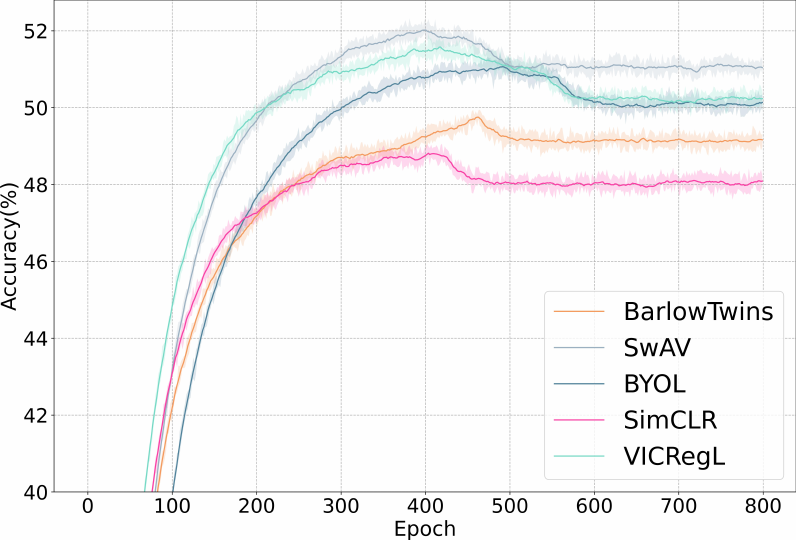}
    }
    \subfigure[ImageNet-100]{
    \includegraphics[width=0.3\textwidth]{ 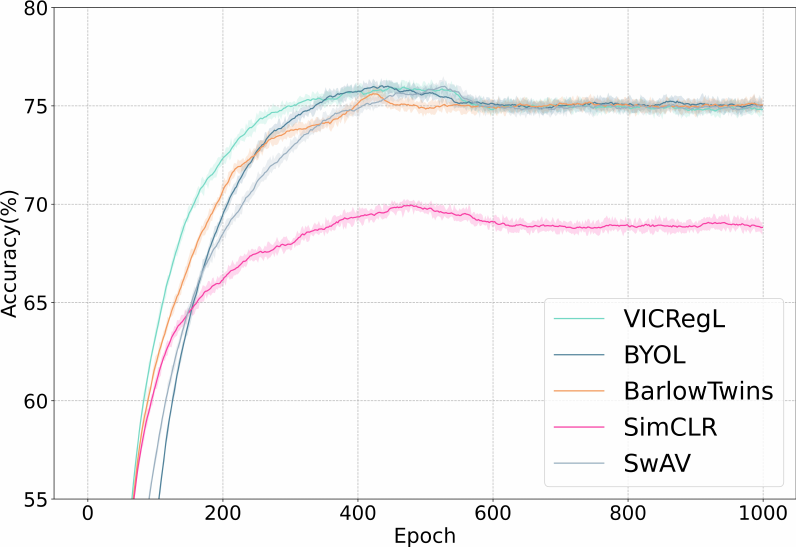}
    }
    \subfigure[ImageNet]{
    \includegraphics[width=0.3\textwidth]{ 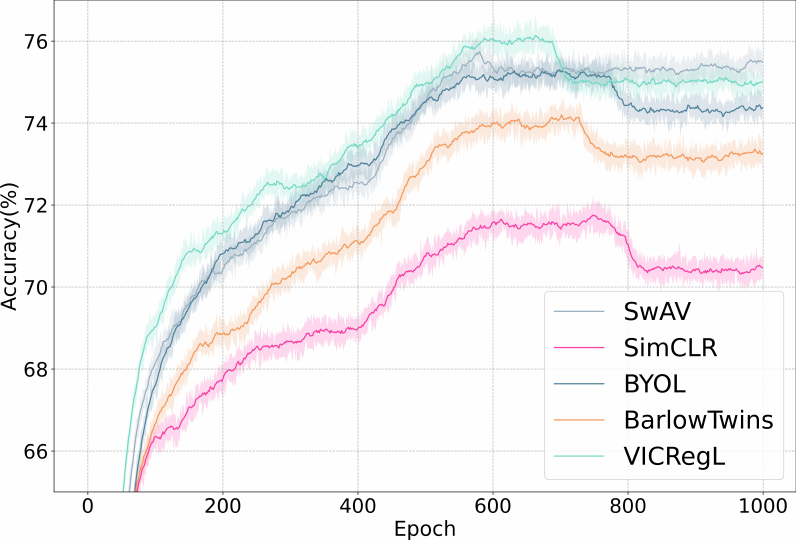}
    }
    \caption{The curves of test accuracy versus training epoch for different SSL methods and datasets. The results from (a) - (e) are based on early-layer output while results from (f) - (j) are based on last-layer output. Each result is the average of 5 independent experiments.}
    \label{Fig.main1}
\end{figure*}

\section{Preliminaries}\label{pre3435}
In the training phase, self-supervised learning (SSL) \cite{chen2020simple,wang2020understanding} projects the original input samples from the raw space into a feature space by a feature extractor $f$ and further projects them from the feature space into an embedding space through a projection head $f_{ph}$. Given a mini-batch of training data ${X_{tr}} = \{ {{x_i}} \}_{i = 1}^N$, where ${{x_i}}$ denotes the $i$-th sample and $N$ represents the number of samples, we apply stochastic data augmentation, e.g., random crop, to transform each sample $x_i$ into two augmented views $x^1_i$ and $x^2_i$. Thus, we obtain an augmented dataset $X_{tr}^{aug} = \{ {x_1^1,x_1^2,...,x_N^1,x_N^2} \}$, which can also be written as $X_{tr}^{aug} = \{ {{}^1X_{tr}^{aug},{}^2X_{tr}^{aug}} \}$, where ${}^iX_{tr}^{aug} = \{ {x_1^i,...,x_N^i} \}$ and $i \in \{ {1,2} \}$. The samples in ${X_{tr}}$ are considered as ancestors of those in $X_{{tr}}^{aug}$. For each sample in the augmented dataset, we have $z_i^j = f_{ph}(f( {x_i^j} ))$, $z_i^j = \frac{z_i^j}{{\| {z_i^j} \|}_2}$, $i \in \{ {1,2,...,N} \}$, $j \in \{ {1,2} \}$.

Existing SSL methods frameworks based on contrastive and non-contrastive frameworks mainly consist of two parts: alignment and constraint. The alignment part aims to make the features of different augmented samples from the same ancestor as similar as possible, while the constraint part imposes additional prior knowledge on the data distribution, parameter update mode, features, etc. during training. Therefore, we can express the SSL methods using a unified framework, which can be presented as follows:
\begin{equation} \label{cvdskop}
\mathop {\min }\limits_{f,{f_{ph}}} {{\cal L}_{{\rm{align}}}}(X_{tr}^{aug},f,{f_{ph}}) + {{\cal L}_{{\rm{prior}}}}(X_{tr}^{aug},f,{f_{ph}}),
\end{equation}
where ${{\cal L}_{{\rm{align}}}}$ and ${{\cal L}_{{\rm{prior}}}}$ denote the alignment and constraint losses, respectively. Next, we revisit three representative SSL methods under this framework.

SimCLR \cite{chen2020simple} is an instance-based learning method that assigns a unique class to each sample during training. As shown in Theorem 1 of \cite{wang2020understanding}, SimCLR constrains the data distribution to be uniform. BYOL \cite{grill2020bootstrap} is a SSL method that ignores negative samples. BYOL constrains the gradient backpropagation based on the training data. Barlow Twins \cite{zbontar2021barlow} is a SSL method that does not need negative samples, gradient stopping, or asymmetric networks. Barlow Twins constrains the representation by decorrelating its vector components. What these methods have in common is that they all impose constraints in the feature space that augmented samples with the same ancestry be similar to each other. The difference is that these methods tend to impose different a priori constraints on the distribution of augmented samples in the feature space. In summary, these methods can be harmonized into Equation \ref{cvdskop} so that the soundness of Equation \ref{cvdskop} can be verified.

\section{Problem Formulation and Empirical Analysis} \label{sec:me}

Deep neural networks usually have far more parameters than training samples, which makes them prone to memorize the statistical characteristics of the training data rather than finding a solution with generalization ability \cite{novak2018sensitivity}. A key issue in self-supervised learning methods is how to improve the model’s generalization ability and avoid overfitting the training dataset. However, this issue lacks in-depth research in the field of self-supervised learning, mainly facing one challenges: how to effectively evaluate whether the model is overfitting. In this section, we address this challenge by conducting the following work: 1) We verify through experiments that self-supervised learning methods indeed overfit the training dataset; 2) We discover that coding rate reduction can serve as an indicator to evaluate the overfitting problem of the model.


\subsection{Overfitting Evidence}
\label{sec:4.1}






To explore the overfitting problem in SSL, we conducted experiments on various self-supervised learning methods and datasets. We selected five representative self-supervised learning method, e.g., SimCLR, BYOL, Barlow Twins, SwAV \cite{swav}, and VICRegL \cite{Vicregl}, three small-scale classification datasets, e.g., CIFAR-10 \cite{cifar10}, CIFAR-100 \cite{cifar10}, Tiny-ImageNet \cite{leTinyImagenetVisual2015}, and two large-scale classification datasets, e.g., ImageNet \cite{krizhevsky2012imagenet} and ImageNet100 \cite{tianContrastiveMultiviewCoding2020}.
Our experimental process is as follows: 1) During the training process, we save the network parameters of the model every 5 epoches; 2) For each saved model, we utilize both early-layer and last-layer outputs to obtain two types of features. Subsequently, we leverage the label information from the training data to train two distinct types of classifiers. 3) Based on the obtained classifiers, validate all saved models on the test dataset and record the test accuracy.

Figure \ref{Fig.main1} shows the curves of test accuracy versus training epoch number for each self-supervised learning method on each dataset. Meanwhile, Figure \ref{Fig.main1} depicts the mean outcome of five independent experiments, with the surrounding shaded region illustrating the variance. We observe that the accuracy based on last-layer testing first increases, then decreases after reaching the highest value, and finally stabilizes, while the accuracy based on early-layer testing first increases and then stabilizes. Therefore, these indicate that the feature discriminability of the last-layer output first increases, then decreases upon reaching its optimum, and subsequently remains stable. Meanwhile, the feature discriminability of the early-layer output  first increases and maintains stability without change upon reaching its optimum. 

We also observe that the highest test accuracy for the last-layer features during the training process surpasses that achieved by the last-layer features at the end of training. Concurrently, the peak test accuracy of the early-layer features throughout the training process mirrors the test accuracy of the early-layer features at the end of the training, indicating consistent performance. Therefore, these findings suggest that the efficacy of the feature extractor based on early-layer features remains stable once optimal performance is achieved, regardless of additional training epochs. In contrast, the performance of the feature extractor based on last-layer features declines upon reaching its peak, before eventually stabilizing.

In conclusion, Figure \ref{Fig.main1} indicates that SSL methods indeed experience overfitting from the perspective of test accuracy. Also, Figure \ref{Fig.main1} concurs that during the initial stages of training, the model's early layer and last layer do not exhibit overfitting issues. However, at a certain point in time, the model's last layer begins to overfit, while the early layer remains unaffected by overfitting issues. Therefore, we can draw the following: 
\begin{itemize}
    \item Generalizing features are learned in early layers for all epochs, and memorization occurs abruptly in later layers and epochs.
\end{itemize}

\begin{figure*}
    \centering
    \subfigure[CIFAR-10]{
    \includegraphics[width=0.3\textwidth]{ 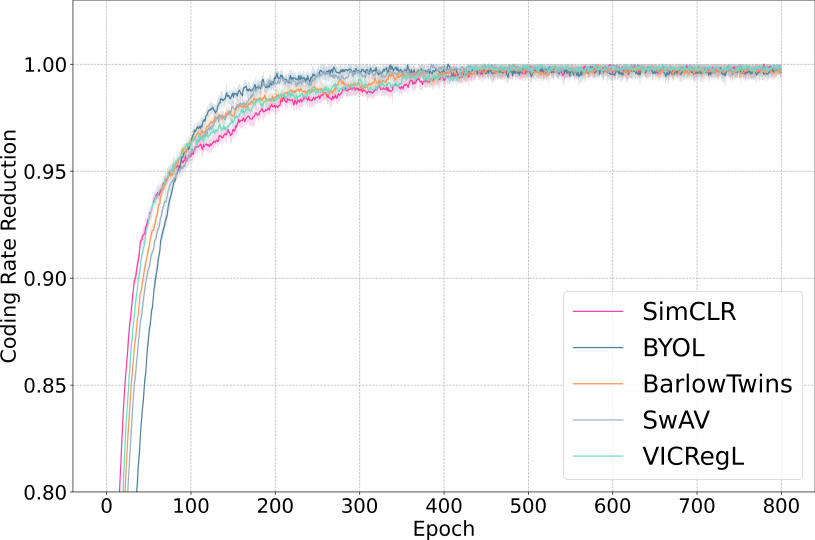}
    }
    \subfigure[CIFAR-100]{
    \includegraphics[width=0.3\textwidth]{ 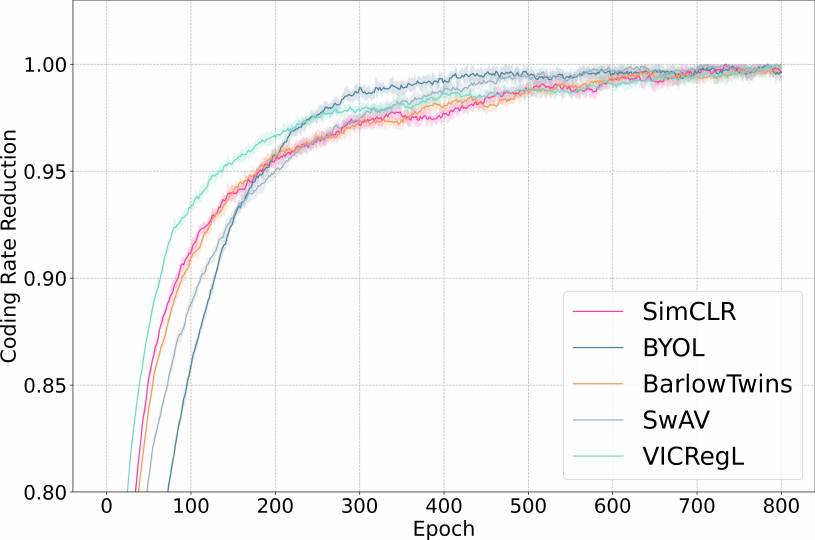}
    }
    \subfigure[Tiny-ImageNet]{
    \includegraphics[width=0.3\textwidth]{ 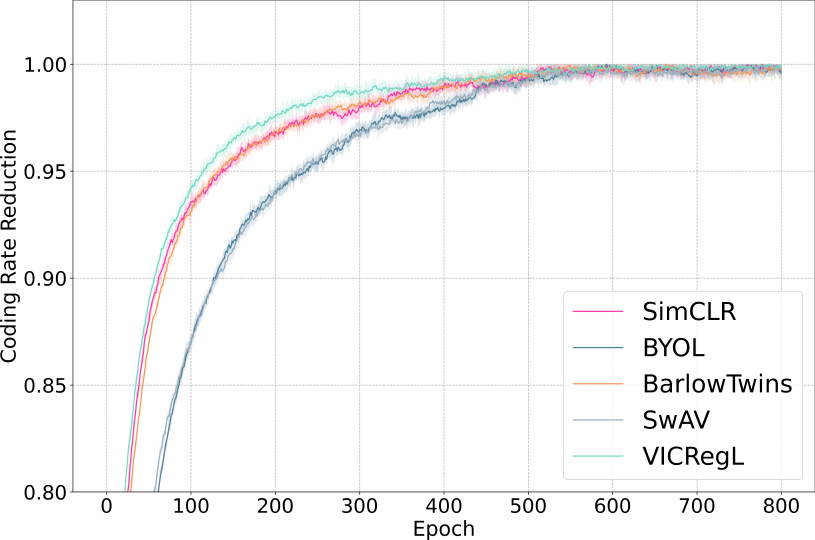}
    }
    \subfigure[ImageNet-100]{
    \includegraphics[width=0.3\textwidth]{ 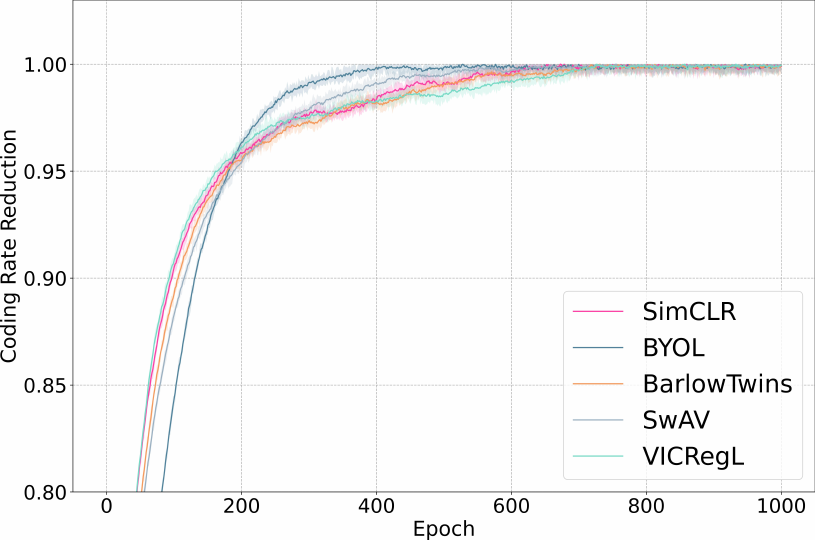}
    }
    \subfigure[ImageNet]{
    \includegraphics[width=0.3\textwidth]{ 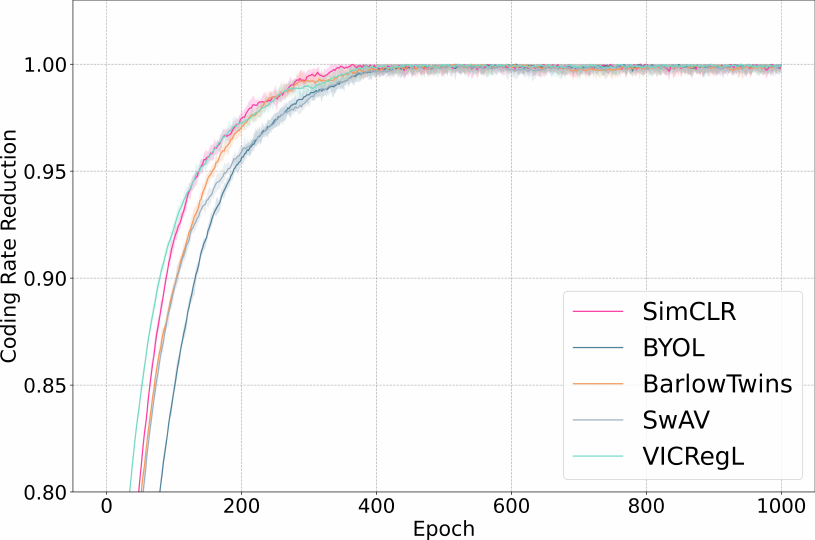}
    }
    \\
    \subfigure[CIFAR-10]{
    \includegraphics[width=0.3\textwidth]{ imgs/fig3/cifar10_crr_middle.pdf}
    }
    \subfigure[CIFAR-100]{
    \includegraphics[width=0.3\textwidth]{ 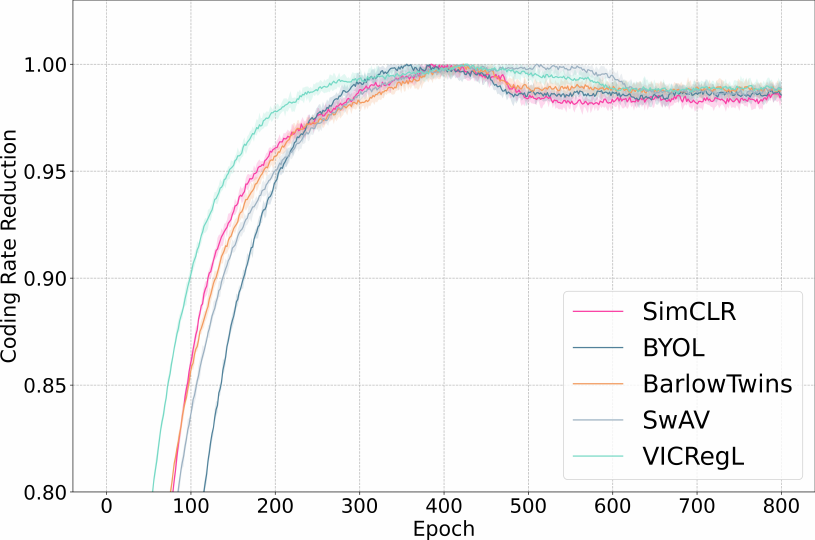}
    }
    \subfigure[Tiny-ImageNet]{
    \includegraphics[width=0.3\textwidth]{ 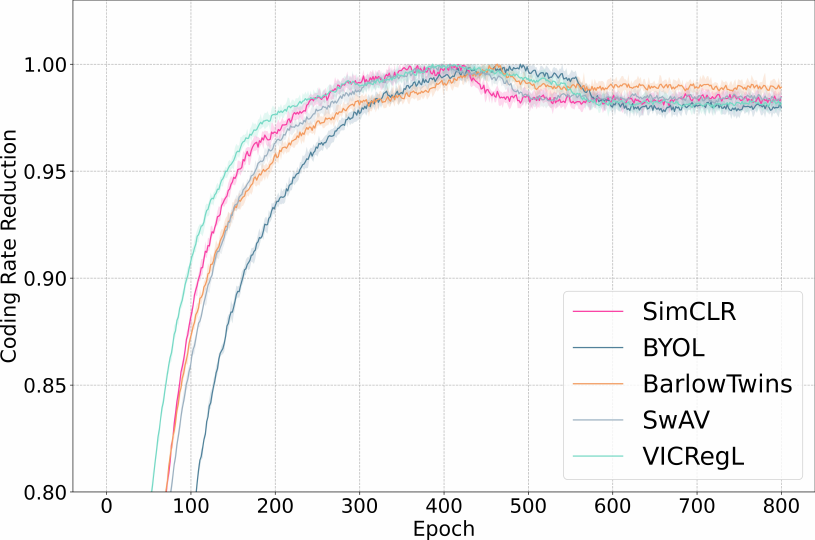}
    }
    \subfigure[ImageNet-100]{
    \includegraphics[width=0.3\textwidth]{ 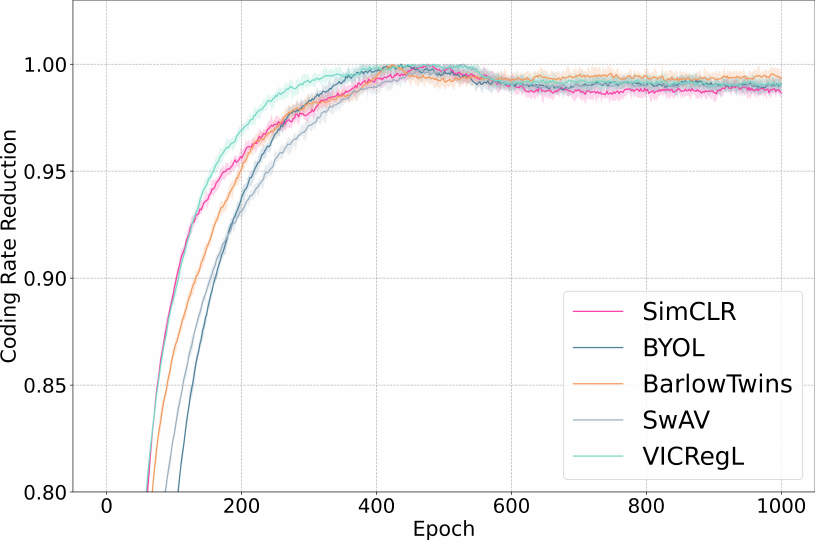}
    }
    \subfigure[ImageNet]{
    \includegraphics[width=0.3\textwidth]{ 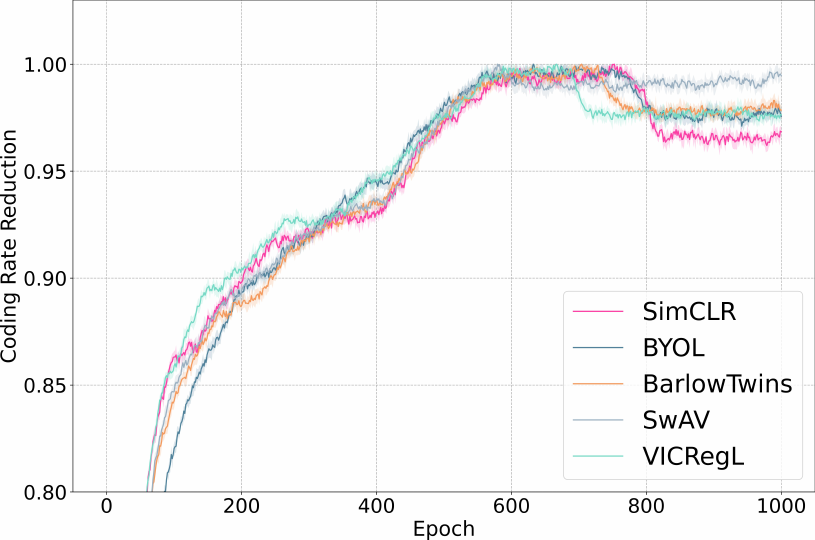}
    }
    \caption{The curves of coding rate reduction versus training epoch for different SSL methods and datasets. The results from (a) - (e) are based on early-layer output while results from (f) - (j) are based on last-layer output. Each result is the average of 5 independent experiments.}
    \label{Fig.main3}
\end{figure*}

\subsection{Quantitative Evidence}


The subsection above highlights the need to reveal overfitting in SSL methods, which currently require storing model parameters at different time points during training and continuously validating these saved models. This process incurs high storage and computational costs. In real-world scenarios, the frequent absence of labeled data precludes the availability of test datasets, thereby obstructing the detection of overfitting through concurrent training and testing. Therefore, there is a need to find a way or metric that can directly reflect overfitting during the training process, making it more efficient and meaningful without the need for multiple rounds of storage and testing. 

From information theory, we can deduce that the more task-relevant information a feature contains, the more it generalizes. This guides us to measure overfitting by quantifying the effective information content in the features. Empirically, a feature representation with generalization properties should satisfy three criteria: 1) Representations from different modes or classes should be highly uncorrelated; 2) Representations from the same mode or class should be relatively correlated; 3) Representations should have as large a dimension as possible to cover all the modes or classes and be variant for the same mode or class. These insights provide a basis for us to use coding rate reduction \cite{chan2022redunet} to quantify the feature information. Before that, we briefly introduce coding rate reduction.




More precisely, given a dataset $X = \{ {{x_1},...,{x_n}} \}$, we input it into a feature extractor $f$ to obtain its representation $Z = \{ {{z_1},...,{z_n}} \}$. To encode $Z$ with an accuracy no lower than $\varepsilon$, the required number of bits can be expressed as: $\mathcal{L}(Z,\varepsilon ) \buildrel\textstyle.\over= (\textstyle{\frac{m + n}{2}})\log \det (I + \textstyle{\frac{m} {n\varepsilon ^2}}Z Z^T)$, where $m$ is the dimension of the feature vector. Therefore, the compactness of $Z$ can be measured by the average encoding length per sample feature, i.e., the coding rate under the $\varepsilon$ constraint, which can be expressed as:
\begin{equation}
    R(Z,\varepsilon ) \buildrel\textstyle.\over= \frac{1}{2}\log \det (I + \frac{m}{n{\varepsilon ^2}}Z{Z^T}) \approx {\rm{tr}}(\frac{m}{{2n{\varepsilon ^2}}}Z{Z^T})
\end{equation}
In general, the features of multi-class data $Z$ may belong to multiple low-dimensional subspaces. To evaluate the coding rate of this mixed data more accurately, we partition $Z$ into multiple subsets: $Z = {Z^1} \cup ,...,{Z^k}$, where $Z^i$ denotes the set of sample features with the same class label. Let $\Pi  = \{ {\Pi ^j} \in {{\mathbbm{R}}^{n \times n}}\} _{j = 1}^k$ be a set of diagonal matrices whose diagonal entries encode the membership of $n$ samples in $k$ classes. More specifically, ${\Pi ^j}(i,i)$ represents the probability that the sample feature $z_i$ belongs to the subset $Z_j$. Therefore, $\Pi$ belongs to a simplex: $\Omega  \buildrel\textstyle.\over= \left\{ {\Pi |{\Pi ^j} \ge 0,{\Pi ^1} + ... + {\Pi ^k} = I} \right\}$. Based on $\Pi$, the average number of bits per sample (coding rate) is defined as:
\begin{equation}
\resizebox{0.92\linewidth}{!}{
$
\begin{array}{l}  
{R_c}(Z,\varepsilon |\Pi ) = \sum\limits_{j = 1}^k {\frac{{{\rm{tr}}({\Pi ^j})}}{{2n}}} \log \det (I + \frac{m}{{{\rm{tr}}({\Pi ^j}){\varepsilon ^2}}}Z{\Pi ^j}{Z^T})\\
 \approx \sum\limits_{j = 1}^k {{\rm{tr}}(\frac{m}{{2n{\varepsilon ^2}}}Z{\Pi ^j}{Z^T})} 
\end{array}
$
}
\end{equation}
Based on $R(Z,\varepsilon )$ and ${R_c}(Z,\varepsilon |\Pi )$, the coding rate reduction is defined as:
\begin{equation}\label{mcr123}
    \Delta R(Z,\Pi ,\varepsilon ) = R(Z,\varepsilon ) - {R_c}(Z,\varepsilon |\Pi )
\end{equation}
Then, from Theorem 1 in \cite{chan2022redunet}, we can see that when $Z$ is obtained by maximizing $\Delta R(Z,\Pi ,\varepsilon )$, then $Z$ satisfies the three criteria mentioned above.

Next, we evaluate the proposed measure of coding rate reduction $\Delta R(Z,\Pi ,\varepsilon )$ on different datasets and methods using Equation \ref{mcr123}. 
Given $X_{tr}^{aug}$, we can obtain the early layer and last layer features for each sample. Since SSL methods are trained in an instance manner, we can consider constraining the value of $k$ in Equation \ref{mcr123} to be of size $2N$, which represents the number of training samples in $X_{tr}^{aug}$. Thus, we have: $\Pi  = \{ {\Pi ^j} \in {{\mathbbm{R}}^{(2N - 1) \times (2N - 1)}}\} _{j = 1}^{2N}$, where ${\Pi ^j}(i,i) = {z_i}^{\rm{T}}{z_j}/\sum\nolimits_{k = 1,k \ne j}^{k = 2N} {{z_k}^{\rm{T}}{z_j}} $ represents the probability that the $i$-th sample $z_i$ and the $j$-th sample $z_j$ in ${}^*Z_{tr}^{aug}$ belong to the same class and ${}^*Z_{tr}^{aug}$ means that the sample features in ${}^*Z_{tr}^{aug}$ are calculated based on $f_{e-l}^*$. Meanwhile, we have $\Pi^1(i-1,i-1)+,...,\Pi^{i-1}(i-1,i-1)+\Pi^{i+1}(i,i)+,...,+\Pi^{2n}(i,i)=1$.

Finally, the trend of $\Delta R(Z, \Pi, \varepsilon)$ changing with epochs in the training phase is shown in Figure \ref{Fig.main3}. Meanwhile, Figure \ref{Fig.main3} depicts the mean outcome of five independent experiments, with the surrounding shaded region illustrating the variance, and the value recorded is normalized by dividing by the maximum value. We observe that, whether based on early layer features or last layer features, the trend of $\Delta R(Z,\Pi ,\varepsilon )$'s change is almost consistent with the trend of test accuracy changes. 

Specifically, we observe that the model with the highest test accuracy also has the highest $\Delta R(Z,\Pi ,\varepsilon )$, which suggests that the features contain more task-relevant information when the model generalizes better. Moreover, we notice that $\Delta R(Z,\Pi ,\varepsilon )$ based on last layer features first increases, then decreases after reaching a peak, and finally stabilizes; while $\Delta R(Z,\Pi ,\varepsilon )$ based on early layer features first increases, then remains stable. This implies that $\Delta R(Z,\Pi ,\varepsilon )$ reflects the generalization trend of the model, where lower values imply more severe overfitting. Therefore, based on the change in $\Delta R(Z, \Pi, \varepsilon)$, we can also conclude from the coding rate reduction perspective that: 
\begin{itemize}
    \item Early layers learn generalizable representations, while the last layer overfits to training data at the end of training.
\end{itemize}


\begin{figure*}[t]
    \centering
    \includegraphics[width=0.95\textwidth]{ 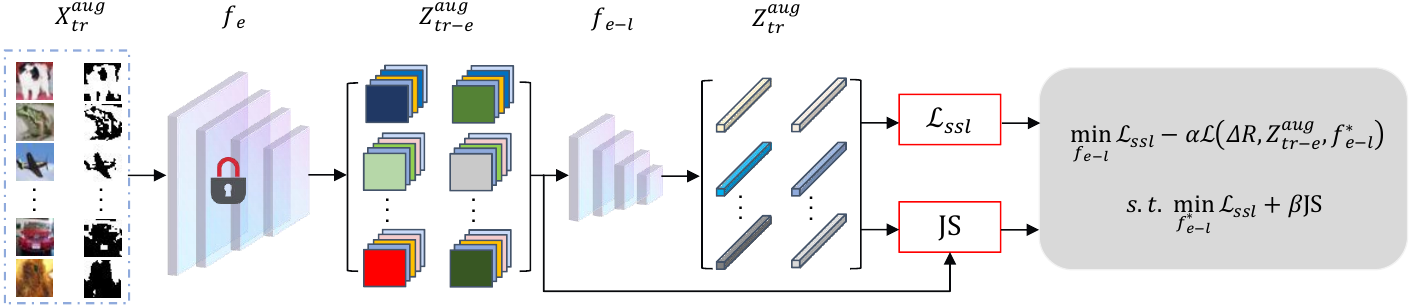}
    \caption{The learning framework of the proposed UMM. Given a pre-trained SSL model, UMM aims to fine-tune module $f_{e-l}$ while keeping module $f_{e}$ frozen. The learning objective of UMM is a quadratic optimization problem.}
    \label{qww:UMM_frame}
\end{figure*}

\section{Methodology}

In this section, we first introduce the key idea for designing the method. Then, we give the details of the propose undoing memorization mechanism which aims to mitigate the overfitting problem of pre-trained self-supervised learning methods on their training data, and the learning framework is shown in Figure \ref{qww:UMM_frame}. At last, we provide a causal analysis of the proposed mechanism based on causal theory and justify its rationality.

\subsection{Key Idea for Designing the Proposed Method}
\label{kidddg}

The understanding of ``generalization'' can be divided into two aspects: one is achieving the best performance with a given model while keeping other training factors constant, which falls into the category of exploring the model's optimal attributes; the other involves changing factors such as the model structure, while keeping training factors constant, to construct a model that most perfectly approximates the ``god model'' or ``ground truth model'', belonging to the category of minimizing modeling error. The ``generalization'' discussed in this paper pertains to the first aspect. Because we address the issue of overfitting. That is, for a model, it concerns whether there's overfitting to the training data during the training phase. If overfitting occurs, then the model is not the best possible under the training scenario.

From the discussion in the above section, it follows that models based on early-layer and those based on last-layer, due to structural differences, naturally exhibit performance discrepancies. Meanwhile, models of early-layer possess the first aspect of ``generalization'', whereas models of the last-layer, although having the second aspect of ``generalization'', have been experimentally validated to lack the first aspect. We aim to endow models of the last-layer with the first aspect of ``generalization''. Since models of early-layer lack the second aspect of ``generalization'' (This can be inferred from Figure \ref{Fig.main1}-\ref{Fig.main3}), directly using features from early-layer models to promote the recovery of the first aspect of ``generalization'' in last-layer models might damage their second aspect of ``generalization''. 

Therefore, we propose using the distribution properties of features from early-layer as prior knowledge to constrain the feature distribution of last-layer models, without involving the direct interaction between the features of different layers. At the same time, distribution information can be regarded as spatial location information, which does not involve specific semantic information in the sample features. We can use JS-divergence to achieve this, because that calculating JS-divergence is based on the probability density of the early-layer feature, not the feature itself. The rationale is that the feature distribution can reflect the information entropy contained in the data. When facing the same task, regardless of how the data changes (different features can indicate different data), the quantity of information related to the task remains constant. We aim to transmit task-related information quantity by constraining distribution.


\subsection{Undoing Memorization Mechanism}\label{umm;qw}

Based on the aforementioned discussion, we propose the Undoing Memorization Mechanism (UMM), aimed at enabling the last-layer to possess the first aspect of ``generalization''. As known from Section \ref{sec:me}, the reduction of coding rate can serve as a monitor for the first aspect of ``generalization''. Therefore, combining the discussions from Subsection \ref{kidddg}, UMM suggests using the feature distribution information of early-layer to maximize the coding rate reduction of the last layer, i.e., UMM first adjusts the coding rate reduction using the JS-divergence between the feature distributions of different layers, and then uses the coding rate reduction to regulate the feature learning process of SSL. We did not directly regulate the feature learning of SSL by minimizing the JS-divergence of feature distributions across different layers or maximizing the coding rate reduction. Our rationale is: 1) Directly minimizing JS-divergence might lead SSL's objective function to transition from one local optimum to another, and while there could be multiple local optima, we seek the one that mitigates overfitting, which UMM can directly optimize for. 2) The value of coding rate reduction may be influenced by multiple factors, with overfitting being one of them. Assuming this factor remains constant, the maximum achievable value of coding rate reduction might also be invariant. Only by addressing the overfitting factor can the value of coding rate reduction be further increased.


To implement the UMM, we begin by introducing a Continuous Piecewise Affine (CPA) mapping. This mapping is employed to model the transformation of feature distributions, both from the input space to the early layer and from the input space to the final layer, within a pre-trained SSL method. A CPA mapping is defined as a first-order spline $S$ that maps each region $w \in \Omega$ of a partition $\Omega$ to the manifold/feature space, e.g., $S(x) = \sum\nolimits_{w \in \Omega } {({A_w}x + {b_w}){\mathbbm{1}_{x \in w}}} $ with ${A_w}$ and ${b_w}$ as per-region slope and offset parameters. $\Omega$ is generated as follows: the input space $\mathcal{X}$ is first divided into different regions, and then these regions are used as elements to form a finite set, which is said to be a partition. When the input space is characterized by a density distribution, this density is transformed by the CPA mapping and resides on the surface of the CPA manifold. Based on these considerations, the following conclusions can be drawn:

\begin{lemma}
    Denote the volume of a region $w \in \Omega $ as ${\rm{vol}}(w)$ and the volume of the affinely transformed region $S(w)$ as ${\rm{vol}}(S(w))$, we have:
    \begin{equation}
    {\rm{vol}}(S(w)) = {\rm{vol}}(w)\sqrt {\det (A_w^{\rm{T}}{A_w})}.
    \end{equation}
\end{lemma}

\begin{theorem}\label{theo:pd}
    Denote the probability density of the input space as $p(x)$ and the probability density generated by $S$ as $p_S(z)$, where $z=S(x)$, we have:
    \begin{equation}
        {p_S}(z) = \sum\nolimits_{w \in \Omega } {\frac{{p({{(A_w^{\rm{T}}{A_w})}^{ - 1}}A_w^{\rm{T}}(z - {b_w}))}}{{\sqrt {\det (A_w^{\rm{T}}{A_w})} }}} {\mathbbm{1}_{x \in w}}
    \end{equation}
\end{theorem}

For the detailed proof, please refer to Appendix. Note that the pre-trained $f$ consists of affine transformations, ReLU operators, absolute values, and max-pooling, it can be approximately viewed as CPA mapping with only one region, thus, we have:
\begin{equation}\label{vbhjmm}
    f\left( x \right) \approx {A_w}x + {b_w}
\end{equation}
Then, taking derivatives of both sides of Equation \ref{vbhjmm} with respect to $x$, we can obtain that ${A_w} \approx {J_f}\left( x \right)$, where ${J_f}$ is the Jacobian matrix of $f$.


At this point, we begin to elaborate on the UMM process. For the feature extractor $f$ pre-trained by SSL method, we denote $f = {f_{e-l}} \cdot {f_e}$, where ${f_e}$ represents the part of $f$ from the input layer to the early layer, and $f_{e-l}$ represents the part of $f$ from the early layer to the last layer. Then, UMM tunes only the parameters of the layers after the early layer of $f$, e.g., $f_{e-l}$. The training dataset for training UMM is consistent with the training set used for pre-training $f$, e.g., $X_{tr}^{aug}$. The early layer feature representation $Z_{tr - e}^{aug} = \{ {z_{{1^e}}^1,z_{{1^e}}^2,...,z_{{N^e}}^1,z_{{N^e}}^2} \}$ is obtained by inputting $X_{tr}^{aug}$ into ${f_e}$, where $z_{{i^e}}^j = {f_e}(x_i^j)$. The last layer feature representation $Z_{tr}^{aug} = \{ {z_{{1}}^1,z_{{1}}^2,...,z_{{N}}^1,z_{{N}}^2} \}$ is obtained by inputting $Z_{tr - e}^{aug}$ into $f_{e-l}$, where $z_{{i}}^j = {f_{e-l}}(z_{i^e}^j)$. Without loss of generality, we assume that $Z_{tr - e}^{aug}$ follows a uniform distribution, e.g. $p(z_{{i^e}}^j) = 1/2N$, this is because that $f_e$ is fixed in this stage. Then, by Theorem \ref{theo:pd}, we have:$p(z_i^j) = {(2N)^{ - 1}}{(\det ({J_{{f_{e-l}}}}{(z_{{i^e}}^j)^{\rm{T}}}{J_{{f_{e-l}}}}(z_{{i^e}}^j)))^{ - 1/2}}$. Finally, the objective function of UMM can be expressed as:
\begin{equation}\label{umm:eq}
\begin{array}{l}
\mathop {\min }\limits_{{f_{e - l}}} {{\mathcal{L}}_{ssl}}(Z_{tr - e}^{aug},f_{e - l}^*) - \alpha {\mathcal{L}}(\Delta R,Z_{tr - e}^{aug},f_{e - l}^*)\\
s.t.\; f_{e - l}^*  \gets {f_{e - l}} - \lambda {\nabla _{{f_{e - l}}}}\{ {{\mathcal{L}}_{ssl}}(Z_{tr - e}^{aug},f_{e - l}) \\ 
\quad\quad\quad\quad\quad\quad\quad\quad\quad\quad+ \beta {\rm{JS}}(p(Z_{tr - e}^{aug})|p(Z_{tr}^{aug}))\} 
\end{array}
\end{equation}
where ${\mathcal{L}}(\Delta R,Z_{tr - e}^{aug},f_{e - l}^*)=\Delta R({}^*Z_{tr}^{aug},\Pi ,\varepsilon )$, ${{\mathcal{L}}_{ssl}}( \cdot ) = {{\mathcal{L}}_{{\rm{align}}}}( \cdot ) + {{\mathcal{L}}_{{\rm{prior}}}}( \cdot )$, \rm{JS} is the JS-divergence, $\alpha ,\beta$ are the hyper-parameters, and $\lambda$ is the learning rate. Algorithm \ref{alg1} is the pseudo-code of UMM.

Note that self-supervised learning methods can be regarded as instance-based methods, in other words, self-supervised learning treats each sample as a separate class during the modeling process. Therefore, when calculating ${\mathcal{L}}(\Delta R,Z_{tr - e}^{aug},f_{e - l}^*)$, we let $\Pi  = \{ {\Pi ^j} \in {{\mathbbm{R}}^{(2N - 1) \times (2N - 1)}}\} _{j = 1}^{2N}$, where ${\Pi ^j}(i,i) = {z_i}^{\rm{T}}{z_j}/\sum\nolimits_{k = 1,k \ne j}^{k = 2N} {{z_k}^{\rm{T}}{z_j}} $ represents the probability that the $i$-th sample $z_i$ and the $j$-th sample $z_j$ in ${}^*Z_{tr}^{aug}$ belong to the same class and ${}^*Z_{tr}^{aug}$ means that the sample
features in ${}^*Z_{tr}^{aug}$ are calculated based on $f_{e-l}^*$. Meanwhile, ${R_c}({}^*Z_{tr}^{aug},\varepsilon |\Pi ) \approx \sum\nolimits_{j = 1}^{2N} {{\rm{tr}}({\textstyle{{m} \over {(2N-1){\varepsilon ^2}}}}{}^*Z_{tr}^{aug}({z_j}){\Pi}^j {}^*Z_{tr}^{aug}{{({z_j})}^{\rm{T}}})} $, where ${{}^*Z_{tr}^{aug}({z_j})}$ denotes removing the $j$-th sample $z_j$ from $Z_{tr}^{aug}$ and $m$ represents the dimension of the samples in ${}^*Z_{tr}^{aug}$. From the constraint condition in Equation \ref{umm:eq}, we can obtain that $f_{e - l}^*$ is a function of $f_{e - l}$. Thus, minimizing the objective function in Equation \ref{umm:eq} can be regarded as using the quadratic gradient to update $f_{e - l}$. Also, the constraint condition in Equation \ref{umm:eq} is equivalent to:
\begin{equation}\label{umm:evshdhq}
\mathop {\min }\limits_{{f_{e - l}}}{{\cal L}_{ssl}}(Z_{tr - e}^{aug},{f_{e - l}},{f_{ph}}) + \beta {\rm{JS}}(p(Z_{tr - e}^{aug})|p(Z_{tr}^{aug}))
\end{equation}
Thus, Equation \ref{umm:eq} can be regarded as a bi-level optimization process.

Upon closer examination of Equation \ref{umm:eq} from a more granular perspective, it can be observed that for any $f_{e - l}$, the optimization process of $f_{e - l}^*$ is controlled by $f_{e - l}$, i.e., $f_{e - l}^*$ changes as $f_{e - l}$ changes. Also, we can see that $f_{e - l}^*$ is obtained by minimizing both ${{\cal L}_{ssl}}(Z_{tr - el}^{aug},{f_{e - l}},{f_{ph}})$ and ${\rm{JS}}(p(Z_{tr - el}^{aug})|p(Z_{tr}^{aug}))$. So, we can regard $f_{e - l}^*$ as having absorbed the generalization information from the early layer to the last layer. Then, the optimization process of the objective function in Equation \ref{umm:eq} can be reformulated as $f_{e - l} \gets {f_{e - l}} - \lambda {\nabla _{{f_{e - l}}}}\{ {{\mathcal{L}}_{ssl}}(Z_{tr - e}^{aug},f_{e - l}^*) - \alpha {\mathcal{L}}(\Delta R,Z_{tr - e}^{aug},f_{e - l}^*)\}$, where ${\nabla _{{f_{e - l}}}}\{ {{\mathcal{L}}_{ssl}}(Z_{tr - e}^{aug},f_{e - l}^*) - \alpha {\mathcal{L}}(\Delta R,Z_{tr - e}^{aug},f_{e - l}^*)\}$ can be regarded as taking the 2nd order derivative with respect to $f_{e - l}$. In other words, the optimization process of Equation \ref{umm:eq} can be regarded as consisting of two step, the first step is to obtain different $f_{e - l}^*$ by adjusting the value of $f_{e - l}$, and the second step is to adjust ${{\mathcal{L}}_{ssl}}(Z_{tr - e}^{aug},f_{e - l}^*) - \alpha {\mathcal{L}}(\Delta R,Z_{tr - e}^{aug},f_{e - l}^*)$ by different $f_{e - l}^*$, choosing the $f_{e - l}$ corresponding to the $f_{e - l}^*$ that minimizes ${{\mathcal{L}}_{ssl}}(Z_{tr - e}^{aug},f_{e - l}^*) - \alpha {\mathcal{L}}(\Delta R,Z_{tr - e}^{aug},f_{e - l}^*)$. This corresponds exactly to the motivation for designing the UMM, that is UMM first adjusts the coding rate reduction using the JS-divergence between the feature distributions of different layers, and then uses the coding rate reduction to regulate the feature learning process of SSL.

In the end, we can understand Equation \ref{umm:eq} from two levels. The first level (constraint condition) aims to use the features extracted by the early layer to adjust the features extracted by the last layer that face overfitting problems. The second level (objective function) further restricts the first level, aiming to constrain its behavior, that is, it should use the information contained in the early layer that can maximize the coding rate reduction of the last layer features to adjust the features extracted by the last layer. In short, Equation \ref{umm:eq} can be understood as undoing memorization of the last layer by rewinding. Meanwhile, UMM can be easily integrated with any self-supervised learning method.


\begin{algorithm}[t]
	\SetAlgoLined
	\KwIn{Training Data $D=\{x\}_{i=1}^n$; Batch Size $N$; Pre-trained Feature Extractor $f=f_{e-l} \cdot f_{e}$; Projection Head $f_{ph}$; Hyper-parameters $\alpha ,\beta ,\varepsilon$; and Learning Rate $\lambda, \gamma $}
	\KwOut{The optimal encoder: $f_{e-l}$}
	
	\For{sample batch $X_{tr}$ from $D$}{
		\# generate two augmented views\\
		${}^1X_{tr}^{aug},{}^2X_{tr}^{aug} = T(X)$, $T$ is data augmentation methods, let $X_{tr}^{aug}=\{{}^1X_{tr}^{aug},{}^2X_{tr}^{aug}\}$\\ 
		\# obtain the output embeddings of the early layer $f_{e}$ \\
		$Z_{tr-e}^{aug} = f_{e}(X_{tr}^{aug})$\\  
		\# obtain the output embeddings of the last layer $f_{e-l}$\\
        $Z_{tr}^{aug} = f_{e-l} \cdot f_{e}(X_{tr}^{aug})$\\  
		\# calculate the probability value \\
		$p(Z_{tr-e}^{aug}) = {(2N)^{ - 1}}{(\det ({J_{{f_{e-l}}}}{(Z_{tr-e}^{aug})^{\rm{T}}}{J_{{f_{e-l}}}}(Z_{tr-e}^{aug})))^{ - 1/2}}$\\
        $p(Z_{tr}^{aug}) = {(2N)^{ - 1}}{(\det ({J_{{f}}}{(Z_{tr}^{aug})^{\rm{T}}}{J_{{f}}}(Z_{tr}^{aug})))^{ - 1/2}}$\\      
		\# update the inner loop of Equation 6 \\
		$f_{e - l}^* = {f_{e - l}} - \lambda {\nabla _{{f_{e - l}}}}\{ {{\mathcal{L}}_{ssl}}(Z_{tr - e}^{aug},f_{e - l},{f_{ph}}) + \beta {\rm{JS}}(p(Z_{tr - e}^{aug})|p(Z_{tr}^{aug}))$\\
		\# update the outer loop of Equation 6\\
		$f_{e - l} = {f_{e - l}} - \gamma {\nabla _{{f_{e - l}}}}\{ {{\mathcal{L}}_{ssl}}(Z_{tr - e}^{aug},f_{e - l}^*,{f_{ph}}) - \alpha {\mathcal{L}}(\Delta R,Z_{tr - e}^{aug},f_{e - l}^*)\}$
	}
	\caption{The Pseudo-Code of UMM
}

 \label{alg1}
\end{algorithm}

\subsection{Theoretical Analysis from a Causal Perspective}

We perform a causal analysis of the self-supervised learning methods from the perspective of data generation. Based on \cite{2021Contrastive,2017Nonlinear}, we can obtain that the data in a realistic scenario can be regarded as obtained from the generating factors, and the feature extraction process of the data can be regarded as extracting the corresponding generating factors. In other words, when the feature representation of a sample can be obtained by an invertible transformation to the corresponding generating factor, then the feature representation of the sample and the generating factor can be considered equivalent. Formally, let $Z$ denote the set of feature representations of the original dataset $X$, let $S_{r}$ denote the task-relevant information contained in $Z$, i.e., generalizable knowledge or foreground related knowledge, and let $S_{ur}$ denote the task-irrelevant information contained in $Z$, i.e., background relevant knowledge. We assume that the original dataset ${X}$ is generated by a mixture function $\rm{f}$ that takes $Z$ as input, e.g., ${X}={\rm{f}}(Z)$. We apply augmentation to ${X}$ to obtain the augmented dataset ${X}^{aug}$. As shown in Section \ref{sec:me}, the features extracted by the early layer of the self-supervised learning methods based on augmentation are generalizable. Therefore, from the perspective of data generation, the augmented dataset ${X}^{aug}$ can be regarded as generated by $S_{r}$ and perturbed $S_{ur}$. We denote the perturbed $S_{ur}$ as $S_{ur}^{*}$, which can be understood as only changing some of the task-irrelevant information in $S_{ur}$.

\begin{table*}
	\centering
	\caption{ The classification accuracies of a linear classifier (linear) and a 5-nearest neighbors classifier (5-nn) with a ResNet-18 as the feature extractor. In the table, the methods with * denote the results with only $\mathcal{L}_{ssl}$ and ${\rm{JS}}(p(Z_{tr - el}^{aug})|p(Z_{tr}^{aug}))$. The methods with ** denote the results with only $\mathcal{L}_{ssl}$ and ${\mathcal{L}}(\Delta R,Z_{tr - el}^{aug},f_{e - l}^*)$.}
	\label{tab:1}
 \resizebox{\linewidth}{!}{
 
	\begin{tabular}{lcccccccc}
		\toprule
		\multirow{2.5}{*}{Method} & \multicolumn{2}{c}{CIFAR-10} & \multicolumn{2}{c}{CIFAR-100} &\multicolumn{2}{c}{STL-10} & \multicolumn{2}{c}{
		Tiny ImageNet} \\
		\cmidrule(lr){2-3} \cmidrule(lr){4-5} \cmidrule(lr){6-7} \cmidrule(lr){8-9}
		& \(\mathbf{linear}\) & \(\mathbf{5-nn}\) & \(\mathbf{linear}\) & \(\mathbf{5-nn}\) & \(\mathbf{linear}\) & \(\mathbf{5-nn}\) & \(\mathbf{linear}\) & \(\mathbf{5-nn}\)\\
       \midrule
		MoCo \cite{moco} & 91.69$\pm$0.12 & 88.66$\pm$0.14& 67.22$\pm$0.16 & 56.29$\pm$0.25 & 90.64$\pm$0.28 & 88.01$\pm$0.19 & 50.92$\pm$0.22 & 35.55$\pm$0.16 \\
		SimSiam \cite{chen2021exploring} & 91.71$\pm$0.27 & 88.65$\pm$0.17& 67.02$\pm$0.26 & 56.36$\pm$0.19 & 91.01$\pm$0.19 & 88.16$\pm$0.19 & 51.14$\pm$0.20& 35.67$\pm$0.16 \\
	    SimCLR \cite{chen2020simple} & 91.80$\pm$0.15 & 88.42$\pm$0.15 & 66.83$\pm$0.27 & 56.56$\pm$0.18 & 90.51$\pm$0.14 & 85.68$\pm$0.10 & 48.84$\pm$0.15 & 32.86$\pm$0.25 \\
		BYOL \cite{byol} & 91.93$\pm$0.22 & 89.45$\pm$0.22& 66.60$\pm$0.16 & 56.82$\pm$0.17 & 91.99$\pm$0.13 & 88.64$\pm$0.20 & 51.00$\pm$0.12 & 36.24$\pm$0.28 \\
		SwAV \cite{swav} & 91.03$\pm$0.19 & 89.52$\pm$0.24 & 66.56$\pm$0.17 & 57.01$\pm$0.25 & 90.72$\pm$0.29 & 86.24$\pm$0.26 & 52.02$\pm$0.26& 37.40$\pm$0.11\\
		Barlow Twins \cite{zbontar2021barlow} & 90.88$\pm$0.19 & 89.68$\pm$0.21 & 66.13$\pm$0.10& 56.70$\pm$0.25 & 90.38$\pm$0.13 & 87.13$\pm$0.23 & 49.78$\pm$0.26& 34.18$\pm$0.18  \\
		W-MSE \cite{chen2021exploring} & 91.99$\pm$0.12 &89.87$\pm$0.25 & 67.64$\pm$0.16 & 56.45$\pm$0.26& 91.75$\pm$0.23 & 88.59$\pm$0.15 & 49.22$\pm$0.16 & 35.44$\pm$0.10 \\
  
		RELIC v2 \cite{tomasev2022pushing}& 91.92$\pm$0.14 & 90.02$\pm$0.22 & 67.66$\pm$0.20 & 57.03$\pm$0.18 & 91.10$\pm$0.23 & 88.66$\pm$0.12 & 49.33$\pm$0.13 & 35.52$\pm$0.22\\

            LMCL \cite{chen2021large} & 91.91$\pm$0.25 & 88.52$\pm$0.29 & 67.01$\pm$0.18 & 56.86$\pm$0.14 & 90.87$\pm$0.18 & 85.91$\pm$0.25 & 49.24$\pm$0.18 & 32.88$\pm$0.13 \\
            ReSSL \cite{zheng2021ressl} & 90.20$\pm$0.16 & 88.26$\pm$0.18 & 66.79$\pm$0.12 & 53.72$\pm$0.28 & 88.25$\pm$0.14 & 86.33$\pm$0.17 & 46.60$\pm$0.18 & 32.39$\pm$0.20\\
            SSL-HSIC \cite{li2021self} & 91.95$\pm$0.14 & 89.99$\pm$0.17 & 67.23$\pm$0.26 & 57.01$\pm$0.27 & 92.09$\pm$0.20 & 88.91$\pm$0.29 & 51.37$\pm$0.15 & 36.03$\pm$0.12 \\
CorInfoMax\cite{CorInfoMax}& 91.81$\pm$0.11 & 89.85$\pm$0.13 & 67.09$\pm$0.24 & 56.92$\pm$0.23 & 91.85$\pm$0.25 & 89.99$\pm$0.24 & 51.23$\pm$0.14 & 35.98$\pm$0.09\\
MEC\cite{MEC}& 90.55$\pm$0.22 & 87.80$\pm$0.10 & 67.36$\pm$0.27 & 57.25$\pm$0.25 & 91.33$\pm$0.14 & 89.03$\pm$0.33 & 50.93$\pm$0.13 & 36.28$\pm$0.14\\
VICRegL\cite{Vicregl}& 90.99$\pm$0.13 & 88.75$\pm$0.26 & 68.03$\pm$0.32 & 57.34$\pm$0.29 & 92.12$\pm$0.26 & 90.01$\pm$0.20 & 51.52$\pm$0.13 & 36.24$\pm$0.16\\
  \hline
	
 SimSiam* &91.91$\pm$0.26 \textcolor{red}{$\uparrow 0.20$} &89.38$\pm$0.14 \textcolor{red}{$\uparrow 0.73$}&67.29$\pm$0.14 \textcolor{red}{$\uparrow 0.27$}&57.26$\pm$0.19 \textcolor{red}{$\uparrow 0.90$}&91.31$\pm$0.26 \textcolor{red}{$\uparrow 0.30$}&89.69$\pm$0.13 \textcolor{red}{$\uparrow 1.53$}&51.35$\pm$0.14 \textcolor{red}{$\uparrow 0.21$}&36.82$\pm$0.11 \textcolor{red}{$\uparrow 1.15$}\\
 SimCLR* &92.03$\pm$0.12 \textcolor{red}{$\uparrow 0.23$}&89.14$\pm$0.13 \textcolor{red}{$\uparrow 0.72$}&66.95$\pm$0.17 \textcolor{red}{$\uparrow 0.12$}&56.95$\pm$0.15 \textcolor{red}{$\uparrow 0.39$}&90.71$\pm$0.22 \textcolor{red}{$\uparrow 0.20$}&85.88$\pm$0.12 \textcolor{red}{$\uparrow 0.20$}&49.18$\pm$0.25 \textcolor{red}{$\uparrow 0.34$}&33.12$\pm$0.16 \textcolor{red}{$\uparrow 0.26$}\\
 BYOL* &92.06$\pm$0.14 \textcolor{red}{$\uparrow 0.13$}&90.02$\pm$0.13 \textcolor{red}{$\uparrow 0.57$}&66.91$\pm$0.12 \textcolor{red}{$\uparrow 0.31$}&56.92$\pm$0.15 \textcolor{red}{$\uparrow 0.10$}&92.03$\pm$0.15 \textcolor{red}{$\uparrow 0.04$}&88.81$\pm$0.21 \textcolor{red}{$\uparrow 0.17$}&51.17$\pm$0.15 \textcolor{red}{$\uparrow 0.17$}&36.42$\pm$0.24 \textcolor{red}{$\uparrow 0.18$}\\
 Barlow Twins* &90.91$\pm$0.11 \textcolor{red}{$\uparrow 0.03$}&90.02$\pm$0.16 \textcolor{red}{$\uparrow 0.34$}&66.71$\pm$0.24 \textcolor{red}{$\uparrow 0.58$}&57.02$\pm$0.15 \textcolor{red}{$\uparrow 0.32$}&90.97$\pm$0.14 \textcolor{red}{$\uparrow 0.59$}&88.61$\pm$0.14 \textcolor{red}{$\uparrow 1.48$}&50.06$\pm$0.24 \textcolor{red}{$\uparrow 0.28$}&34.31$\pm$0.16 \textcolor{red}{$\uparrow 0.13$}\\
 W-MSE* &92.01$\pm$0.11 \textcolor{red}{$\uparrow 0.02$}&90.15$\pm$0.19 \textcolor{red}{$\uparrow 0.28$}&67.72$\pm$0.25 \textcolor{red}{$\uparrow 0.08$}&56.71$\pm$0.14 \textcolor{red}{$\uparrow 0.26$}&91.92$\pm$0.13 \textcolor{red}{$\uparrow 0.17$}&88.72$\pm$0.23 \textcolor{red}{$\uparrow 0.13$}&50.42$\pm$0.16 \textcolor{red}{$\uparrow 1.20$}&35.65$\pm$0.17 \textcolor{red}{$\uparrow 0.21$}\\
 VICRegL* &91.46$\pm$0.21 \textcolor{red}{$\uparrow 0.47$}&89.24$\pm$0.13 \textcolor{red}{$\uparrow 0.49$}&68.28$\pm$0.13 \textcolor{red}{$\uparrow 0.25$}&57.91$\pm$0.10 \textcolor{red}{$\uparrow 0.57$}&92.27$\pm$0.17 \textcolor{red}{$\uparrow 0.15$}&90.10$\pm$0.13 \textcolor{red}{$\uparrow 0.09$}&51.73$\pm$0.23 \textcolor{red}{$\uparrow 0.21$}&36.71$\pm$0.23 \textcolor{red}{$\uparrow 0.47$}\\
\midrule
 SimSiam** &91.97$\pm$0.27 \textcolor{red}{$\uparrow 0.26$}&88.98$\pm$0.44 \textcolor{red}{$\uparrow 0.33$}&67.32$\pm$0.33 \textcolor{red}{$\uparrow 0.30$}&56.47$\pm$0.37 \textcolor{red}{$\uparrow 0.11$}&91.25$\pm$0.42 \textcolor{red}{$\uparrow 0.24$}&88.21$\pm$0.17 \textcolor{red}{$\uparrow 0.05$}&51.53$\pm$0.14 \textcolor{red}{$\uparrow 0.39$}&35.96$\pm$0.24 \textcolor{red}{$\uparrow 0.29$}\\
 SimCLR**  &91.95$\pm$0.32 \textcolor{red}{$\uparrow 0.15$}&88.73$\pm$0.10 \textcolor{red}{$\uparrow 0.31$}&67.23$\pm$0.28 \textcolor{red}{$\uparrow 0.40$}&56.76$\pm$0.27 \textcolor{red}{$\uparrow 0.20$}&90.59$\pm$0.20 \textcolor{red}{$\uparrow 0.08$}&85.82$\pm$0.27 \textcolor{red}{$\uparrow 0.14$}&49.13$\pm$0.43 \textcolor{red}{$\uparrow 0.29$}&33.25$\pm$0.20 \textcolor{red}{$\uparrow 0.39$}\\
 BYOL**  &92.36$\pm$0.35 \textcolor{red}{$\uparrow 0.43$}&89.54$\pm$0.43 \textcolor{red}{$\uparrow 0.09$}&67.07$\pm$0.16 \textcolor{red}{$\uparrow 0.47$}&57.17$\pm$0.44 \textcolor{red}{$\uparrow 0.35$}&92.26$\pm$0.35 \textcolor{red}{$\uparrow 0.27$}&88.71$\pm$0.14 \textcolor{red}{$\uparrow 0.07$}&51.42$\pm$0.11 \textcolor{red}{$\uparrow 0.42$}&36.69$\pm$0.41 \textcolor{red}{$\uparrow 0.45$}\\
 Barlow Twins**  &91.19$\pm$0.23 \textcolor{red}{$\uparrow 0.31$}&90.12$\pm$0.22 \textcolor{red}{$\uparrow 0.44$}&66.39$\pm$0.36 \textcolor{red}{$\uparrow 0.26$}&56.81$\pm$0.37 \textcolor{red}{$\uparrow 0.11$}&90.52$\pm$0.14 \textcolor{red}{$\uparrow 0.14$}&87.46$\pm$0.16 \textcolor{red}{$\uparrow 0.33$}&49.91$\pm$0.25 \textcolor{red}{$\uparrow 0.13$}&34.26$\pm$0.15 \textcolor{red}{$\uparrow 0.08$}\\
 W-MSE**  &92.16$\pm$0.22 \textcolor{red}{$\uparrow 0.17$}&90.07$\pm$0.26 \textcolor{red}{$\uparrow 0.20$}&67.80$\pm$0.45 \textcolor{red}{$\uparrow 0.16$}&56.51$\pm$0.32 \textcolor{red}{$\uparrow 0.06$}&92.01$\pm$0.44 \textcolor{red}{$\uparrow 0.26$}&88.77$\pm$0.21 \textcolor{red}{$\uparrow 0.18$}&49.68$\pm$0.13 \textcolor{red}{$\uparrow 0.46$}&35.62$\pm$0.48 \textcolor{red}{$\uparrow 0.18$}\\
 VICRegL** &91.16$\pm$0.15 \textcolor{red}{$\uparrow 0.17$}&89.08$\pm$0.32 \textcolor{red}{$\uparrow 0.33$}&68.33$\pm$0.20 \textcolor{red}{$\uparrow 0.30$}&57.78$\pm$0.31 \textcolor{red}{$\uparrow 0.44$}&92.55$\pm$0.23 \textcolor{red}{$\uparrow 0.43$}&90.33$\pm$0.48 \textcolor{red}{$\uparrow 0.32$}&51.84$\pm$0.25 \textcolor{red}{$\uparrow 0.32$}&36.36$\pm$0.28 \textcolor{red}{$\uparrow 0.12$}\\
\midrule
 SimSiam + UMM&92.81$\pm$0.16 \textcolor{red}{$\uparrow 1.10$}&90.66$\pm$0.12 \textcolor{red}{$\uparrow 2.01$}&68.31$\pm$0.12 \textcolor{red}{$\uparrow 1.29$}&57.69$\pm$0.16 \textcolor{red}{$\uparrow 1.33$}&92.24$\pm$0.14 \textcolor{red}{$\uparrow 1.23$}&90.05$\pm$0.17 \textcolor{red}{$\uparrow 1.89$}&52.42$\pm$0.07 \textcolor{red}{$\uparrow 1.28$}&37.12$\pm$0.31 \textcolor{red}{$\uparrow 1.45$}\\
 SimCLR + UMM&92.93$\pm$0.14 \textcolor{red}{$\uparrow 1.13$}&90.24$\pm$0.15 \textcolor{red}{$\uparrow 1.82$}&67.95$\pm$0.13 \textcolor{red}{$\uparrow 1.12$}&57.25$\pm$0.24 \textcolor{red}{$\uparrow 0.69$}&91.79$\pm$0.13 \textcolor{red}{$\uparrow 1.28$}&86.99$\pm$0.25 \textcolor{red}{$\uparrow 1.31$}&50.11$\pm$0.15 \textcolor{red}{$\uparrow 1.27$}&34.33$\pm$0.24 \textcolor{red}{$\uparrow 1.47$}\\
 BYOL + UMM&92.99$\pm$0.13 \textcolor{red}{$\uparrow 1.06$}&90.81$\pm$0.17 \textcolor{red}{$\uparrow 1.36$}&67.84$\pm$0.25 \textcolor{red}{$\uparrow 1.24$}&57.22$\pm$0.22 \textcolor{red}{$\uparrow 0.40$}&\bf 93.99$\pm$0.18 \textcolor{red}{$\uparrow 2.00$}&89.95$\pm$0.27 \textcolor{red}{$\uparrow 1.31$}&\bf52.54$\pm$0.19 \textcolor{red}{$\uparrow 1.54$}&\bf 37.61$\pm$0.10 \textcolor{red}{$\uparrow 1.37$}\\
 Barlow Twins + UMM&91.55$\pm$0.26 \textcolor{red}{$\uparrow 0.67$}&90.74$\pm$0.24 \textcolor{red}{$\uparrow 1.06$}&67.77$\pm$0.14 \textcolor{red}{$\uparrow 1.64$}&58.07$\pm$0.24 \textcolor{red}{$\uparrow 1.37$}&92.27$\pm$0.15 \textcolor{red}{$\uparrow 1.89$}&89.69$\pm$0.16 \textcolor{red}{$\uparrow 2.56$}&50.82$\pm$0.13 \textcolor{red}{$\uparrow 1.04$}&35.44$\pm$0.23 \textcolor{red}{$\uparrow 1.26$}\\
 W-MSE + UMM&93.06$\pm$0.13 \textcolor{red}{$\uparrow 1.07$}&91.04$\pm$0.18 \textcolor{red}{$\uparrow 1.17$}&68.62$\pm$0.14 \textcolor{red}{$\uparrow 0.98$}&57.41$\pm$0.14 \textcolor{red}{$\uparrow 0.96$}&92.98$\pm$0.09 \textcolor{red}{$\uparrow 1.23$}&89.35$\pm$0.25 \textcolor{red}{$\uparrow 0.76$}&51.26$\pm$0.21 \textcolor{red}{$\uparrow 2.04$}&36.38$\pm$0.33 \textcolor{red}{$\uparrow 0.94$}\\
 VICRegL + UMM&92.77$\pm$0.15 \textcolor{red}{$\uparrow 1.78$}&90.74$\pm$0.15 \textcolor{red}{$\uparrow 1.99$}&\bf 69.21$\pm$0.27 \textcolor{red}{$\uparrow 1.18$}&\bf 58.42$\pm$0.21 \textcolor{red}{$\uparrow 1.08$}&92.48$\pm$0.19 \textcolor{red}{$\uparrow 0.36$}&\bf 91.17$\pm$0.29 \textcolor{red}{$\uparrow 1.16$}&52.45$\pm$0.17 \textcolor{red}{$\uparrow 0.93$}&37.45$\pm$0.12 \textcolor{red}{$\uparrow 1.21$}\\
 RELIC v2 + UMM&\bf 93.42$\pm$0.26 \textcolor{red}{$\uparrow 1.50$}&\bf 91.19$\pm$0.21 \textcolor{red}{$\uparrow 1.17$}&68.98$\pm$0.17 \textcolor{red}{$\uparrow 1.32$}&58.35$\pm$0.39 \textcolor{red}{$\uparrow 1.32$}&92.14$\pm$0.23 \textcolor{red}{$\uparrow 1.04$}&90.23$\pm$0.24 \textcolor{red}{$\uparrow 1.57$}&50.58$\pm$0.17 \textcolor{red}{$\uparrow 1.25$}&36.70$\pm$0.48 \textcolor{red}{$\uparrow 1.18$}\\
  \bottomrule
	\end{tabular}
}
\label{tab:exp1_1}
\end{table*}

\begin{figure}[tp]
    \centering   \includegraphics[width=0.315\textwidth]{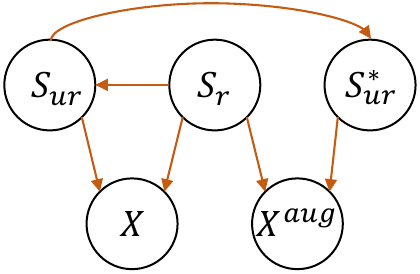}
    \caption{SCM for data generation process. $X$ and $X^{aug}$ represent the original and augmented datasets, respectively. $S_{r}$ denotes the task-relevant information, and $S_{ur}$ denotes the task-irrelevant information. $S_{ur}^{*}$ denote the perturbed $S_{ur}$, which can be understood as only changing some of the task-irrelevant information in $S_{ur}$.}
    \label{Fig.11subqw.1}
\end{figure}

To this end, we use a structural causal model (SCM) to provide a causal explanation for the data generation process and to describe the (allowed) causal relationships between latent variables. This allows us to interpret the augmented dataset $X^{aug}$ in terms of counterfactuals. Figure \ref{Fig.11subqw.1} shows the SCM that we obtain based on our assumptions and understanding of the data generation process. In this SCM, we assume that $S_{r}$ may affect $S_{ur}$, because different foregrounds tend to have characteristic background information, e.g., for a picture of a fish, the background is likely to be water. Therefore, we formalize their relationship as follows: ${S_r}: = {{\rm{f}}_r}({U_r}),{S_{ur}}: = {{\rm{f}}_u}({S_r},{U_{ur}}), {X}={\rm{f}}(S_r, S_{ur})$, where ${U_r}, {U_{ur}}$ are independent exogenous variables, and ${{\rm{f}}_r}, {{\rm{f}}_u}$ are deterministic functions. Given an observation $X = {\rm{f}}({S_r},{S_{ur}})$ where the real foreground exists, we can ask a counterfactual question: what would happen if ${S_{r}}$ remained unchanged and $S_{ur}$ was randomly perturbed? We can answer this question by performing a soft intervention on $S_{ur}$, resulting in an intervention that changes ${\rm{f}}_u$ to $do({S_{ur}}): = {\rm{f}}_u^{*}({S_r},{U_{ur}},{U_A})$, where ${U_A}$ represents a set of random augmentations. Based on the established SCM, fixing the exogenous variables to their factual values, and performing a soft intervention on $S_{ur}$, we obtain the counterfactual observation ${X^{cf}}$ generation process: ${S_r}: = {{\rm{f}}_r}({U_r}),{S^*_{ur}}: = {\rm{f}}_u^{*}({S_r},{U_{ur}},{U_A}),{X^{cf}} = {\rm{f}}({S_r},{S^*_{ur}})$. Therefore, ${X}^{aug}$ can be viewed as a counterfactual version of $X$. For self-supervised learning, the goal is to obtain $Z^{aug}$ from the counterfactual observation ${X}^{aug}$, i.e., ${S_{r}}$ and the perturbed $S_{ur}$. Based on the established SCM and the data generation mechanism, we can draw the following conclusions:

\begin{theorem}\label{theo:fac}
    Given the ${\rm{SCM}}$ shown in Figure \ref{Fig.11subqw.1}, let $f$ denote the feature extractor, let $X^{aug}_1$ and $X^{aug}_2$ denote the corresponding augmented datasets generated by $X$ using different data augmentations. If $f$ can minimize the following objective function:
    \begin{align}
        \mathcal{L}(f,{X^{aug}_1},{X^{aug}_2}) ={\rm D}(f(X_1^{aug}),f(X_2^{aug})) \nonumber \\
        - [H(f({X^{aug}_1})) + H(f({X^{aug}_2}))]
    \end{align}
    where ${\rm D}(\cdot)$ is the distance operator for two discrete distributions and ${H}(\cdot)$ is the information entropy, then $f({X^{aug}_1})$ and $f({X^{aug}_2})$ contain all of task-related information $S_{r}$.
\end{theorem}




Please refer to Appendix for the proof. 
According to Theorem \ref{theo:fac}, only by aligning the feature representations of different augmented samples and maximizing the information entropy of the feature representations during the training process can we ensure that SSL learns all task-relevant factors. As known from Figure \ref{Fig.main1}-\ref{Fig.main3}, the information entropy of the feature representations in the SSL early-layer remains maximized and unchanged from a certain stage of training, while the information entropy of the feature representations in the SSL last-layer is indeed maximized during the early stages of training but then decreases and stabilizes. Thus, we can deduce: SSL indeed extracts all task-related generative factors (as inferred from the maximum information entropy of the early-layer's feature representations), but as the network deepens, more task-irrelevant generative factors are extracted (as inferred from the decrease in information entropy of the last-layer's feature representations in the later stages of training, indicating overfitting). Since the dimensionality of the features is fixed and the maximum capacity for information they can hold is also fixed, when the proportion of task-irrelevant information increases, the quantity of task-relevant information is suppressed, thereby degrading SSL performance. Moreover, based on literature \cite{chan2022redunet}, we can interpret coding rate reduction as measuring the information entropy of a distribution from the perspective of a geometric manifold. As derived from Subsection \ref{umm;qw}, the introduction of UMM is precisely to increase the information quantity of the pre-trained SSL's feature representations again. Therefore, based on Theorem 2, we can infer that UMM is effective and can improve generalization.

\begin{table}
	\centering
	\caption{ The top-1 and top-5 classification accuracies of linear classifier on the ImageNet-100 dataset with ResNet-50 as feature extractor. In the table, the methods with * denote the results with only $\mathcal{L}_{ssl}$ and ${\rm{JS}}(p(Z_{tr - el}^{aug})|p(Z_{tr}^{aug}))$. The methods with ** denote the results with only $\mathcal{L}_{ssl}$ and ${\mathcal{L}}(\Delta R,Z_{tr - el}^{aug},f_{e - l}^*)$.}
	\label{tab:image100} 
 \resizebox{0.9\linewidth}{!}{
	\begin{tabular}{lcc}
		\toprule
	    Method & \textbf{Top-1} & \textbf{Top-5} \\
	    \midrule
	    SimCLR \cite{chen2020simple} & 70.15$\pm$0.16 & 89.75$\pm$0.14  \\
		MoCo \cite{moco} & 72.80$\pm$0.12 & 91.64$\pm$0.11 \\ 
		SimSiam \cite{chen2021exploring} & 73.01$\pm$0.21 & 92.61$\pm$0.27 \\ 
		SwAV \cite{swav} & 75.78$\pm$0.16 & 92.86$\pm$0.15 \\
		LMLC \cite{chen2021large} & 75.89$\pm$0.19 & 92.89$\pm$0.28 \\
	    W-MSE \cite{ermolov2021whitening} & 76.01$\pm$0.27 & 93.12$\pm$0.21 \\
        BYOL \cite{byol} & 75.66$\pm$0.18 & 92.07$\pm$0.22 \\
        Barlow Twins \cite{zbontar2021barlow} & 75.97$\pm$0.23 & 92.91$\pm$0.19 \\
        RELIC v2 \cite{tomasev2022pushing} & 75.88$\pm$0.15 & 93.52$\pm$0.13 \\
        ReSSL \cite{zheng2021ressl} & 75.77$\pm$0.21 & 92.91$\pm$0.27 \\
        CorInfoMax\cite{CorInfoMax}& 75.54$\pm$0.20 & 92.23$\pm$0.25\\
        MEC\cite{MEC}& 75.38$\pm$0.17 & 92.84$\pm$0.20\\
        VICRegL\cite{Vicregl}& 75.96$\pm$0.19 & 92.97$\pm$0.26\\
        \midrule
 SimCLR* &71.21$\pm$0.12 \textcolor{red}{$\uparrow 1.06$}&90.42$\pm$0.06 \textcolor{red}{$\uparrow 0.67$}\\
 BYOL* &75.74$\pm$0.15 \textcolor{red}{$\uparrow 0.08$}&92.58$\pm$0.17 \textcolor{red}{$\uparrow 0.51$}\\
 Barlow Twins* &75.99$\pm$0.08 \textcolor{red}{$\uparrow 0.02$}&93.01$\pm$0.13 \textcolor{red}{$\uparrow 0.10$}\\
        \midrule
 SimCLR** &70.27$\pm$0.22 \textcolor{red}{$\uparrow 0.12$}&89.82$\pm$0.37 \textcolor{red}{$\uparrow 0.07$}\\
 BYOL** &75.77$\pm$0.46 \textcolor{red}{$\uparrow 0.11$}&92.21$\pm$0.17 \textcolor{red}{$\uparrow 0.14$}\\
 Barlow Twins** &76.13$\pm$0.49 \textcolor{red}{$\uparrow 0.16$}&92.96$\pm$0.33 \textcolor{red}{$\uparrow 0.05$}\\
\midrule
 SimSiam + UMM &74.15$\pm$0.14 \textcolor{red}{$\uparrow 1.14$}&93.64$\pm$0.22 \textcolor{red}{$\uparrow 1.03$}\\
 SimCLR + UMM &71.75$\pm$0.17 \textcolor{red}{$\uparrow 1.60$}&91.61$\pm$0.21 \textcolor{red}{$\uparrow 1.86$}\\
 W-MSE + UMM &\bf 77.45$\pm$0.05 \textcolor{red}{$\uparrow 1.44$}&94.28$\pm$0.14 \textcolor{red}{$\uparrow 1.16$}\\
 BYOL + UMM &76.98$\pm$0.12 \textcolor{red}{$\uparrow 1.32$}&\bf 94.99$\pm$0.13 \textcolor{red}{$\uparrow 2.92$}\\
 Barlow Twins + UMM &77.18$\pm$0.15 \textcolor{red}{$\uparrow 1.21$}&94.17$\pm$0.17 \textcolor{red}{$\uparrow 1.26$}\\
 VICRegL + UMM&77.12$\pm$0.14 \textcolor{red}{$\uparrow 1.16$}&94.72$\pm$0.13 \textcolor{red}{$\uparrow 1.75$}\\
 RELIC v2 + UMM&77.44$\pm$0.14 \textcolor{red}{$\uparrow 1.56$}&94.73$\pm$0.41 \textcolor{red}{$\uparrow 1.21$}\\
\bottomrule
	\end{tabular}
 }
\end{table}

\begin{table*}[h]
	\centering
	\caption{The Top-1 and Top-5 classification accuracies of linear classification on the ImageNet dataset with
ResNet-50 as the feature extractor. We record the comparison results from 100, 200, 400, and 1000 epochs.}
	\label{tab:imagenet}
 \resizebox{0.9\linewidth}{!}{
		\begin{tabular}{lcccccc}
		\toprule
		\multirow{2.5}{*}{Method}  &\multicolumn{2}{c}{100 Epochs} & \multicolumn{2}{c}{200 Epochs}&\multicolumn{1}{c} {400 Epochs} &\multicolumn{1}{c} {1000 Epochs}\\
	    \cmidrule(lr){2-3} \cmidrule(lr){4-5}\cmidrule(lr){6-6}\cmidrule(lr){7-7}
	    & \textbf{Top-1} & \textbf{Top-5} & \textbf{Top-1} & \textbf{Top-5} & \textbf{Top-1} & \textbf{Top-1}\\
	    \midrule
	     Supervised&71.9 &-&73.4&-&74.9&76.5\\
	  \midrule
	     MoCo \cite{moco} & 64.5$\pm$0.2 & 86.1$\pm$0.1 & 67.5$\pm$0.2 & 88.4$\pm$0.1 & 69.7$\pm$0.1& 71.1$\pm$0.2\\
	     BYOL \cite{byol} & 68.6$\pm$0.2 & 89.9$\pm$0.1 & 71.1$\pm$0.2& 90.4$\pm$0.2 & 72.8$\pm$0.1 & 74.3$\pm$0.2\\
	     SimCLR \cite{chen2020simple} & 66.5$\pm$0.2& 88.1$\pm$0.2 & 68.6$\pm$0.3&89.7$\pm$0.2 & 69.2$\pm$0.2 & 70.4$\pm$0.3\\
	     SwAV \cite{swav} & 68.8$\pm$0.2 & 88.4$\pm$0.2 & 69.1$\pm$0.2 &89.3$\pm$0.2 & 72.2$\pm$0.3& 75.3$\pm$0.1\\
	     Barlow Twins \cite{zbontar2021barlow} & 67.2$\pm$0.2 & 88.6$\pm$0.1& 69.1$\pm$0.3& 88.4$\pm$0.2& 71.4$\pm$0.1 & 73.2$\pm$0.1 \\
	     SimSiam \cite{chen2021exploring} & 68.1$\pm$0.2 & 87.1$\pm$0.2 & 70.0$\pm$0.1 & 88.7$\pm$0.2 & 70.8$\pm$0.3 & 71.3$\pm$0.2\\
	     RELIC v2 \cite{tomasev2022pushing} & 71.0$\pm$0.2 & 90.3$\pm$0.1 & 72.1$\pm$0.2& 90.6$\pm$0.1& 73.9$\pm$0.2 & 77.1$\pm$0.2\\
	     LMCL \cite{chen2021large} & 66.7$\pm$0.1& 89.8$\pm$0.3 & 70.8$\pm$0.2& 90.0$\pm$0.2 & 72.5$\pm$0.2& 72.9$\pm$0.2\\
	     ReSSL \cite{zheng2021ressl}& 67.4$\pm$0.2 & 90.5$\pm$0.2 & 69.9$\pm$0.2 &91.2$\pm$0.1& 72.4$\pm$0.2 & 72.9$\pm$0.3\\
	     SSL-HSIC \cite{li2021self}&69.3$\pm$0.1  &91.0$\pm$0.1 & 70.6$\pm$0.1 &91.5$\pm$0.1&73.8$\pm$0.2 & 74.8$\pm$0.2\\
	CorInfoMax\cite{CorInfoMax} & 70.1$\pm$0.1 & 91.1$\pm$0.2 & 70.8$\pm$0.1 &91.5$\pm$0.2 &73.2$\pm$0.2 & 74.8$\pm$0.3\\
        MEC\cite{MEC}& 69.9$\pm$0.1 & 90.6$\pm$0.1 & 70.3$\pm$0.2&91.2$\pm$0.3 &72.9$\pm$0.2 & 75.0$\pm$0.2\\
        VICRegL\cite{Vicregl}& 69.9$\pm$0.2 & 91.2$\pm$0.1 & 71.4$\pm$0.2 &91.6$\pm$0.2 &73.2$\pm$0.2 & 75.0$\pm$0.2\\
      \midrule
         SimCLR + UMM &67.6$\pm$0.1 \textcolor{red}{$\uparrow 1.1$}&\bf 91.8$\pm$0.3 \textcolor{red}{$\uparrow 3.7$}&71.4$\pm$0.3 \textcolor{red}{$\uparrow 2.8$}&92.1$\pm$0.1 \textcolor{red}{$\uparrow 2.4$}&72.2$\pm$0.3 \textcolor{red}{$\uparrow 3.0$}&73.1$\pm$0.1 \textcolor{red}{$\uparrow 2.7$}\\
 BYOL + UMM &71.7$\pm$0.3 \textcolor{red}{$\uparrow 3.1$}&91.2$\pm$0.1 \textcolor{red}{$\uparrow 1.3$}&72.3$\pm$0.1 \textcolor{red}{$\uparrow 0.3$}&91.2$\pm$0.2 \textcolor{red}{$\uparrow 0.8$}&73.4$\pm$0.3 \textcolor{red}{$\uparrow 0.6$}&76.5$\pm$0.2 \textcolor{red}{$\uparrow 2.2$}\\
 Barlow Twins + UMM &68.2$\pm$0.4 \textcolor{red}{$\uparrow 1.0$}&89.4$\pm$0.2 \textcolor{red}{$\uparrow 0.8$}&72.8$\pm$0.5 \textcolor{red}{$\uparrow 3.7$}&91.5$\pm$0.2 \textcolor{red}{$\uparrow 3.1$}&73.2$\pm$0.3 \textcolor{red}{$\uparrow 1.8$}&75.8$\pm$0.3 \textcolor{red}{$\uparrow 2.6$}\\
 MoCo + UMM &66.5$\pm$0.2 \textcolor{red}{$\uparrow 2.0$}&87.3$\pm$0.1 \textcolor{red}{$\uparrow 1.2$}&70.5$\pm$0.1 \textcolor{red}{$\uparrow 3.0$}&89.8$\pm$0.2 \textcolor{red}{$\uparrow 1.4$}&72.1$\pm$0.2 \textcolor{red}{$\uparrow 2.4$}&72.3$\pm$0.1 \textcolor{red}{$\uparrow 1.2$}\\
 SimSiam + UMM &69.2$\pm$0.2 \textcolor{red}{$\uparrow 1.1$}&88.9$\pm$0.3 \textcolor{red}{$\uparrow 1.8$}&71.4$\pm$0.1 \textcolor{red}{$\uparrow 1.4$}&90.7$\pm$0.2 \textcolor{red}{$\uparrow 2.0$}&71.6$\pm$0.1 \textcolor{red}{$\uparrow 0.8$}&73.6$\pm$0.4 \textcolor{red}{$\uparrow 2.3$}\\
 VICRegL + UMM &\bf 72.5$\pm$0.3 \textcolor{red}{$\uparrow 2.6$}&91.6$\pm$0.2 \textcolor{red}{$\uparrow 0.4$}&72.7$\pm$0.4 \textcolor{red}{$\uparrow 1.3$}&\bf 92.1$\pm$0.3 \textcolor{red}{$\uparrow 0.5$}&73.6$\pm$0.2 \textcolor{red}{$\uparrow 0.4$}&76.5$\pm$0.1 \textcolor{red}{$\uparrow 1.5$}\\
 RELIC v2 + UMM &72.4$\pm$0.4 \textcolor{red}{$\uparrow 1.4$}&91.6$\pm$0.3 \textcolor{red}{$\uparrow 1.3$}&\bf 73.6$\pm$0.2 \textcolor{red}{$\uparrow 1.5$}&92.0$\pm$0.4 \textcolor{red}{$\uparrow 1.4$}&\bf 75.3$\pm$0.2 \textcolor{red}{$\uparrow 1.4$}&\bf 78.2$\pm$0.4 \textcolor{red}{$\uparrow 1.1$}\\
         
  \bottomrule
	\end{tabular}}
\end{table*}

\begin{table}[h]
	\centering
	\caption{The classification accuracies on the ImageNet dataset under a semi-supervised learning task with the ResNet-50 pre-trained on the Imagenet dataset.}
\label{tab:semi}
\resizebox{\linewidth}{!}{
		\begin{tabular}{lccccc}
		\toprule
		\multirow{2.5}{*}{Method} &\multirow{2.5}{*} {Epochs} &\multicolumn{2}{c}{1\%} & \multicolumn{2}{c}{10\%} \\
	    \cmidrule(lr){3-4} \cmidrule(lr){5-6}
	    & & \textbf{Top-1} & \textbf{Top-5} & \textbf{Top-1} & \textbf{Top-5} \\
	     \midrule
	    
	     MoCo\cite{moco}&200&43.8$\pm$0.2 &72.3$\pm$0.1 &61.9$\pm$0.1 &84.6$\pm$0.2\\
	     BYOL\cite{byol}&200&54.8$\pm$0.2 &78.8$\pm$0.1&68.0$\pm$0.2&88.5$\pm$0.2\\
	   \midrule
	     MoCo + UMM &200 &45.5$\pm$0.2 \textcolor{red}{$\uparrow 1.7$}&72.7$\pm$0.3 \textcolor{red}{$\uparrow 0.4$}&62.9$\pm$0.2 \textcolor{red}{$\uparrow 1.0$}&85.4$\pm$0.1 \textcolor{red}{$\uparrow 0.8$}\\
 BYOL + UMM &200 &\bf 55.7$\pm$0.3 \textcolor{red}{$\uparrow 0.9$}&\bf 80.4$\pm$0.1 \textcolor{red}{$\uparrow 1.6$}&\bf 69.2$\pm$0.3 \textcolor{red}{$\uparrow 1.2$}&\bf 89.7$\pm$0.2 \textcolor{red}{$\uparrow 1.2$}\\
	     
	    \midrule
	     MoCo\cite{moco} &1000 &52.3$\pm$0.1 & 77.9$\pm$0.2 &68.4$\pm$0.1 &88.0$\pm$0.2\\
	     BYOL\cite{byol} & 1000 & 56.3$\pm$0.2 & 79.6$\pm$0.2 & 69.7$\pm$0.2& 89.3$\pm$0.1\\
	     SimCLR\cite{chen2020simple} & 1000 & 48.3$\pm$0.2 & 75.5$\pm$0.1 & 65.6$\pm$0.1 & 87.8$\pm$0.2\\
	     BarlowTwins\cite{zbontar2021barlow} & 1000 & 55.0$\pm$0.1& 79.2$\pm$0.1 & 67.7$\pm$0.2 & 89.3$\pm$0.2\\
	     SimSiam\cite{chen2021exploring} & 1000 & 54.9$\pm$0.2 & 79.5$\pm$0.2
      
      & 68.0$\pm$0.1 &89.0$\pm$0.3 \\
	     RELIC v2 \cite{tomasev2022pushing} &1000 & 55.2$\pm$0.2 & 80.0$\pm$0.1& 68.0$\pm$0.2 & 88.9$\pm$0.2\\
	     LMCL \cite{chen2021large} & 1000 & 54.8$\pm$0.2 & 79.4$\pm$0.2 & 70.3$\pm$0.1  & 89.9$\pm$0.2\\
	     ReSSL \cite{zheng2021ressl} & 1000 & 55.0$\pm$0.1 & 79.6$\pm$0.3 & 69.9$\pm$0.1 & 89.7$\pm$0.1\\
	     SSL-HSIC \cite{li2021self} & 1000 & 55.4$\pm$0.3 & 80.1$\pm$0.2 & 70.4$\pm$0.1 & 90.0$\pm$0.1 \\
      CorInfoMax\cite{CorInfoMax}& 1000 & 55.0$\pm$0.2 & 79.6$\pm$0.3 & 70.3$\pm$0.2 & 89.3$\pm$0.2\\
      MEC\cite{MEC}& 1000 & 54.8$\pm$0.1 & 79.4$\pm$0.2&  70.0$\pm$0.1 & 89.1$\pm$0.1\\
      VICRegL\cite{Vicregl}& 1000 & 54.9$\pm$0.1 & 79.6$\pm$0.2 & 67.2$\pm$0.1  & 89.4$\pm$0.2\\
	   \hline
 	     MoCo + UMM  &1000  &54.4$\pm$0.2 \textcolor{red}{$\uparrow 2.1$}&79.6$\pm$0.1 \textcolor{red}{$\uparrow 1.7$}&\bf 70.6$\pm$0.2 \textcolor{red}{$\uparrow 2.2$}&89.6$\pm$0.3 \textcolor{red}{$\uparrow 1.6$}\\
 BYOL + UMM  &1000 &57.1$\pm$0.1 \textcolor{red}{$\uparrow 0.8$}&80.7$\pm$0.3 \textcolor{red}{$\uparrow 1.1$}&70.4$\pm$0.1 \textcolor{red}{$\uparrow 0.7$}&\bf 90.4$\pm$0.3 \textcolor{red}{$\uparrow 1.1$}\\
 SimCLR + UMM &1000 &52.6$\pm$0.1 \textcolor{red}{$\uparrow 4.3$}&76.5$\pm$0.2 \textcolor{red}{$\uparrow 1.0$}&67.2$\pm$0.3 \textcolor{red}{$\uparrow 1.6$}&88.3$\pm$0.1 \textcolor{red}{$\uparrow 0.5$}\\
 Barlow Twins + UMM &1000 &\bf 57.2$\pm$0.2 \textcolor{red}{$\uparrow 2.2$}&80.6$\pm$0.3 \textcolor{red}{$\uparrow 1.4$}&68.8$\pm$0.2 \textcolor{red}{$\uparrow 1.1$}&90.3$\pm$0.3 \textcolor{red}{$\uparrow 1.0$}\\
 SimSiam + UMM  &1000 &56.3$\pm$0.3 \textcolor{red}{$\uparrow 1.4$}&80.2$\pm$0.2 \textcolor{red}{$\uparrow 0.7$}&69.2$\pm$0.3 \textcolor{red}{$\uparrow 1.2$}&90.1$\pm$0.2 \textcolor{red}{$\uparrow 1.1$}\\
 VICRegL + UMM &1000 &56.8$\pm$0.1 \textcolor{red}{$\uparrow 1.9$}&81.2$\pm$0.1 \textcolor{red}{$\uparrow 1.6$}&69.7$\pm$0.2 \textcolor{red}{$\uparrow 2.5$}&90.3$\pm$0.2 \textcolor{red}{$\uparrow 0.9$}\\
 RELIC v2 + UMM &1000 &56.6$\pm$0.5 \textcolor{red}{$\uparrow 1.4$}&\bf 81.5$\pm$0.3 \textcolor{red}{$\uparrow 1.5$}&68.9$\pm$0.4 \textcolor{red}{$\uparrow 0.9$}&\bf 90.4$\pm$0.1 \textcolor{red}{$\uparrow 1.5$}\\
		\bottomrule
	\end{tabular}}
\end{table}

\begin{table*}
	\centering
	\caption{The results of transfer learning on object detection and instance segmentation with C4-backbone as the feature extractor.}
\label{tab:pascol}
		\resizebox{\linewidth}{!}{
		
		\begin{tabular}{lcccccccccccc}
		\toprule
		\multirow{2.5}{*}{Method} &\multicolumn{3}{c}{VOC 07 detection} & \multicolumn{3}{c}{VOC 07+12 detection} &\multicolumn{3}{c}{COCO detection}&\multicolumn{3}{c}{COCO instance segmentation}\\
	    \cmidrule(lr){2-4} \cmidrule(lr){5-7} \cmidrule(lr){8-10} \cmidrule(lr){11-13} 
	    & \(\mathbf{AP_{50}}\)& \(\mathbf{AP}\) & \(\mathbf{AP_{75}}\)& \(\mathbf{AP_{50}}\)& \(\mathbf{AP}\) & \(\mathbf{AP_{75}}\)& \(\mathbf{AP_{50}}\)& \(\mathbf{AP}\) & \(\mathbf{AP_{75}}\)& \(\mathbf{AP^{mask}_{50}}\)& \(\mathbf{AP^{mask}}\) & \(\mathbf{AP^{mask}_{75}}\)\\
	       \midrule
	     Supervised & 74.4 & 42.4 & 42.7 & 81.3 & 53.5 & 58.8 & 58.2 & 38.2 & 41.2 & 54.7 & 33.3 & 35.2\\
	   \midrule
	     SimCLR \cite{chen2020simple} & 75.9 & 46.8 & 50.1 & 81.8 & 55.5 & 61.4 & 57.7 & 37.9 & 40.9 & 54.6 & 33.3 & 35.3\\
	     MoCo \cite{moco} & 77.1 & 46.8 & 52.5 & 82.5 & 57.4 & 64.0 & 58.9 & 39.3 & 42.5 & 55.8 & 34.4 & 36.5\\
	     BYOL\cite{byol} & 77.1 & 47.0 & 49.9 & 81.4 & 55.3 & 61.1 & 57.8 & 37.9 & 40.9 & 54.3 & 33.2 & 35.0\\
	     SwAV \cite{swav} & 75.5 & 46.5 & 49.6 & 82.6 & 56.1 & 62.7 & 58.6 & 38.4 & 41.3 & 55.2 & 33.8 & 35.9\\
	     Barlow Twins \cite{zbontar2021barlow} & 75.7 & 47.2 & 50.3 & 82.6 & 56.8 & 63.4 & 59.0 & 39.2 & 42.5 & 56.0 & 34.3 & 36.5\\
	     SimSiam \cite{chen2021exploring} & 77.3 & 48.5 & 52.5 & 82.4 & 57.0 & 63.7 & 59.3 & 39.2 & 42.1 & 56.0 & 34.4 & 36.7\\
	     MEC \cite{MEC} & 77.4 & 48.3 & 52.3 & 82.8 & 57.5 & 64.5 & 59.8 & 39.8 & 43.2 & 56.3 & 34.7 & 36.8\\
	     RELIC v2 \cite{tomasev2022pushing} & 76.9 & 48.0 & 52.0 & 82.1 & 57.3 & 63.9 & 58.4 & 39.3 & 42.3 & 56.0 & 34.6 & 36.3\\
	 CorInfoMax\cite{CorInfoMax}& 76.8 & 47.6 & 52.2 & 82.4 & 57.0 & 63.4 & 58.8 & 39.6 & 42.5 & 56.2 & 34.8 & 36.5\\
      VICRegL\cite{Vicregl}& 75.9 & 47.4 & 52.3 & 82.6 & 56.4 & 62.9 & 59.2 & 39.8 & 42.1 & 56.5 & 35.1 & 36.8\\    
     \midrule
 SimCLR + UMM &76.7 \textcolor{red}{$\uparrow 0.8$}&48.1 \textcolor{red}{$\uparrow 1.3$}&50.8 \textcolor{red}{$\uparrow 0.7$}&82.7 \textcolor{red}{$\uparrow 0.9$}&56.9 \textcolor{red}{$\uparrow 1.4$}&62.7 \textcolor{red}{$\uparrow 1.3$}&58.9 \textcolor{red}{$\uparrow 1.2$}&39.1 \textcolor{red}{$\uparrow 1.2$}&42.1 \textcolor{red}{$\uparrow 1.2$}&55.7 \textcolor{red}{$\uparrow 1.1$}&34.2 \textcolor{red}{$\uparrow 0.9$}&36.1 \textcolor{red}{$\uparrow 0.8$}\\
 MoCo + UMM &78.2 \textcolor{red}{$\uparrow 1.1$}&47.8 \textcolor{red}{$\uparrow 1.0$}&53.5 \textcolor{red}{$\uparrow 1.0$}&83.6 \textcolor{red}{$\uparrow 1.1$}&58.2 \textcolor{red}{$\uparrow 0.8$}&65.2 \textcolor{red}{$\uparrow 1.2$}&60.2 \textcolor{red}{$\uparrow 1.3$}&40.7 \textcolor{red}{$\uparrow 1.4$}&43.6 \textcolor{red}{$\uparrow 1.1$}&56.9 \textcolor{red}{$\uparrow 1.1$}&35.7 \textcolor{red}{$\uparrow 1.3$}&37.9 \textcolor{red}{$\uparrow 1.4$}\\
 BYOL + UMM &78.3 \textcolor{red}{$\uparrow 1.2$}&48.1 \textcolor{red}{$\uparrow 1.1$}&50.4 \textcolor{red}{$\uparrow 0.5$}&82.6 \textcolor{red}{$\uparrow 1.2$}&56.8 \textcolor{red}{$\uparrow 1.5$}&62.8 \textcolor{red}{$\uparrow 1.7$}&59.6 \textcolor{red}{$\uparrow 1.8$}&39.7 \textcolor{red}{$\uparrow 1.8$}&43.2 \textcolor{red}{$\uparrow 2.3$}&56.7 \textcolor{red}{$\uparrow 2.4$}&35.9 \textcolor{red}{$\uparrow 2.7$}&36.5 \textcolor{red}{$\uparrow 1.5$}\\
 SwAV + UMM &76.4 \textcolor{red}{$\uparrow 0.9$}&47.3 \textcolor{red}{$\uparrow 0.8$}&50.2 \textcolor{red}{$\uparrow 0.6$}&83.2 \textcolor{red}{$\uparrow 0.6$}&57.6 \textcolor{red}{$\uparrow 1.5$}&63.3 \textcolor{red}{$\uparrow 0.6$}&59.7 \textcolor{red}{$\uparrow 1.1$}&39.4 \textcolor{red}{$\uparrow 1.0$}&42.6 \textcolor{red}{$\uparrow 1.3$}&56.6 \textcolor{red}{$\uparrow 1.4$}&35.6 \textcolor{red}{$\uparrow 1.8$}&37.1 \textcolor{red}{$\uparrow 1.2$}\\
 SimSiam + UMM &\bf 78.6 \textcolor{red}{$\uparrow 1.3$}&\bf 49.7 \textcolor{red}{$\uparrow 1.2$}&53.9 \textcolor{red}{$\uparrow 1.4$}&83.5 \textcolor{red}{$\uparrow 1.1$}&\bf 58.7 \textcolor{red}{$\uparrow 1.7$}&\bf 65.5 \textcolor{red}{$\uparrow 1.8$}&60.4 \textcolor{red}{$\uparrow 1.1$}&40.4 \textcolor{red}{$\uparrow 1.2$}&43.3 \textcolor{red}{$\uparrow 1.2$}&57.1 \textcolor{red}{$\uparrow 1.1$}&35.2 \textcolor{red}{$\uparrow 0.8$}&37.3 \textcolor{red}{$\uparrow 0.6$}\\
 MEC + UMM   &78.3 \textcolor{red}{$\uparrow 0.9$}&49.4 \textcolor{red}{$\uparrow 1.1$}&\bf 54.2 \textcolor{red}{$\uparrow 1.9$}&83.7 \textcolor{red}{$\uparrow 0.9$}&58.3 \textcolor{red}{$\uparrow 0.8$}&65.4 \textcolor{red}{$\uparrow 0.9$}&\bf 60.7 \textcolor{red}{$\uparrow 0.9$}&\bf 40.8 \textcolor{red}{$\uparrow 1.0$}&\bf 44.1 \textcolor{red}{$\uparrow 0.9$}&57.2 \textcolor{red}{$\uparrow 0.9$}&35.8 \textcolor{red}{$\uparrow 1.1$}&37.8 \textcolor{red}{$\uparrow 1.0$}\\
 VICRegL + UMM&77.3 \textcolor{red}{$\uparrow 1.4$}&48.6 \textcolor{red}{$\uparrow 1.2$}&53.2 \textcolor{red}{$\uparrow 0.9$}&\bf 83.9 \textcolor{red}{$\uparrow 1.3$}&57.4 \textcolor{red}{$\uparrow 1.0$}&64.7 \textcolor{red}{$\uparrow 1.8$}&59.9 \textcolor{red}{$\uparrow 0.7$}&40.4 \textcolor{red}{$\uparrow 0.6$}&43.3 \textcolor{red}{$\uparrow 1.2$}&\bf 57.7 \textcolor{red}{$\uparrow 1.2$}&35.8 \textcolor{red}{$\uparrow 0.7$}&\bf 38.5 \textcolor{red}{$\uparrow 1.7$}\\    
 RELIC v2 + UMM &78.2 \textcolor{red}{$\uparrow 1.3$}&49.5 \textcolor{red}{$\uparrow 1.5$}&53.2 \textcolor{red}{$\uparrow 1.2$}&83.2 \textcolor{red}{$\uparrow 1.1$}&58.4 \textcolor{red}{$\uparrow 1.1$}&\bf 65.5 \textcolor{red}{$\uparrow 1.6$}&60.2 \textcolor{red}{$\uparrow 1.8$}&\bf 40.8 \textcolor{red}{$\uparrow 1.5$}&43.5 \textcolor{red}{$\uparrow 1.2$}&57.2 \textcolor{red}{$\uparrow 1.2$}&\bf 36.4 \textcolor{red}{$\uparrow 1.8$}&37.4 \textcolor{red}{$\uparrow 1.1$}\\
		\bottomrule
	\end{tabular}
	}
\end{table*}

\section{Experiments} \label{sec:exp}
To evaluate the performance of the proposed UMM, we conduct experiments on multiple downstream tasks. With the pre-trained model from SSL methods, we further train the SSL+UMM according to Algorithm \ref{alg1} and Equation \ref{umm:eq} for 200 epochs on the dataset used for pre-training the corresponding SSL methods. The bi-level optimization process is implemented with the approximate implicit differentiation method [51]. All results reported in this paper are the averages of the results of five independent experiments.

\subsection{Benchmark Datasets}
For unsupervised learning task, we evaluate UMM on six image datasets, including CIFAR-10 dataset \cite{cifar10}, CIFAR-100 dataset \cite{cifar10}, STL-10 dataset \cite{stl10}, Tiny ImageNet dataset \cite{leTinyImagenetVisual2015}, ImageNet-100 dataset \cite{tianContrastiveMultiviewCoding2020}, and ImageNet dataset \cite{krizhevsky2012imagenet}. For semi-supervised learning task, we evaluate UMM on ImageNet dataset. For transfer learning task, we validate UMM on the instance segmentation and object detection dataset, e.g., the PASCAL VOC dataset\cite{everingham2010pascal} and the COCO \cite{lin2014microsoft} dataset. (1) CIFAR-10 dataset is a widely used image classification benchmark dataset with 10 classes, and each class has 5000 images with a resolution of 32$\times$32. (2) CIFAR-100 dataset is a small-scale widely used image classification benchmark dataset with 100 classes, and each of the classes has 5000 images with a resolution of 32$\times$32. (3) STL-10 dataset contains 10 classes, with 500 training images and 800 test images per class, all at a high resolution of 96x96 pixels. Additionally, the dataset includes a vast set of 100,000 unlabeled images for unsupervised learning. (4) Tiny ImageNet dataset is an image classification dataset released by Stanford University in 2016 and is also a subset of ImageNet. It includes 200 classes. Each class has 500 training pictures, 50 verification pictures, and 50 test pictures. (5) ImageNet-100 dataset is a subset of the ImageNet dataset. It has 100 classes, and each of them has 1000 pictures. (6) ImageNet dataset is organized according to the WordNet hierarchy. It is a well-known large-scale dataset that contains 1.3M training images and 50K test images with over 1000 classes. (7) PASCAL VOC dataset is a well-known dataset including object classification, detection, and segmentation that consists of 20 classes with a total of 11,530 images. The PASCAL VOC dataset consists of two parts: VOC 07 and VOC 12. (8) COCO dataset is mainly used for object detection and segmentation tasks. It includes 91 classes, 328,000 samples, and 2,500,000 labels.

\subsection{Unsupervised Learning}
\textbf{Experimental setup.} Experiments follow the most common evaluation protocol for self-supervised learning, freezing the feature extractor network and training a supervised linear classifier on top of it. We utilize the Stochastic Gradient Descent (SGD) optimizer with a momentum of 0.9 to optimize our objective functions. The linear classifier is trained for 500 epochs with a mini-batch of 256 with an initial learning rate of $10^{-2}$, and decays to $5 \times 10^{-6}$ until the training is completed. The feature extractor for small-scale datasets, e.g., CIFAR-10, CIFAR-100, STL-10, and Tiny ImageNet is ResNet-18. For medium-scale datasets, e.g., ImageNet-100, and large-scale datasets, e.g., ImageNet, the feature extractor is ResNet-50.

\noindent\textbf{Results.} Table \ref{tab:exp1_1} shows the results of the linear classifier and 5-nn classifier on four small-scale datasets. We can observe that the experimental results of the proposed UMM outperform the state-of-the-art methods on all four datasets by a significant margin. Specifically, for the CIFAR-10 dataset, the W-MSE + UMM achieves the best results in both linear evaluation and 5-nn classification with an average linear evaluation accuracy of 93.06\% and an average 5-nn classification accuracy of 91.04\%. For the CIFAR-100 dataset, the VICRegL + UMM achieves the best results with an average linear evaluation accuracy of 69.21\% and an average 5-nn classification accuracy of 58.42\%. For the STL-10 dataset, the BYOL + UMM gets the best linear evaluation results of average accuracy 93.99\%, while the VICRegL + UMM gets the best results for 5-nn classification results with an average accuracy 91.17\%. Lastly, for the Tiny ImageNet dataset, BYOL + UMM gets the best results with an average linear evaluation accuracy of 52.54\%, and an average 5-nn classification of 37.61\%.

Table \ref{tab:image100} presents the top-1 and top-5 classification accuracy of linear classifier on ImageNet-100. Notebly, UMM demonstrates its ability to enhence the performance of contrastive learning methods. Specifically, the W-MSE + UMM achieves the highest average top-1 accuracy of 77.45\%, which represents a 1.79\% improvement over the accuracy of W-MSE alone. Additionally, BYOL + UMM achieves the best average top-5 accuracy of 94.99\%, which is 2.92\% better than BYOL's top-5 accuracy. 

Table \ref{tab:imagenet} shows the top-1 and top-5 classification accuracy on the ImageNet dataset. Remarkably, the average classification accuracy of the proposed UMM consistently surpasses that of other state-of-the-art methods in 100, 200, and 400 epochs. For 1000 epochs, although the result of the proposed UMM may not match that of RELIC v2, the results of BYOL + UMM and VICRegL + UMM outperform other state-of-the-art methods with average accuracy of 76.5\%. 

These results allow us to safely conclude that the proposed UMM can improve the performance of SSL methods on unsupervised classification tasks.

\subsection{Semi-Supervised Learning}

\textbf{Experimental setup}. Experiments follow the most common evaluation protocol for semi-supervised learning \cite{zbontar2021barlow}. We create two subsets by sampling 1\% and 10\% of the training dataset while ensuring class balance. We fine-tune the models on these two subsets for 50 epochs. The learning rate for classifier is 1.0 for the 1\% subset and 0.1 for the 10\% subset. The learning rate for the backbone network is set as 0.0001 for the 1\% subset and 0.01 for the 10\% subset.

\noindent\textbf{Results}. In Table \ref{tab:semi}, the classification results obtained on the ImageNet are compared with existing methods using two pre-trained models with 200 and 1000 epochs. From the results, we can observe that the classification accuracy of the proposed UMM outperforms other state-of-the-art methods. Specifically, when only 1\% of labels available in 1000 epochs, the average top-1 accuracy of Barlow Twins + UMM surpasses that of Barlow Twins by 2.2\%, and the average top-5 accuracy of VICRegL + UMM exceeds that of VICRegL by 1.6\%. When only 10\% of labels available in 1000 epochs, the average top-1 accuracy of MoCo + UMM outperforms MoCo by 2.2\%, the average top-5 accuracy of BYOL + UMM outperforms BYOL by 1.1\%. These results demonstrate the effectiveness of the proposed method.

\begin{table}[t]
    \centering
   \caption{Influence of $\alpha$}
    \resizebox{\linewidth}{!}{
    \begin{tabular}{lccccccc}
    		\toprule
    		  Method & $10^{-3}$ & $10^{-2}$ & $10^{-1}$ & $1$ & $10$ & $10^{2}$ & $10^{3}$\\
    		\midrule
    		BYOL + UMM & 91.98 & 92.59 & 92.77 & 92.99 & 92.56 & 91.92 & 91.02\\
    		\bottomrule
    	\end{tabular}}
    \label{tab:relation21}
\end{table}

\begin{table}[t]
    \centering    
    \caption{Influence of $\beta$.}
    \resizebox{\linewidth}{!}{
    \begin{tabular}{lccccccc}
    		\toprule
    		  Method & $10^{-3}$ & $10^{-2}$ & $10^{-1}$ & $1$ & $10$ & $10^{2}$ & $10^{3}$\\
    		\midrule
    		BYOL + UMM & 91.87 & 92.64 & 92.99 & 92.71 & 92.33 & 91.84 & 90.97\\
    		\bottomrule
    	\end{tabular}}
\label{tab:relation22}
\end{table}

\subsection{Transfer Learning}
\textbf{Experimental Setup}. Experiments follow the common setting for transfer learning used by existing methods\cite{chen2020simple,zbontar2021barlow,grill2020bootstrap}. We evaluate UMM on object detection and instance segmentation tasks on Pascal VOC and COCO datasets. For the VOC detection task, we use Faster R-CNN\cite{ren2015faster} with a C4-backbone \cite{wu2019detectron2}, while for the COCO detection and instance segmentation tasks, we use Mask R-CNN\cite{he2017mask} with a 1$\times$ schedule and the same backbone as Faster R-CNN. To fine-tune our models, we search for an optimal learning rate on the target datasets and keep all other parameters the same as those in the Detectron2 library \cite{wu2019detectron2}. We train our Faster R-CNN model on the VOC 07+12 trainval set, which contains 16K images. During training, we reduce the initial learning rate by a factor of 10 after 18K and 22K iterations. Additionally, we train our model using only the VOC 07 trainval set, which contains 5K images, with smaller iterations that match the dataset size. For the Mask R-CNN model, we train it on the COCO 2017 train split and report our results on the val split. We fine-tune this model end-to-end on the target dataset using the same learning rate search strategy as with Faster R-CNN.

\noindent\textbf{Results}. We report the results of our proposed method compared with baseline methods in Table \ref{tab:pascol}. These results highlight the performance improvements achieved by the proposed UMM in both detection and instance segmentation tasks. Specifically, For the VOC 07 detection task, SimSiam + UMM and MEC + UMM demonstrate the best performance. In the VOC 07+12 detection task, VICRegL + UMM, MEC + UMM, and SimSiam + UMM outperform other methods. In the COCO detection task, MEC + UMM, MoCo + UMM achieves the best results. Finally, in the COCO instance segmentation task, VICRegL + UMM and BYOL + UMM get the best results. 
These results demonstrate that the proposed UMM can effectively improves the performance in the transfer learning task.

\begin{figure*}[tb]
\vspace{-0.1cm}
\centering 
\subfigure[CIFAR-10]{
\label{Fig.1subqw.1}
\includegraphics[width=0.315\textwidth]{ 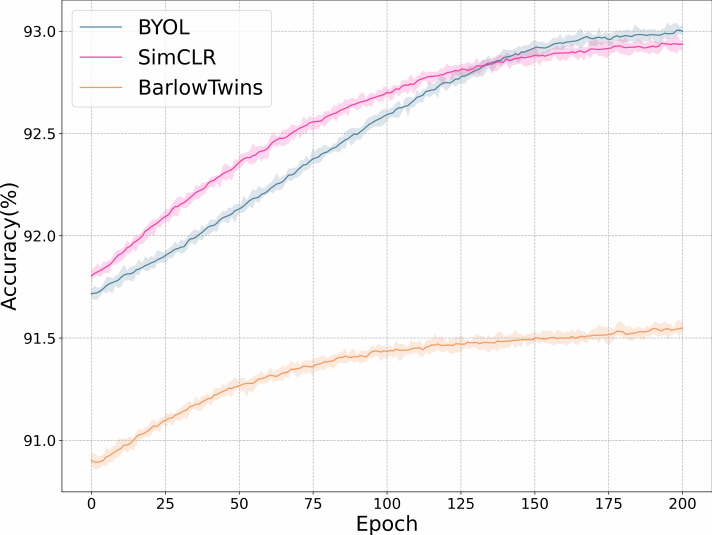}}
\subfigure[CIFAR-100]{
\label{Fig.1subqw.2}
\includegraphics[width=0.315\textwidth]{ 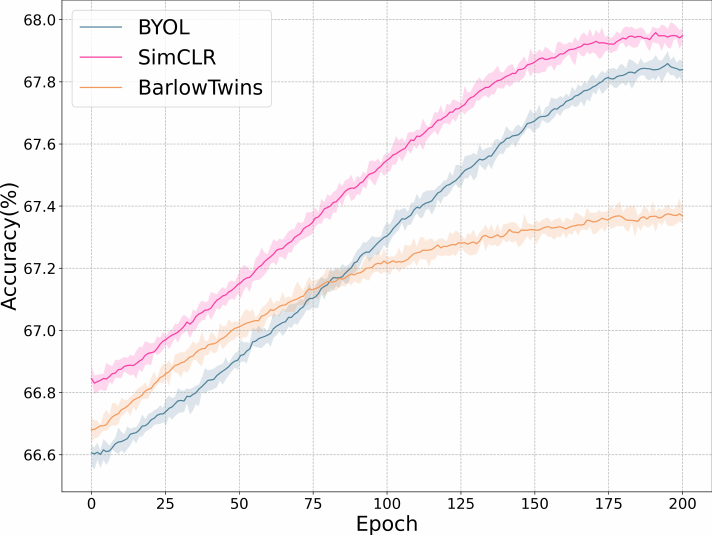}}
\subfigure[Tiny-ImageNet]{
\label{Fig.1subqw.3}
\includegraphics[width=0.315\textwidth]{ 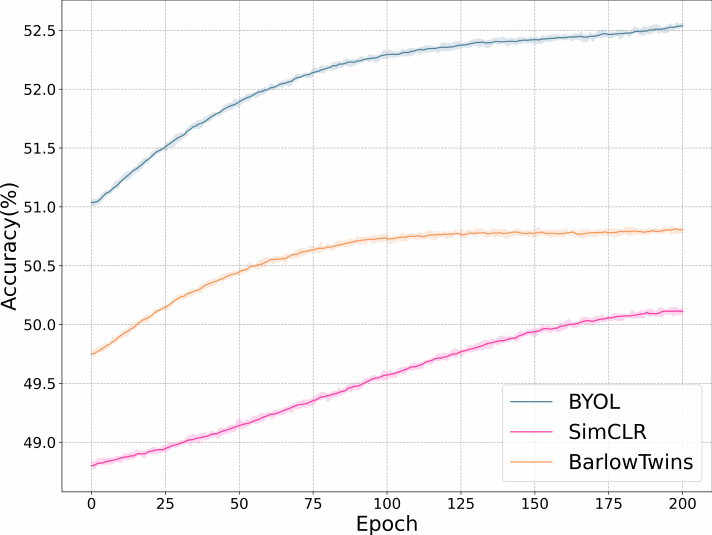}}
\subfigure[CIFAR-10]{
\label{Fig.1subqw.4}
\includegraphics[width=0.315\textwidth]{ 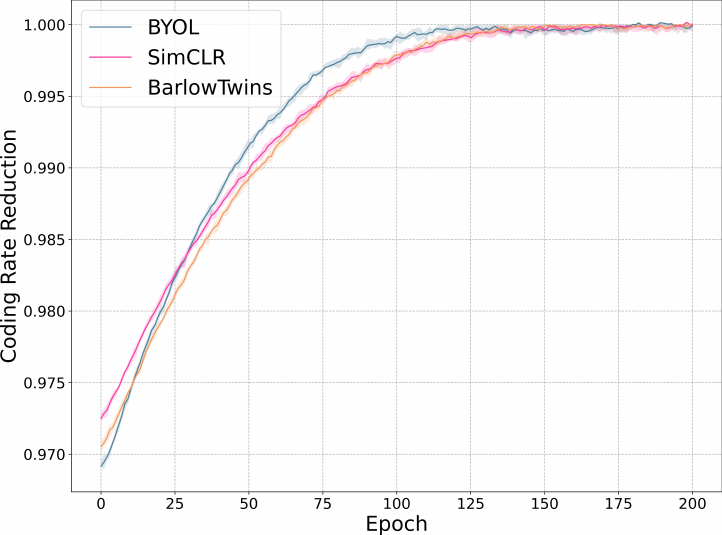}}
\subfigure[CIFAR-100]{
\label{Fig.1subqw.5}
\includegraphics[width=0.315\textwidth]{ 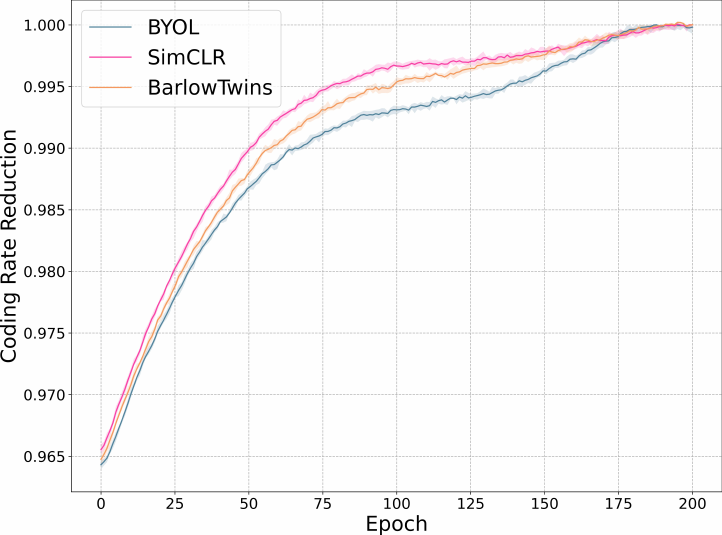}}
\subfigure[Tiny-ImageNet]{
\label{Fig.1subqw.6}
\includegraphics[width=0.315\textwidth]{ 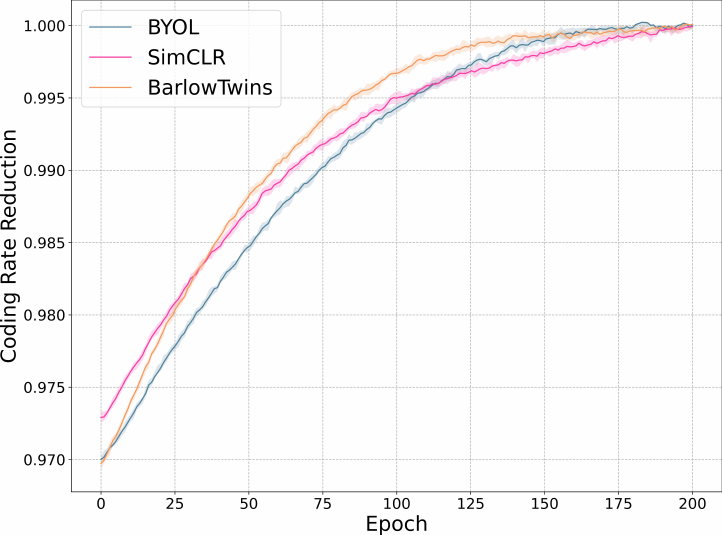}}
\caption{(a)-(c) represent the curves of test accuracy versus training epoch. (d)-(f) denote the curves of coding rate reduction versus training epoch. All results are based on the last layer output.}
\label{Fig.main4}
\end{figure*}

\begin{table}[htb]
	\begin{center}
		\caption{OOD-based test accuracy on CIFAR-10-LT,  CIFAR-100-LT, and ImageNet-100-LT datasets.}
		\vspace{-0.3cm}
		\label{tab:ood}
		\setlength\tabcolsep{3.9pt}
  \resizebox{\linewidth}{!}{
		\begin{tabular}{c|c|c|c}
			\toprule
			{Method}  & {CIFAR-10-LT}  & {CIFAR-100-LT} & {ImageNet-100-LT}\\ 
            \midrule
			SimCLR\cite{chen2020simple} & 75.34$\pm$0.31 & 47.65$\pm$0.27 & 67.08$\pm$0.26 \\ 
			SimCLR + UMM & \bf{84.32$\pm$0.18} & \bf{58.27$\pm$0.25} & \bf{73.65$\pm$0.14} \\ 
			\midrule
			SDCLR\cite{jiang2021self} & 81.74$\pm$0.14 & 55.48$\pm$0.16 & 67.54$\pm$0.34\\ 
			SDCLR + UMM & \bf{86.16$\pm$0.17} & \bf{62.81$\pm$0.28} & \bf 72.13$\pm$0.16 \\ 
			\midrule
			BCL-I\cite{zhou2022contrastive} & 81.99$\pm$0.14 & 55.70$\pm$0.23 & 67.72$\pm$0.36 \\ 
			BCL-I + UMM & \bf{85.62$\pm$0.24} & \bf{58.74$\pm$0.17} & \bf 72.67$\pm$0.14 \\
			\bottomrule
		\end{tabular}}
	\end{center}
\end{table}

\subsection{Out-of-Distribution}
\textbf{Experimental Setup}. Experiments follow the common setting for out-of-distribution learning used by existing methods \cite{bai2023effectiveness}. Three popular datasets are selected, e.g., CIFAR-10-LT \cite{cui2019class}, CIFAR-100-LT \cite{cui2019class}, and ImageNet-100-LT 
 \cite{jiang2021self}. These datasets are long-tail
subsets sampled from the original CIFAR-10, CIFAR-100, and ImageNet-100. For the CIFAR-10-LT and CIFAR-100-LT datasets, 300K Random Images \cite{hendrycks2018deep} are selected as the OOD dataset. For the ImageNet-100-LT dataset, the ImageNetR \cite{hendrycks2021many} is selected as the OOD dataset. We first train SSL methods on CIFAR-10-LT, CIFAR-100-LT, and  ImageNet-100-LT datasets, and then conduct testing on their corresponding OOD datasets. We report
performance under the widely used linear-probing evaluation protocols.

\noindent\textbf{Results.} The Table \ref{tab:ood} presents the accuracy evaluation using linear probing. In the case of the CIFAR-10-LT dataset, our proposed UMM method enhances SimCLR performance by 8.98\%, SDCLR\cite{jiang2021self} by 4.42\% and BCL-I\cite{zhou2022contrastive} by 3.63\%. On the CIFAR-100-LT dataset, UMM boosts SimCLR by 10.62\%, SDCLR by 7.33\% and BCL-I by 3.04\%. For the ImageNet-100-LT dataset, UMM elevates SimCLR by 6.57\%, SDCLR by 4.59\% and BCL-I by 4.95\%. These results demonstrate the effectiveness of our proposed UMM on the out-of-distribution datasets.

\subsection{Ablation Study}

Firstly, we carry out experiments to evaluate the impact of ${\mathcal{L}}(\Delta R,Z_{tr - el}^{aug},f_{e - l}^*)$ and ${\rm{JS}}(p(Z_{tr - el}^{aug})|p(Z_{tr}^{aug}))$. We set $\alpha = 0$ in Equation \ref{umm:eq} to eliminate the influence of ${\mathcal{L}}(\Delta R,Z_{tr - el}^{aug},f_{e - l}^*)$ and denote the resulting method as *. We set $\beta = 0$ in Equation \ref{umm:eq} to eliminate the influence of ${\rm{JS}}(p(Z_{tr - el}^{aug})|p(Z_{tr}^{aug}))$. The results are shown in Table \ref{tab:exp1_1} and Table \ref{tab:image100}. We observe that the performances of * and **surpass the compared methods but are lower than UMM. Therefore, comparing * with the baseline methods, we can obtain that ${\rm{JS}}(p(Z_{tr - el}^{aug})|p(Z_{tr}^{aug}))$ is effective. Comparing ** with the baseline methods, we can obtain that ${\mathcal{L}}(\Delta R,Z_{tr - el}^{aug},f_{e - l}^*)$ is effective. Comparing * and ** with UMM, we can conclude that only if ${\mathcal{L}}(\Delta R,Z_{tr - el}^{aug},f_{e - l}^*)$ and ${\rm{JS}}(p(Z_{tr - el}^{aug})|p(Z_{tr}^{aug}))$ learn together can the SSL method be freed from overfitting. Further experiments about hyperparameters $\alpha$ and $\beta$ in Equation \ref{umm:eq} are carried out based on CIFAR-10 dataset and BYOL. The results are shown in Table \ref{tab:relation21} and Table \ref{tab:relation22}. We can see that a properly weighted can promote the improvement of UMM. Also, we show in Figure \ref{Fig.main4} the curves of test accuracy (\ref{Fig.1subqw.1}-\ref{Fig.1subqw.3}) and the $\Delta R(Z,\Pi ,\varepsilon )$ (\ref{Fig.1subqw.4}-\ref{Fig.1subqw.6}) versus training epoch for UMM on the three datasets. All results are based on the last layer output. From Figure \ref{Fig.1subqw.1}-\ref{Fig.1subqw.6}, we observe that the test accuracy and $\Delta R(Z,\Pi ,\varepsilon )$ all first increase and then stabilize. At the same time, their trends of change are nearly identical, and they exhibit a strong correlation. These results show the the proposed UMM indeed make the SSL methods regain generalization. 

Secondly, by examining Figure \ref{Fig.main1} and Figure \ref{Fig.main3}, it becomes apparent that the peaks of linear evaluation accuracy and coding rate reduction occur at similar epochs. In this subsection, we provide an analysis to explain why the coding rate reduction is an efficient choice as an indicator of stop training. In SSL, evaluating the performance of a pre-trained model on a test set commonly involves two methods: linear evaluation and the K-Nearest Neighbor (K-NN) algorithm. Linear evaluation entails training a linear classifier using all available training data, while the K-NN algorithm involves comparing the similarity between each test sample and all training samples. In contrast, calculating the coding rate reduction does not rely on training data. To demonstrate the efficiency of these methods, we compare their execution times in Table \ref{timescac}. The results in Table \ref{timescac} are based on the CIFAR-10 dataset, with linear evaluation using 500 epochs and K-NN using 5 nearest neighbors. This experiment is conducted on a single NVIDIA Tesla V100 GPU. As we can see, the computation complexity of the coding rate reduction is the lowest. Therefore, we can conclude that the coding rate reduction not only has a significant advantage in computational complexity, but at the same time does not require sample labeling, which can make it easier to meet the needs of real-world scenarios.

\begin{table}[t]
\caption{The comparison of the execution time for different evaluation methods.}
    \centering
    \begin{tabular}{c|c}
    \toprule
       Method  & Time (s) \\
    \midrule
       Linear Evaluation  & 92.76 \\
       KNN & 13.63 \\
        CRR   &  4.03 \\
    \bottomrule
    \end{tabular}
    
    \label{timescac}
\end{table}

Thirdly, to further explore the correlation between the test accuracy and the value of the coding rate reduction in Section \ref{sec:me}, we compare the trend of $\Delta R(Z,\Pi ,\varepsilon )$ and the test accuracy of the pre-training process of SSL methods in five datasets in the same figure, which is shown in Figure \ref{Fig.compare}. 
From Figure \ref{Fig.compare}, it is clear that the trend of test accuracy is nearly consistent with that of $\Delta R(Z,\Pi,\varepsilon )$. We also provide a quantified result of the trends between $\Delta R(Z,\Pi ,\varepsilon )$ and the test accuracy using Pearson Correlation Coefficient in Appendix \ref{app:corr}. Therefore, we can conclude that using the coding rate reduction as a means to monitor the occurrence of overfitting in SSL is effective and reasonable.

\begin{figure*}[htb]
    \centering
    \subfigure[CIFAR-10]{
    \includegraphics[width=0.3\textwidth]{ 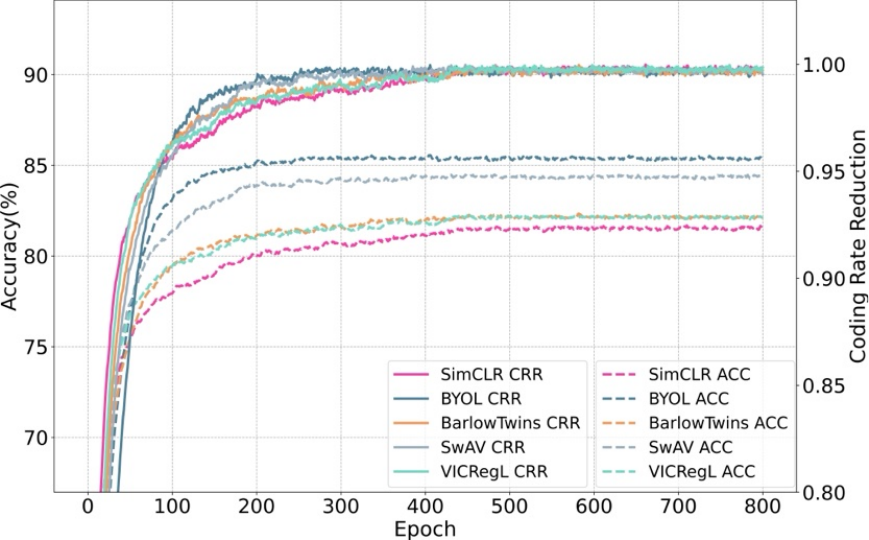}
    }
    \subfigure[CIFAR-100]{
    \includegraphics[width=0.3\textwidth]{ 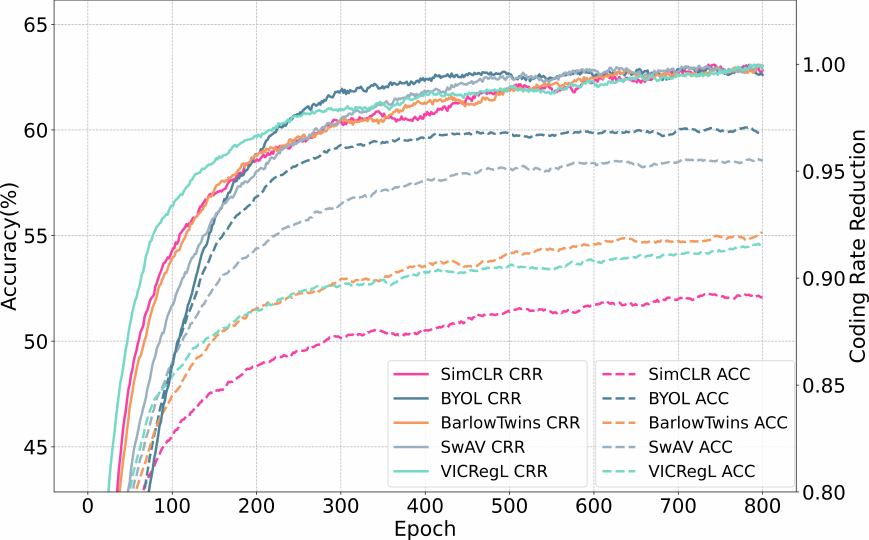}
    }
    \subfigure[Tiny-ImageNet]{
    \includegraphics[width=0.3\textwidth]{ 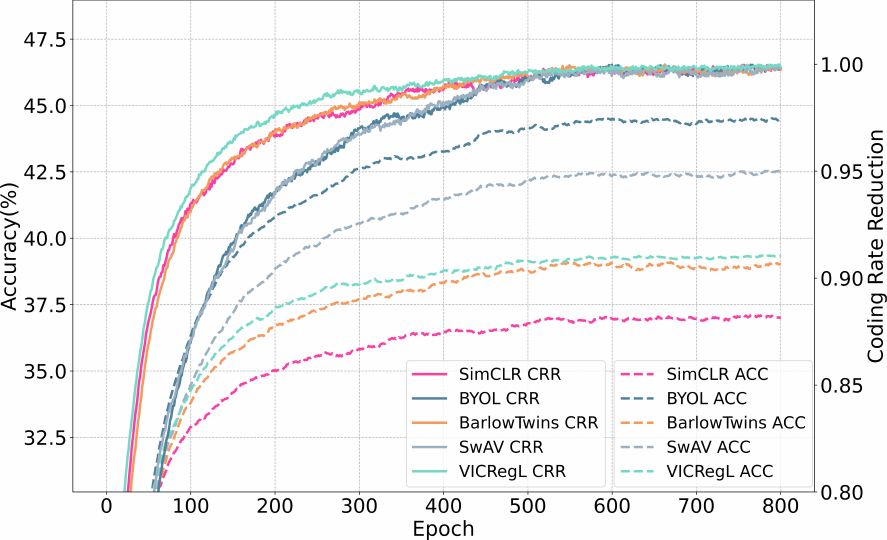}
    }
    \subfigure[ImageNet-100]{
    \includegraphics[width=0.3\textwidth]{ 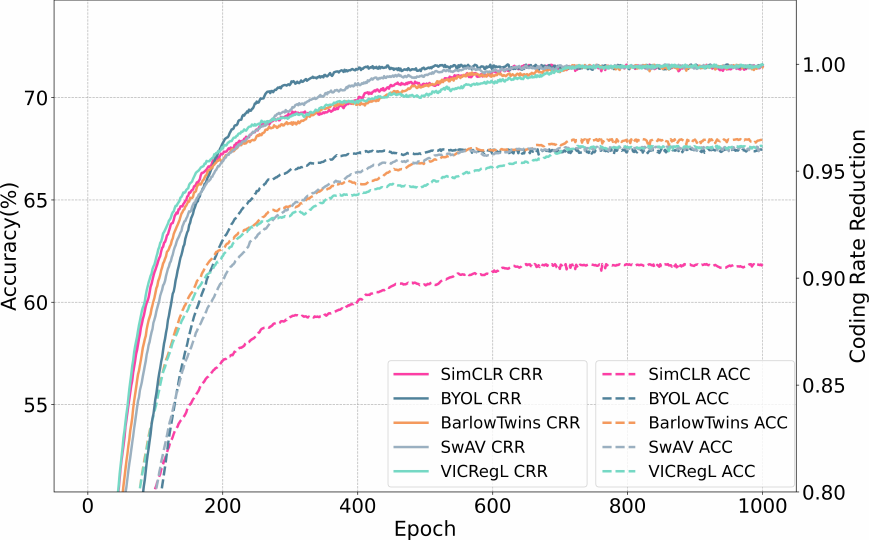}
    }
    \subfigure[ImageNet]{
    \includegraphics[width=0.3\textwidth]{ 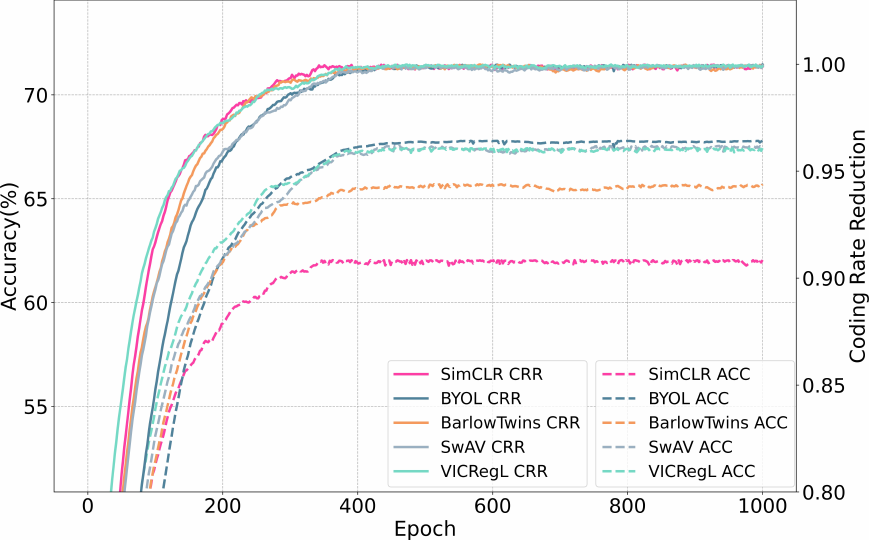}
    }
    \\
    \subfigure[CIFAR-10]{
    \includegraphics[width=0.3\textwidth]{ 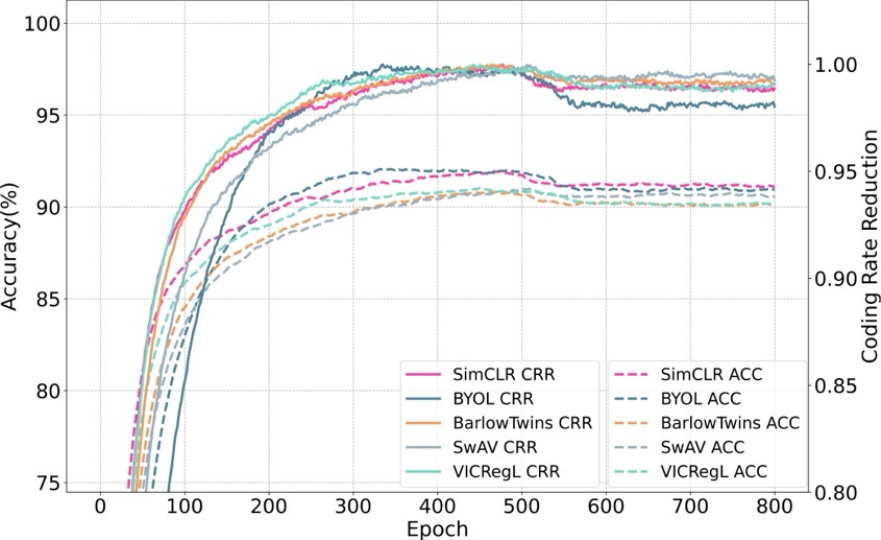}
    }
    \subfigure[CIFAR-100]{
    \includegraphics[width=0.3\textwidth]{ 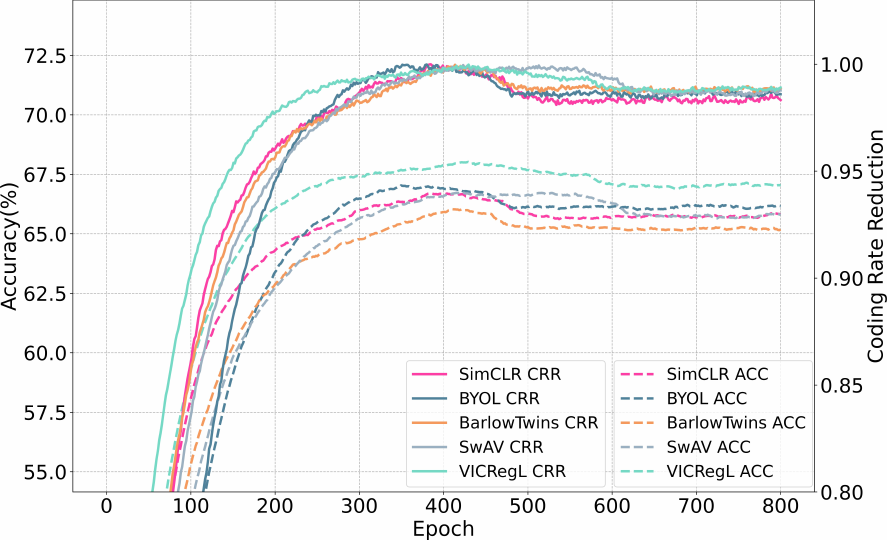}
    }
    \subfigure[Tiny-ImageNet]{
    \includegraphics[width=0.3\textwidth]{ 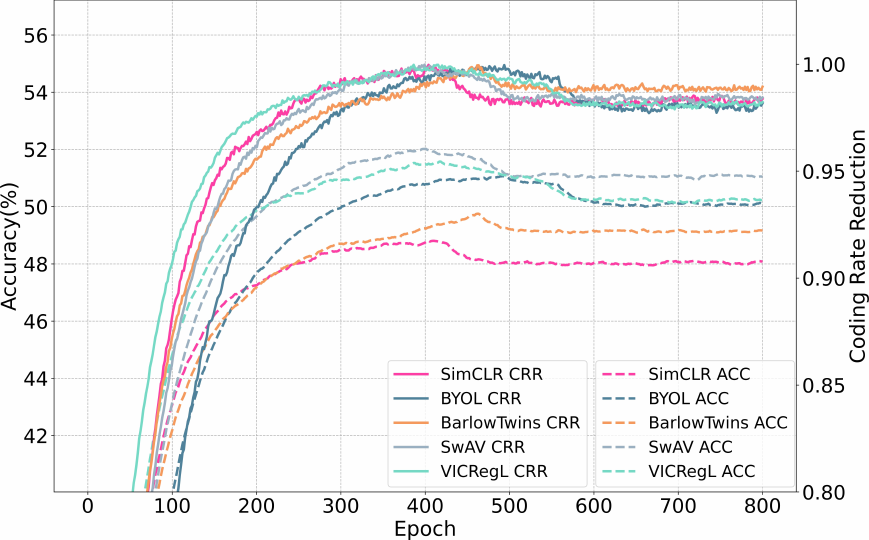}
    }
    \subfigure[ImageNet-100]{
    \includegraphics[width=0.3\textwidth]{ 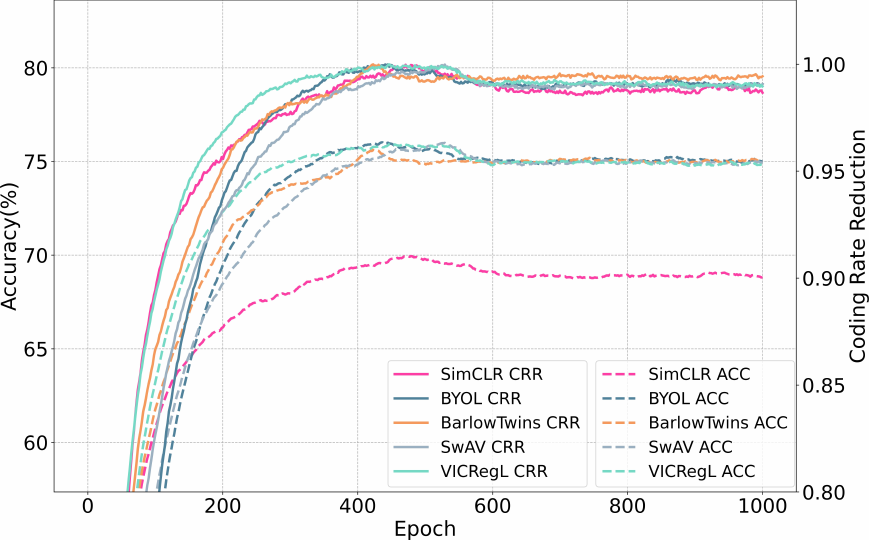}
    }
    \subfigure[ImageNet]{
    \includegraphics[width=0.3\textwidth]{ 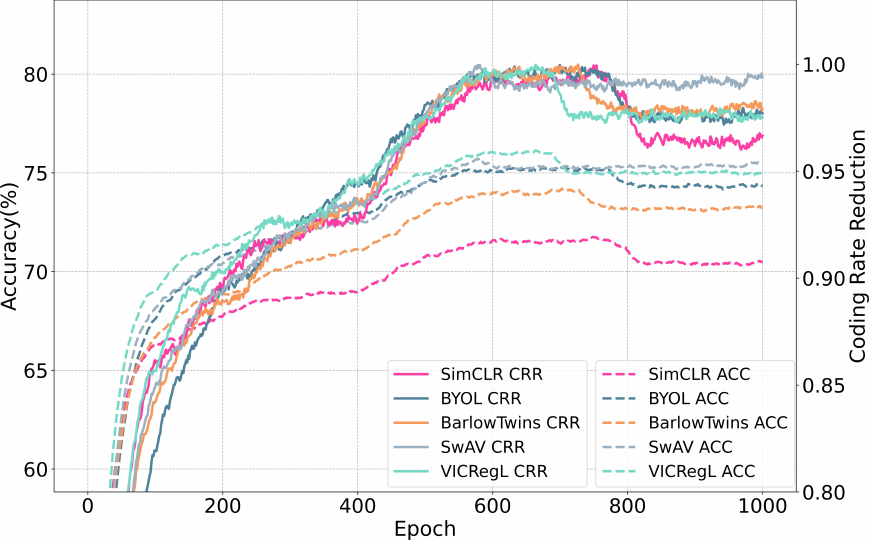}
    }
    \caption{The comparison of the trend of the coding rate reduction and the trend of test accuracy for different SSL methods and datasets. The results from (a) - (e) are based on early-layer output while results from (f) - (j) are based on last-layer output.}
    \label{Fig.compare}
\end{figure*}

\begin{figure*}[htb]
    \centering
    \subfigure[SimCLR]{
    \includegraphics[width=0.3\textwidth]{ 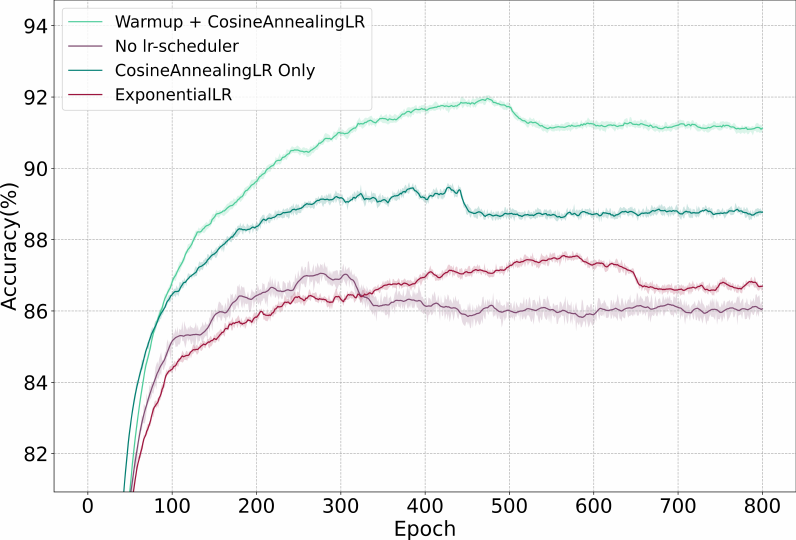}
    }
    \subfigure[BYOL]{
    \includegraphics[width=0.3\textwidth]{ 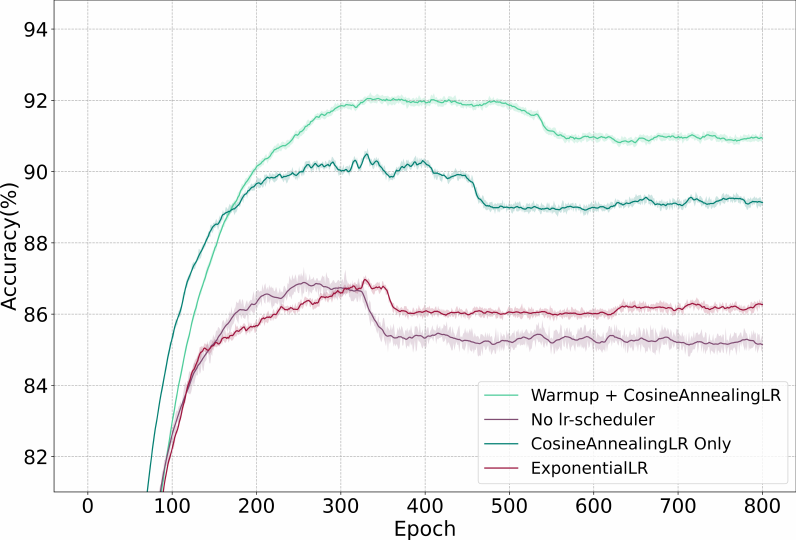}
    }
    \subfigure[VICRegL]{
    \includegraphics[width=0.3\textwidth]{ 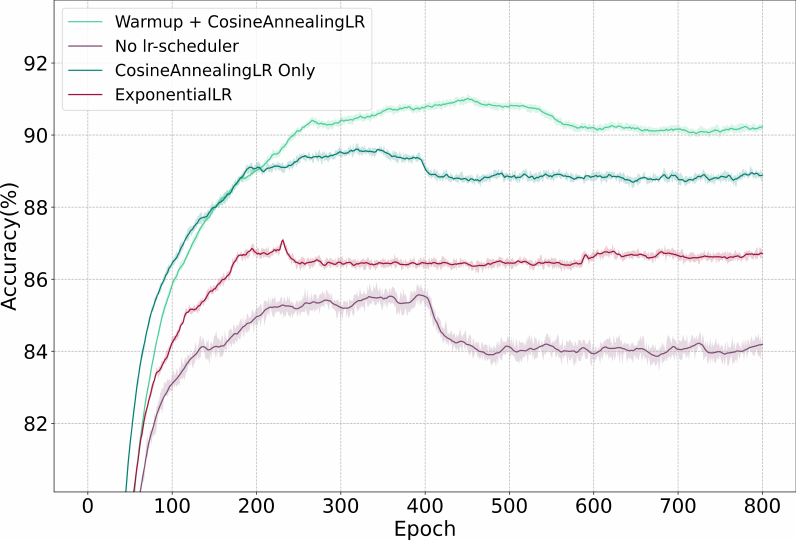}
    }
    \subfigure[SimCLR]{
    \includegraphics[width=0.3\textwidth]{ 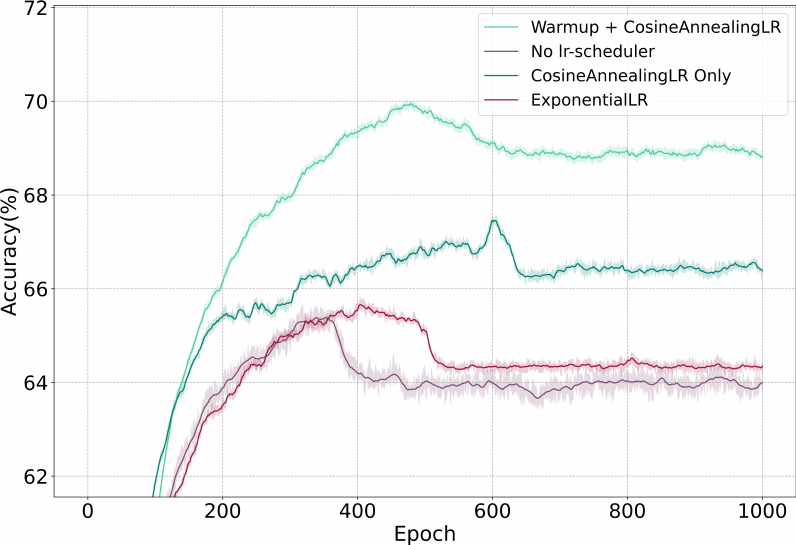}
    }
    \subfigure[BYOL]{
    \includegraphics[width=0.3\textwidth]{ 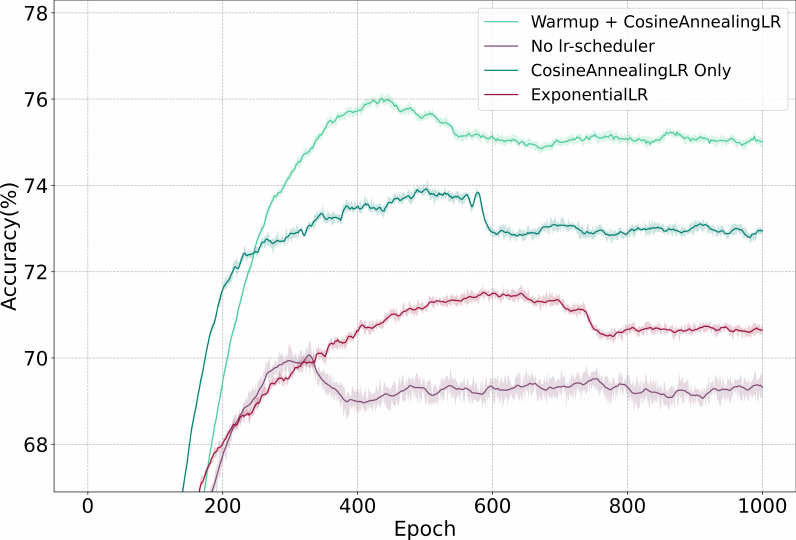}
    }
    \subfigure[VICRegL]{
    \includegraphics[width=0.3\textwidth]{ 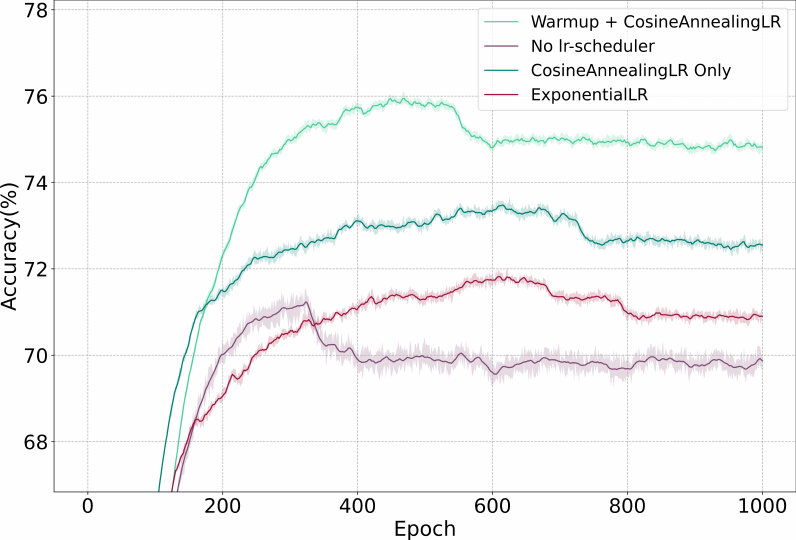}
    }
    \caption{The curves of test accuracy based on last-layer output versus training epoch under different learning rate scheduler. The results from (a) - (c) are based on the CIFAR-10 dataset. The results from (d) - (e) are based on the ImageNet-100 dataset.}
    \label{Fig.lr_sche}
\end{figure*}

\begin{table}[t]
	\begin{center}
		\caption{The classification accuracies of a linear classifier with a ResNet-18
as the feature extractor for simultaneous training of SSL and UMM. As well as the results for simultaneous training of SSL with ${\mathcal{L}}(\Delta R,Z_{tr - el}^{aug},f_{e - l}^*)$ and ${\rm{JS}}(p(Z_{tr - el}^{aug})|p(Z_{tr}^{aug}))$.}
		\label{tab:stfvhg}
		\setlength\tabcolsep{3.9pt}
  \resizebox{\linewidth}{!}{
		\begin{tabular}{c|c|c|c}
			\toprule
			{Method}  & {CIFAR-100}  & {STL-10} & Tiny ImageNet\\ 
			\midrule
                SimSiam \& UMM & 67.92$\pm$0.21 & 92.01$\pm$0.17 & 52.04$\pm$0.21 \\
			SimCLR \& UMM & 67.59$\pm$0.19 & 91.55$\pm$0.25 & 50.10$\pm$0.18 \\ 
			BYOL \& UMM & 67.66$\pm$0.20 & 93.57$\pm$0.22 & 52.41$\pm$0.21 \\ 
			Barlow-Twins \& UMM & 67.01$\pm$0.15 & 92.01$\pm$0.19 & 50.49$\pm$0.21\\ 
			W-MSE \& UMM & 68.37$\pm$0.17 & 92.68$\pm$0.15 &  51.03$\pm$0.24 \\ 
			VICRegL \& UMM & 69.01$\pm$0.23 & 92.48$\pm$0.27 & 51.98$\pm$0.26 \\
			\midrule
   SimSiam + ${\mathcal{L}}$ + ${\rm{JS}}$ & 67.82$\pm$0.20 & 91.99$\pm$0.18 & 52.01$\pm$0.22 \\
			SimCLR + ${\mathcal{L}}$ + ${\rm{JS}}$ & 67.29$\pm$0.20 & 91.25$\pm$0.20 & 50.01$\pm$0.20 \\ 
			BYOL + ${\mathcal{L}}$ + ${\rm{JS}}$ & 67.21$\pm$0.18 & 93.38$\pm$0.20 & 52.21$\pm$0.20 \\ 
			Barlow-Twins + ${\mathcal{L}}$ + ${\rm{JS}}$& 66.88$\pm$0.19 & 91.89$\pm$0.20 & 50.33$\pm$0.20\\ 
			W-MSE + ${\mathcal{L}}$ + ${\rm{JS}}$ & 68.13$\pm$0.18 & 92.47$\pm$0.20 &  50.22$\pm$0.21 \\ 
			VICRegL + ${\mathcal{L}}$ + ${\rm{JS}}$ & 68.51$\pm$0.20 & 92.08$\pm$0.18 & 51.68$\pm$0.21 \\
			\bottomrule
		\end{tabular}}
	\end{center}
\end{table}

Fourthly, in all experiments conducted in Figure \ref{Fig.main1} and \ref{Fig.main3}, we employ the same learning rate scheduler strategy, which includes a learning rate warm-up for 20 epochs followed by cosine annealing. The warm-up process starts with a learning rate of $1\times 10^{-5}$ and ends with a rate of $1\times 10^{-3}$. The minimum learning rate of the cosine annealing scheduler is set to $1\times 10^{-6}$. To investigate whether different learning rate schedulers influence the phenomenon of overfitting, we compare four commonly used learning rate scheduler methods based on the features extracted from the last layer, e.g., the warm-up followed by cosine annealing, constant learning rate (no learning rate scheduler), cosine annealing only, and exponential decay learning rate scheduler. For the solely employing cosine annealing, the initial and minimum LR values of the LR scheduler are identical to those in Figure \ref{Fig.main1} and \ref{Fig.main3}. For the exponential decay, the multiplicative factor $\varsigma$ is computed using $\varsigma=(\gamma_{e} / \gamma_{s}^{\frac{1}{T}})$, where $\gamma_{s}=1\times 10^{-3}$ denotes the initial LR, $\gamma_{e}=1\times 10^{-5}$ denotes the final LR, and $T$ represents the number of training iterations. We conduct these experiments on a small-scale dataset, CIFAR-10, as well as a large-scale dataset, ImageNet-100. Additionally, we compare three different SSL methods, namely SimCLR, BYOL, and VICRegL. The results are shown in Figure \ref{Fig.lr_sche}, we observe that the choice of different learning rate schedulers has a significant impact on test accuracy. Specifically, when learning without a learning rate scheduler, the test accuracies are lower than all other compared methods. Additionally, the warm-up + CosineAnnealing method we adopted exhibits higher test accuracies than other LR schedulers. It is worth noting that although the epochs required to reach the maximum vary, the trend of test accuracy is to first increase, then decrease after reaching the maximum, and finally stabilize, regardless of the learning rate scheduler method used. This indicates that the overfitting phenomenon we describe in Subsection \ref{sec:4.1} is common across all SSL methods and is independent of the choice of learning rate scheduler.

Finally, UMM is constructed upon the foundation of pre-trained SSL methods. An inherent challenge is whether it is feasible to incorporate UMM during the training of SSL methods. To address this question, we introduce UMM into the objective function of the SSL method during its training phase. Specifically, during training, we first update the feature extractor $f$ and projection head $f_{ph}$ using Equation \ref{cvdskop}. Subsequently, we fix $f_e$ and proceed to update $f_{e-l}$ based on Equation \ref{umm:eq}. This two-step process is iteratively carried out during the training phase until convergence is achieved. We denote this training style as \{SSL method\} \& UMM, e.g., SimCLR \& UMM. The final results are summarized in Table \ref{tab:stfvhg}. We can observe that the test accuracy of \& UMM surpasses that of directly training the SSL method, yet it falls short of the test accuracy achieved by the $+$ UMM approach. This suggests that \& UMM proves effective both during the training phase and the fine-tuning phase. However, one possible reason for \& UMM's performance being inferior to $+$ UMM performance could be attributed to the early stages of training. During this initial phase, neither the features in the early layer nor those in the last layer have overfit, and their individual performances may not be sufficiently strong. Distilling information from the early layer features into the last layer features at this point might result in an increased presence of task-irrelevant information in the last layer features. It is only when the early layer features demonstrate significant improvement and undergo overfitting that the distillation operation truly becomes effective. Meanwhile, UMM is trained in a bi-level optimization paradigm, and we have explained in Subsection \ref{umm;qw} why we use this paradigm. To further explore the effectiveness of training UMM by simultaneously using ${\mathcal{L}}(\Delta R,Z_{tr - el}^{aug},f_{e - l}^*)$ and ${\rm{JS}}(p(Z_{tr - el}^{aug})|p(Z_{tr}^{aug}))$, we have conducted experiments based on 6 methods, and 3 datasets, and the experimental results are shown in Table \ref{tab:stfvhg}. We can observe that the experimental accuracies of training UMM by simultaneously using ${\mathcal{L}}(\Delta R,Z_{tr - el}^{aug},f_{e - l}^*)$ and ${\rm{JS}}(p(Z_{tr - el}^{aug})|p(Z_{tr}^{aug}))$ is lower then these of training UMM in a bi-level optimization paradigm. This demonstrates the effectiveness of the bi-level design. Also, we can conclude that UMM indeed first adjusts ${\mathcal{L}}(\Delta R,Z_{tr - el}^{aug},f_{e - l}^*)$ using ${\rm{JS}}(p(Z_{tr - el}^{aug})|p(Z_{tr}^{aug}))$, and then uses ${\mathcal{L}}(\Delta R,Z_{tr - el}^{aug},f_{e - l}^*)$ to regulate the feature learning process of SSL.

\section{Conclusion}
In this paper,We find that in the SSL training process, overfitting occurs abruptly in later layers and epochs, while generalizing features are learned in early layers and epochs. Then we find that coding rate reduction can be used as an indicator of overfitting. Next, we propose Undoing Memorization Mechanism (UMM), a bi-level optimization-based learning mechanism that undoes memorization of by rewinding. We provide a causal analysis and empirical evaluation to show the effectiveness of UMM. 

\textbf{Limitations}. Our empirical analysis is based on the classification tasks. For regression tasks, generation tasks, etc., the relevant empirical analysis is also worth exploring.


\section*{Data Availability}

The benchmark datasets can be downloaded from the literature cited in Subsubsection \ref{sec:exp}.

\section*{Conflict of interest}
  
The authors declare no conflict of interest.
  
  
\bibliographystyle{unsrt}
\bibliography{reference}

\clearpage
\appendix

\section{Further Explanation for Overfitting}


According to Figure 1 and Figure 3 in the main paper, we observe that the accuracy curve and the curve of the coding rate reduction based on the last layer output features of the network first increases, then decreases after reaching the peak, and finally stabilizes. We know that the training accuracy during training is basically close to 100\%. Therefore, this phenomenon indicates that overfitting occurs during the training process, otherwise the accuracy would not decline.

\begin{figure*}[h]
\centering  
\subfigure[CIFAR-10]{
\label{Fig.1sub.1.app}
\includegraphics[width=0.315\textwidth]{ 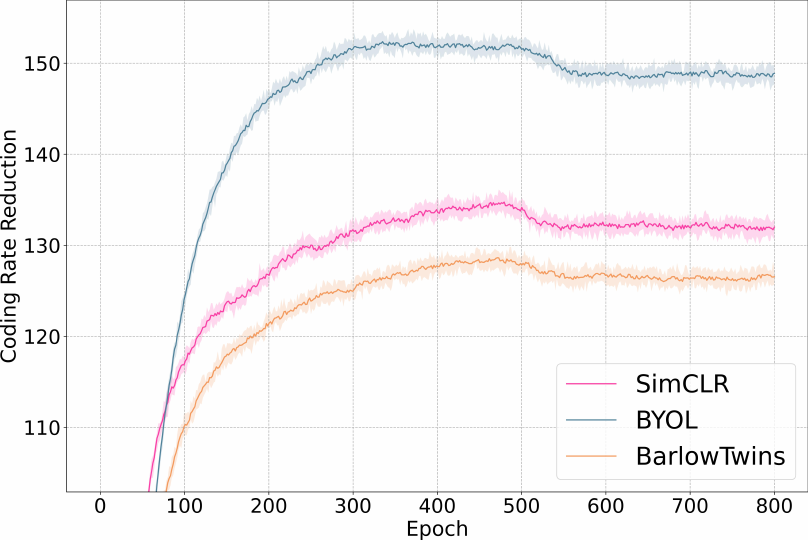}}
\subfigure[CIFAR-100]{
\label{Fig.1sub.2.app}
\includegraphics[width=0.315\textwidth]{ 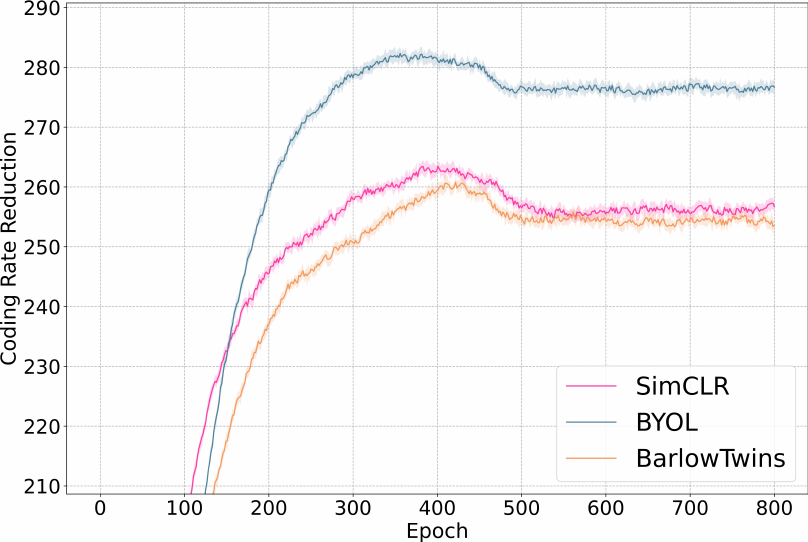}}
\subfigure[Tiny-ImageNet]{
\label{Fig.1sub.3.app}
\includegraphics[width=0.315\textwidth]{ 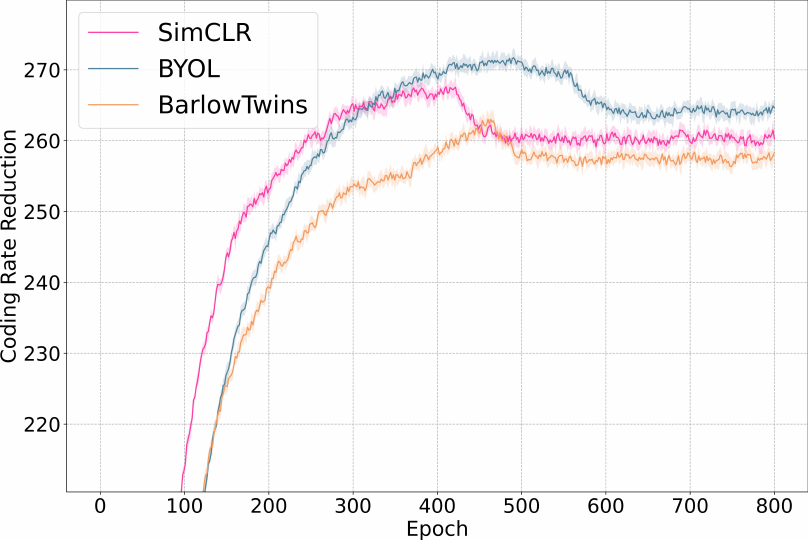}}
\caption{The curves of coding rate reduction versus training epoch.}
\label{123qeesedv}
\end{figure*}

\section{Additional Experimental Results}

\subsection{Influence of Hyperparameter $\varepsilon$}
To understand the impacts of hyperparameter $\varepsilon$, we select SimCLR+UMM, BYOL+UMM, and Barlow Twins+UMM to conduct comparisons on the CIFAR-10 dataset. Specifically, $\varepsilon$ controls the calculation of the coding rate reduction. We first set $\alpha  = 1, \beta = 0.1$ and select $\varepsilon$ from the range of $\{ 0.1, 0.2,...,0.9, 1 \}$. The results are shown in Table \ref{tab:vasdasdoc}. We observe that the test accuracy of different methods is almost constant as $\varepsilon $ changes. This illustrates that the proposed UMM is insensitive
to $\varepsilon $. Therefore, in this paper, we set $\varepsilon = 1$ for all methods and datasets.

\subsection{The Quantified Correlation Analysis}
\label{app:corr}

To explore the correlation between the test accuracy and the value of the coding rate reduction in a quantified manner, we measure the correlation of the trend of test accuracy and $\Delta R(Z,\Pi ,\varepsilon )$ with Pearson correlation coefficient in Table \ref{tab:pcc}. To do this, we first need to obtain the corresponding vectors. Since we save the parameters every 5 epochs during the training process, we can discretize each curve in Figure \ref{Fig.compare} as a vector, with each dimension of the vector representing the value of the test accuracy or the value of the coding rate reduction obtained by one of the saved models. In this way, we obtain some columns of vectors based on the test accuracy and vectors based on the value of the coding rate reduction. Pearson correlation coefficient closer to 1 in the table indicates that the compared trend exhibits a stronger correlation. From Table \ref{tab:pcc}, we observe that nearly all correlation coefficients between test accuracy and coding rate reduction are close to 1. This finding further supports the effectiveness and reasonableness of using coding rate reduction to monitor the occurrence of overfitting in SSL.

\begin{table}[h]
    \centering
    \caption{The average Pearson correlation coefficient between testing accuracy and coding rate reduction for distinct datasets and methods.}
    \resizebox{\linewidth}{!}
    {
    \begin{tabular}{c|c|c|c|c|c}
    \toprule
       Method  & CIFAR-10 & CIFAR-100 & Tiny-ImageNet & ImageNet-100 & ImageNet \\
    \midrule
       SimCLR  & 0.94 & 0.97 & 0.93 & 0.95 & 0.98 \\
       BYOL & 0.92 & 0.94 & 0.96 & 0.98 & 0.97 \\
       Barlow Twins & 0.94 & 0.95 & 0.96 & 0.93 & 0.97 \\
       SwAV & 0.94 & 0.92 & 0.95 & 0.98 & 0.94 \\
       VICRegL & 0.95 & 0.97 & 0.97 & 0.94 & 0.98 \\
    \bottomrule
    \end{tabular}
    }
    \label{tab:pcc}
\end{table}

\begin{table}[htb]
\centering
\setlength{\tabcolsep}{2pt}
\caption{The influence of hyperparameter $\varepsilon$ under linear evaluation}
\resizebox{\linewidth}{!}{\begin{tabular}{lccccccccc}
\toprule
Method  & 0.1 & 0.2 & 0.3 & 0.4 & 0.5 & 0.6 & 0.7 & 0.8 & 0.9  \\
\midrule
SimCLR+UMM  & 92.93 & 92.92 & 92.94 & 92.93 & 92.92 & 92.93 & 92.93 & 92.93 & 92.93        \\
BYOL+UMM   & 92.98 & 92.99 & 92.99 & 92.98 & 93.00 & 92.98 & 92.99 & 92.99 & 92.99          \\
Barlow Twins+UMM  & 91.54 & 91.55 & 91.56 & 91.55 & 91.54 & 91.55 & 91.55 & 91.55 & 91.55      \\
\bottomrule
\end{tabular}}
\label{tab:vasdasdoc}
\end{table}

\section{Further discussion about $\Pi^j(i,i)$}

In the definition of coding rate, we assume that there are $n$ samples, for class $j$, we denote the probability of the $i$-th sample as $\Pi^j(i,i)$. In this scenario, the class is given in advance, so each sample can have a probability value. Thus, we have $\Pi^j(i,i)\subseteq\mathbb{R}^{n\times n}$. However, in SSL scenarios, the number of training samples is $n$, then, we can obtain $2n$ augmented samples. Next, we regard one of samples as anchor, which can be regard as the index of class. Thus, we are left with only 2n-1 samples. We denote the $j$-th sample as anchor, then $\Pi^j(i,i)$ can be considered as the probability that the $i$-th and the anchor have the same ground-truth label, where $i\ne j,j\subseteq \{1,...,2n\}$. Thus, we have $\Pi^j(i,i)\subseteq\mathbb{R}^{(2n-1)\times (2n-1)}$ in SSL scenarios.

In SSL scenarios, the number of training samples is $n$, then, we can obtain $2n$ augmented samples.
Next, we regard the $j$-th sample as anchor, we can obtain that among the remaining $2n-1$ augmented samples, the augmented samples with the same ancestor as the anchor (see the first paragraph of Section \ref{pre3435} for the definition of ancestor) are considered as positive samples with respect to the anchor, while the remaining $2n-2$ samples are considered as negative samples with respect to the anchor. Thus, we only need to obtain $2n-1$ probability values indicating the probability that each sample and the anchor belong to the same class. We get the corresponding probability based on the feature similarity between the remaining augmented samples and the anchor. Thus we have only $2n-1$ terms in the denominator.

In the original definition of coding rate, the class is given in advance, $\Pi^j(i,i)$ can be regarded as the probability of the $i$-th sample belonging to the $j$-th class. Suppose that there are $K$ classes, then we can regard $\{ \Pi^1(i,i),...,\Pi^K(i,i) \} $ as the probability distribution of the category to which it corresponds given the $i$-th sample. For SSL scenarios, we can also obtain the probability distribution of the category to which it corresponds given the $i$-th sample. Because the category in SSL scenarios is refer to anchor, so if we denote the augmented samples as $\{x_{1},...,x_{2n} \}$, for the $i$-th sample, the probability distribution of the category of the $i$-th sample is denote as $\{ \Pi^1(i-1,i-1),...,\Pi^{i-1}(i-1,i-1),\Pi^{i+1}(i,i),...,\Pi^{2n}(i,i) \} $ not $\{ \Pi^1(i,i),...,\Pi^{2n-1}(i,i) \} $.

\section{Proofs}

\subsection{Proofs of Lemma 1}

\renewcommand{\thelemma}{1}
\begin{lemma}
    Denote the volume of a region $w \in \Omega $ as ${\rm{vol}}(w)$ and the volume of the affinely transformed region $S(w)$ as ${\rm{vol}}(S(w))$, we have: ${\rm{vol}}(S(w)) = {\rm{vol}}(w)\sqrt {\det (A_w^{\rm{T}}{A_w})}$.
\end{lemma}

\begin{proof}
From \cite{ruzhansky2017advances} and the Theorem 7.26 in \cite{rudin2006real}, given an affine transform of the coordinate ${A_w} \in {\mathbbm{R}^{d \times d}}$, we can obtain that the change of volume for ${A_w}$ is given by $|\det ({A_w})|$. However, in the case of this paper, the mapping can be regarded as a rectangular matrix. This is caused by that spaning an affine subspace in the ambient space $\mathbbm{R}^{d}$ will lead to $|\det ({A_w})|$ not defined.

We can see that a surface produced by a per-region affine mapping belongs to the Riemannian manifold. Thus, the volume for a mapping can be given by the square root of the determinant of the metric tensor. Then, for a surface with intrinsic dimension $n$ embedded in Euclidean space of dimension $m$ (which in the scenario of our paper, the affine mapping per region yields an affine subspace) parameterized by the function $S:{\mathbbm{R}^n} \mapsto {\mathbbm{R}^m}$. This function is simply the affine mapping $S(x) = {A_w}x + {b_w}$ for each region. Thus, the metric tensor is expressed as $g = D{S^{\rm{T}}}DS$, where $D$ represents the Jacobian/differential operator, which in this case, $g={{A_w}^{\rm{T}}}{A_w}$ for each region. This outcome is also recognized as Sard's theorem \cite{spivak2018calculus}. Consequently, the change in volume from region $w$ to the affine subspace is determined by $\sqrt {\det ({{A_w}^{\rm{T}}}{A_w})}$, which alternatively can be expressed using the SVD decomposition of the matrix $A$ as follows:
\begin{equation}
    \begin{array}{l}
\sqrt {\det ({{A_w}^{\rm{T}}}{A_w})}  \\
= \sqrt {\det ({{(U\Lambda {V^{\rm{T}}})}^{\rm{T}}}(U\Lambda {V^{\rm{T}}}))}  \\
= \sqrt {\det ((V{\Lambda ^{\rm{T}}}{U^{\rm{T}}})(U\Lambda {V^{\rm{T}}}))} \\
 = \sqrt {\det (V{\Lambda ^{\rm{T}}}\Lambda {V^{\rm{T}}})}  = \sqrt {\det ({\Lambda ^{\rm{T}}}\Lambda )}  \\
 = \prod\limits_{i:{\sigma _i} \ne 0} {{\sigma _i}({A_w})} 
\end{array}
\end{equation}
and
\begin{equation}
\int\limits_{Aff(w,{A_w},b)} {dx}  = \frac{1}{{\sqrt {\det ({{A_w}^{\rm{T}}}{A_w})} }}\int\limits_w {dz}
\end{equation}

Therefore, we can obtain that Lemma 1 holds.
\end{proof}

\subsection{Proofs of Theorem 1}

\renewcommand{\thetheorem}{1}
\begin{theorem}
    Denote the probability density of the input space as $p(x)$ and the probability density generated by $S$ as $p_S(z)$, where $z=S(x)$, we have:
    \begin{equation}
        {p_S}(z) = \sum\nolimits_{w \in \Omega } {\frac{{p({{(A_w^{\rm{T}}{A_w})}^{ - 1}}A_w^{\rm{T}}(z - {b_w}))}}{{\sqrt {\det (A_w^{\rm{T}}{A_w})} }}} {\mathbbm{1}_{x \in w}}
    \end{equation}
\end{theorem}

\begin{proof}
Let variables $z$ change as $z=(A_w^{\rm T}A_w)^{-1}{A_w^{\rm T}}(x-b_w)=A_w^{+}(x-b_w)$, we can obtain $J_{G^{-1}}(x)=A^+$. Then, we can obtain $P_{G(z)}(x \in \omega )=P_z(z \in G^{-1}(\omega))=\int_{{G^{ - 1}}(\omega)} {{p_z}(z)dz}$ with the invertible assumption. Then, we can obtain that
\begin{equation}\label{sdasfasdasd}
\begin{array}{l}
{P_G}(x \in w)\\
 = \sum\limits_w {\int_{w \cap \omega } {{p_z}({G^{ - 1}}(x))\sqrt {\det ({J_{{G^{ - 1}}}}{{(x)}^{\rm{T}}}{J_{{G^{ - 1}}}}(x))} dx} } \\
 = \sum\limits_w {\int_{w \cap \omega } {{p_z}({G^{ - 1}}(x))\sqrt {\det ({{(A_w^ + )}^{\rm{T}}}A_w^ + )} dx} } \\
 = \sum\limits_w {\int_{w \cap \omega } {{p_z}({G^{ - 1}}(x))(\prod\limits_{i:{\sigma _i}(A_w^ + ) > 0} {{\sigma _i}(A_w^ + )} )dx} } \\
 = \sum\limits_w {\int_{w \cap \omega } {{p_z}({G^{ - 1}}(x)){{(\prod\limits_{i:{\sigma _i}({A_w}) > 0} {{\sigma _i}({A_w})} )}^{ - 1}}dx} } \\
 = \sum\limits_w {\int_{w \cap \omega } {{p_z}({G^{ - 1}}(x))\frac{1}{{\sqrt {\det (A_w^{\rm{T}}{A_w})} }}dx} } 
\end{array}
\end{equation}

The fourth step in Equation \ref{sdasfasdasd} can be obtained by proving ${\sigma _i}(A{}^ + ) = {({\sigma _i}(A))^{ - 1}}$. First, we have:
\begin{equation}
\begin{array}{l}
{A^ + }\\
 = {({A^{\rm{T}}}A)^{ - 1}}{A^{\rm{T}}}\\
 = {({(U\Lambda {V^{\rm{T}}})^{\rm{T}}}(U\Lambda {V^{\rm{T}}}))^{ - 1}}{(U\Lambda {V^{\rm{T}}})^{\rm{T}}}\\
 = {((V{\Lambda ^{\rm{T}}}{U^{\rm{T}}}U\Lambda {V^{\rm{T}}}))^{ - 1}}{(U\Lambda {V^{\rm{T}}})^{\rm{T}}}\\
 = {((V{\Lambda ^{\rm{T}}}\Lambda {V^{\rm{T}}}))^{ - 1}}V{\Lambda ^{\rm{T}}}{U^{\rm{T}}}\\
 = V{({\Lambda ^{\rm{T}}}\Lambda )^{ - 1}}{\Lambda ^{\rm{T}}}{U^{\rm{T}}}
\end{array}
\end{equation}

Then, we can obtain: ${\sigma _i}(A{}^ + ) = {({\sigma _i}(A))^{ - 1}}$ and $\sqrt {\det ({{(A_w^ + )}^{\rm{T}}}A_w^ + )}  = 1/\sqrt {\det (A_w^{\rm{T}}{A_w})} $. Then, we have:
\begin{equation}
\begin{array}{l}
\sqrt {\det ({{(A_w^ + )}^{\rm{T}}}A_w^ + )}  = \prod\limits_{i:{\sigma _i} \ne 0} {{\sigma _i}(A_w^ + )}  = \prod\limits_{i:{\sigma _i} \ne 0} {{\sigma _i}{{({A_w})}^{ - 1}}} \\
 = {(\prod\limits_{i:{\sigma _i} \ne 0} {{\sigma _i}({A_w})} )^{ - 1}} = \frac{1}{{\sqrt {\det (A_w^{\rm{T}}{A_w})} }}
\end{array}
\end{equation}

Therefore, we can obtain that Theorem 1 holds.
\end{proof}

\subsection{Proofs of Theorem 2}

\renewcommand{\thetheorem}{2}
\begin{theorem}
    Given the ${\rm{SCM}}$ shown in Figure \ref{Fig.11subqw.1}, let $f$ denote the feature extractor, let $X^{aug}_1$ and $X^{aug}_2$ denote the corresponding augmented datasets generated by $X$ using different data augmentations. If $f$ can minimize the following objective function:
    \begin{align}\label{eq:proof_2}
        \mathcal{L}(f,{X^{aug}_1},{X^{aug}_2}) ={\rm D}(f(X_1^{aug}),f(X_2^{aug})) \nonumber \\
        - [H(f({X^{aug}_1})) + H(f({X^{aug}_2}))]
    \end{align}
    where ${\rm D}(\cdot)$ is the distance operator for two discrete distributions and ${H}(\cdot)$ is the information entropy, then $f({X^{aug}_1})$ and $f({X^{aug}_2})$ contain all of task-related information $S_{r}$.
\end{theorem}

\begin{table*}[t]
\centering
\caption{Symbol definitions for the proof of Theorem \ref{theo:fac}}
\resizebox{\linewidth}{!}{
\begin{tabular}{|c|l|}
\hline
\textbf{Symbol} & \textbf{Definition} \\ \hline
$\mathcal{L}(f,X_1^{aug},X_2^{aug})$ & Loss function, where $X_1^{aug}$ and $X_2^{aug}$ represent augmented views of the data \\ \hline
$d_r$ & The dimension of task-relevant information \\ \hline
$f$ & A function composed of two parts: one mapping from $X$ to $S_r$ and one mapping from $S_r$ to a uniform distribution \\ \hline
${\rm {f}}^{-1}_{1:d_r}$ & The inverse mapping from $X$ to $S_r$, which extracts task-relevant information \\ \hline
$S_r$ & The set of task-relevant information \\ \hline
${\rm F}$ & A mapping from $S_r$ to $p_{ud}$, where $p_{ud}$ represents a uniform distribution \\ \hline
$f_c$ & The $c$-th dimension of the output of the function $f$ \\ \hline
$\varphi$ & A function used to define the difference in $f_c$ \\ \hline
$S_{ur}$ & The set of unconstrained variables \\ \hline
$S_{-l}$ & The set obtained by removing the $l$-th element from the set $S$ \\ \hline
$s_l$ & The $l$-th element of the set $S_{ur}$ \\ \hline
$O_{s^{*}_l}$ & A neighborhood of the set $s_l^*$ \\ \hline
\end{tabular}
}
\label{tab:symbol_definitions}
\end{table*}

\begin{proof}

We first provide the Table \ref{tab:symbol_definitions} for the proof to illustrate each symbol and the corresponding definition. Then, we demonstrate that the representation $f(X)$ derived from any smooth function $f$ minimizing Equation \ref{eq:proof_2} is connected to the true latent variable through a smooth mapping. Additionally, $f(X)$ must exhibit invariance across $(X_1^{aug},X_2^{aug})$ almost surely (a.s.) with respect to the true generative process, and it must adhere to a uniform distribution.


The global minimum of $\mathcal{L}(f,X_1^{aug},X_2^{aug})$ is reached when the distance operator is minimized (i.e., equal to zero) and the information entropy is maximized. Without additional moment constraints, the unique maximum entropy distribution is the uniform distribution \cite{darmois1951analyse,hyvarinen1999nonlinear}. First, we show that there exists a smooth function $f$ which attains the global minimum of $\mathcal{L}(f,X_1^{aug},X_2^{aug})$.

Given the function ${\rm {f}}^{-1}{1:d_r}: X \to S_r$, where the initial $d_r$ dimensions of the output of ${\rm {f}}^{-1}$ denote the task-relevant details, we proceed to construct a function ${\rm F}: S_r \to p_{ud} $ utilizing a recursive method known as the Darmois construction \cite{hyvarinen1999nonlinear}, where $p_{ud}$ denotes the uniform distribution. Thus, we obtain:
\begin{equation}
{{\rm{F}}_i}({S_r}){\rm{: = }}{{{F}}_i}(S_r^i|S_r^{1:i - 1}) = P(z \le S_r^i|S_r^{1:i - 1})
\end{equation}
where ${\rm{F}}_i$ represents the $i$-th dimension of the output of ${\rm F}$, $S_r^i$ represents the $i$-th dimension of $S_r$, $S_r^{1:1-i}$ represents the 1 to $i-1$-th dimensionw of $S_r$, $i = 1,....,{d_{{r}}}$, ${d_{{r}}}$ is the dimension of $S_r$, and ${{F}}_i$ represents the conditional cumulative distribution function. From \cite{hyvarinen1999nonlinear}, we can see that this  construction leads to a uniform distribution.

Then, we define $f = {\rm{F}} \circ {\rm{f}}_{1:{d_r}}^{ - 1}$, which is a smooth function since it is a composition of two smooth functions. Let ${\rm {f}}^{-1}_{1:d_r}(X_1^{aug})=S_r^1$ and ${\rm {f}}^{-1}_{1:d_r}(X_2^{aug})=S_r^2$, we have:
\begin{equation}
\begin{array}{cl}
\mathcal{L}(f,{X^{aug}_1},{X^{aug}_2}) & ={\rm D}(f(X_1^{aug}),f(X_2^{aug})) - [H(f({X^{aug}_1}))\\
&\quad + H(f({X^{aug}_2}))] \\
&=0
\end{array}
\end{equation}
In this equation, we have used the fact that $X_1^{aug}=X_2^{aug}$ almost surely w.r.t.\ to the ground truth generative process, so the first term is zero; and the fact that $f({X^{aug}_1})$ is uniformly distributed on $(0,1)^{n_c}$ and the uniform distribution on the unit hypercube has zero entropy, so the second term is also zero. 

Next, let $f$ be any smooth function that can obtain the global minimum of $\mathcal{L}(f,X_1^{aug},X_2^{aug})$. Define $h := f  \circ {\rm f}$, we can obtain that the first term and second term of $\mathcal{L}(f,X_1^{aug},X_2^{aug})$ all equal to $0$ (as mentioned above). Therefore, we can conclude that the representation extracted by any smooth function $f$ that minimizes $\mathcal{L}(f,X_1^{aug},X_2^{aug})$ is related to the true latent representations that are invariant to augmentation and follow a uniform distribution.

In the next step, we aim to demonstrate that $f(X)$ can only depend on $S_r$ instead of $S_{ur}$. Suppose that $f_r(S_r,S_{ur}):=f(S_r,S_{ur})_{1:n_r}=f(z)_{1:n_r}$, we assume that the partial derivative of $f_r$ w.r.t. some variable $s_l$ in $S_{ur}$ is non-zero at some point $z^*\in \mathcal{Z}={S_r}\times S_{ur} $. By continuity of the partial derivative, $\partial {f_r}/\partial {s_l}$ must be non-zero in a neighborhood of $z^*$.

Then, suppose that $\varphi :{S_r} \times {S_{ur}} \times {S_{ur}} \to {\mathbbm{R}_{ \ge 0}}$, and the can be defined as follows:
\begin{equation}\label{eq:proof_3_1}
    \varphi ({S_r},{S_{ur}},{S_{ur}}): = |{f_c}({S_r},{S_{ur}}) - {f_c}({S_r},{{\mathord{\buildrel{\lower3pt\hbox{$\scriptscriptstyle\frown$}} 
\over S} }_{ur}})| \ge 0
\end{equation}
where $f_c$ represents the $c$-th dimension of the output of $f$. Typically, when constructing neural networks or other mapping functions, there are multiple output dimensions, each of which can individually represent specific information. To obtain a contradiction to the invariance condition mentioned above, it remains to show that $\varphi$ is \textit{strictly positive} with probability greater than zero. Consider $\partial {f_r}/\partial {s_l}$ must be non-zero in a neighborhood of $z^*$, the strict monotonicity implies that:
\begin{equation}
\begin{array}{l}
    \varphi (S_r^*,(S_{ - l}^*,{s_l}),(S_{ - l}^*,{{\mathord{\buildrel{\lower3pt\hbox{$\scriptscriptstyle\frown$}} 
\over s} }_l})) > 0 \\[8pt]
\forall ({s_l},{{\mathord{\buildrel{\lower3pt\hbox{$\scriptscriptstyle\frown$}} 
\over s} }_l}) \in (s_l^*,s_l^* + \varepsilon ) \times (s_l^* - \varepsilon ,s_l^*)
\end{array}
\end{equation}
Since $\varphi$ is a composition of continuous functions (absolute value of the difference of two continuous functions), $\varphi$ is continuous.

Consider the open set ${\mathbbm{R}_{ \ge 0}}$, then, for a continuous function, pre-images (or inverse images) of open sets are always open. Then, based on $\varphi$, we obtain the pre-image as an open set ${S_u} \in {S_r} \times {S_{ur}} \times {S_{ur}}$, on which $\varphi$ is strictly positive. Then, due to \ref{eq:proof_3_1}, we can obtain that:
\begin{equation}
   \{ S_r^*\}  \times (\{ S_{ - l}^*\}  \times (s_l^*,s_l^* + \varepsilon )) \times (\{ S_{ - l}^*\}  \times (s_l^* - \varepsilon ,s_l^*)) \in {S_u}
\end{equation}
where $S^{*}_{-l}$ represents the set obtained by removing the $l$-th element from the set $S^*$. In this context, this usually means removing a certain dimension from a multi-dimensional space. Typically, $s_l$ and the elements of $S_{ur}$ are of the same size, as $s_l$ is an element of $S_{ur}$. In practical problems, $s_l$ might be a subset or a specific value within $S_{ur}$. So, ${S_u}$ is non-empty. Assume that there exists at least one subset $S_a$ of changing $S_{ur}$. Then, we can obtain that for any $s_a \in S_a$, there is an open subset $O_a  \subseteq S_a$ containing $s_a$. Define the space: 
\begin{equation}
   {\mathbbm{R}_a}: = \{ ({s_a},{{\mathord{\buildrel{\lower3pt\hbox{$\scriptscriptstyle\frown$}} 
\over s} }_a})|{s_a} \in {S_a},{{\mathord{\buildrel{\lower3pt\hbox{$\scriptscriptstyle\frown$}} 
\over s} }_a} \in {O_a}\} 
\end{equation}

Now consider the intersection $\mathbbm{R}_{ \ge 0}\cap \mathbbm{R}_c$ of the open set $\mathbbm{R}_{ \ge 0}$ with the topological subspace $\mathbbm{R}_c$. Since $\mathbbm{R}_{ \ge 0}$ is open, by the definition of topological subspaces, the intersection $\mathbbm{R}_{ \ge 0}\cap \mathbbm{R}_c\subseteq \mathbbm{R}_c$ is \textit{open} in $\mathbbm{R}_c$, (and thus has the same dimension as $\mathbbm{R}_c$ if non-empty). 

Define ${s^c} = {S_{ur}}\backslash {S_a}$ and ${\mathbbm{R}_c}={S_r} \times {S_{{s^c}}} \times {\mathbbm{R}_a}$, then we can obtain that there exists ${\varepsilon ^1}>0$,such that $\{ S_{ - l}^*\}  \times (s_l^* - {\varepsilon ^1},s_l^*) \subset {O_{s_l^*}}$. Then, we have:
\begin{equation}\label{qsasddfhsef}
\{ S_r^*\}  \times \{ S_{{s^c}}^*\}  \times (\{ {S_{{s^c}}}\}  \times (s_l^*,s_l^* + \varepsilon )) \times (\{ {S_{{s^c}}}\}  \times (s_l^* - \varepsilon ,s_l^*)) \subset {\mathbbm{R}_c}
\end{equation}
where $O_{s^{*}_l}$ represents an open set containing $s_l^*$ and its neighborhood. By the property that the inverse image of an open set under a continuous function is also open, we can deduce the non-emptiness of this open set. Then, we can obtain that:
\begin{equation}\label{qschtsfhsef}
\{ S_r^*\}  \times (\{ S_{ - l}^*\}  \times (s_l^*,s_l^* + \varepsilon )) \times (\{ S_{ - l}^*\}  \times (s_l^* - \varepsilon ,s_l^*)) \subset {\mathbbm{R}_c}
\end{equation}
For Equivalence of Equation \ref{qsasddfhsef} and Equation \ref{qschtsfhsef}, the derivation shows that when the loss function $\mathcal{L}(f,X_1^{aug},X_2^{aug})$ is minimized, the output of the function $f$ depends only on the task-relevant information $S_r$ and not on the unconstrained variables $S_{ur}$. Thus, $f(X)$ can only depend on $S_r$ and not on $S_{ur}$.

Therefore, we can obtain that when the first term and second term of  $\mathcal{L}(f,X_1^{aug},X_2^{aug})$ can be minimized, then the $S_r$ in the output of $f$ can only depend on the true content $S_r$ and not on $S_{ur}$.

\end{proof}

\end{document}